\documentclass[11pt]{article}

\usepackage[default]{sourcesanspro}

\usepackage{amsmath}
\usepackage{bbm}
\usepackage{newtxmath}      

\usepackage{xcolor}

\usepackage{microtype}
\usepackage[top=1in, bottom=1.25in, left=0.9in, right=0.9in]{geometry}

\usepackage{amsthm}

\usepackage{pdflscape,graphicx}
\graphicspath{{fig/}}

\usepackage{calrsfs}

\usepackage{algorithm}
\usepackage{algorithmicx}
\usepackage{algpseudocode}

\usepackage{booktabs}
\usepackage{changepage}
\usepackage{enumitem}

\graphicspath{{fig/}}

\usepackage{etoolbox}
\let\bbordermatrix\bordermatrix
\patchcmd{\bbordermatrix}{8.75}{4.75}{}{}
\patchcmd{\bbordermatrix}{\left(}{\left[}{}{}
\patchcmd{\bbordermatrix}{\right)}{\right]}{}{}

\theoremstyle{plain}
\newtheorem{theorem}{Theorem}[section]
\newtheorem{conjecture}[theorem]{Conjecture}
\newtheorem{lemma}[theorem]{Lemma}
\newtheorem{proposition}[theorem]{Proposition}
\newtheorem{corollary}[theorem]{Corollary}
\newtheorem{assumption}[theorem]{Assumption}

\newtheorem{remark}[theorem]{Remark}

\newcommand{\Prob}{\mathbb{P}}
\newcommand{\Ex}{\mathbb{E}}

\newcommand{\1}{\mathbbm{1}}

\newcommand{\orbit}{\mathrm{orbit}}

\usepackage{appendix}

\usepackage[sort&compress]{natbib}
\setcitestyle{numbers,square}

\usepackage[colorlinks=true,breaklinks=true,bookmarks=true,urlcolor=blue,
     citecolor=blue,linkcolor=blue,bookmarksopen=false,draft=false]{hyperref}

\def\EMAIL#1{\href{mailto:#1}{#1}}
\def\URL#1{\href{#1}{#1}}

\newcommand{\Rbar}{\overline{\mathbb{R}}}
\newcommand{\R}{\mathbb{R}}

\begin{document}

\title{Restless bandits with imperfect binary feedback: PCL-indexability analysis and computation}

\author{Jos\'e Ni\~no-Mora\thanks{\small The author's work was supported in part by Universidad Carlos III de Madrid (UC3M) through an internal research program grant and a grant for the acquisition of research tools.}
\\ Departamento de Estad\'{\i}stica \\
    Universidad Carlos III de Madrid \\
     28903 Getafe (Madrid), Spain \\  \EMAIL{jose.nino@uc3m.es}  \\  
      \URL{https://alum.mit.edu/www/jnimora} \\
      \URL{http://orcid.org/0000-0002-2172-3983}}
\date{Submitted 27/3/2026}

\maketitle

\begin{abstract}
We study restless bandits with binary latent states and imperfect binary feedback, motivated by opportunistic spectrum access with sensing errors. For the associated belief-state model, we develop a partial conservation laws (PCL)-based analytical and computational framework for establishing indexability and evaluating the Whittle index, building on a verification theorem for real-state discounted restless bandits. The framework analyzes the stochastic dynamics via an associated deterministic skeleton, renewal decompositions, and combinatorics on words. It yields tractable expressions for discounted reward and resource metrics in several threshold regimes, enabling full verification of the PCL-indexability conditions there. For the remaining regime, where a complete analytic verification is not achieved in this paper, we derive efficient numerical schemes for computing the relevant marginal metrics and the marginal productivity (MP) index, which equals the Whittle index when those conditions hold. Extensive computational experiments provide strong evidence that these conditions also hold in that regime across broad parameter ranges and without the stringent parameter restrictions imposed in prior work. The experiments further show that theMP index policy typically outperforms standard benchmark policies, often by a substantial margin.
\\[1ex]
\end{abstract}

\section{Introduction}
\label{s:intro}

\subsection{Background and motivation}
\label{s:motiv}

Many sequential resource-allocation problems involve a collection of stochastically evolving entities, here called \emph{projects}, whose latent states are only partially observed. In each discrete-time period, the controller may activate only a limited number of projects, thereby obtaining information and potentially earning reward. We consider a practically relevant setting in which each project's latent state is binary (bad/good), and activation yields imperfect, \emph{one-sided} binary ACK/NACK feedback: an ACK certifies a good state and yields immediate reward, whereas a NACK (no observed success) yields no reward and is inherently ambiguous, since it may occur even when the latent state is good. Because projects continue to evolve when not selected, the controller must balance immediate exploitation against information gathering and future opportunities. This exploration--exploitation tradeoff leads naturally to a \emph{restless multi-armed bandit problem} (RMABP), a widely studied modeling framework introduced by \citet{whit88}; see, e.g., \cite{liuZhao10,mehtaetal18,meshrametal18,kazaetal19,liuetal24,akbarMahaj24} for representative work on partially observed restless bandits, and \cite{nmmath23} for a broader review.

We study a partially observed RMABP in which \(N\) projects, labeled by \(n\in\mathcal N\triangleq\{1,\ldots,N\}\), have latent binary states that evolve exogenously as Markov chains, independently of the controller's actions. In each period, the controller chooses which projects to activate, subject to an activation-capacity constraint; nonselected projects are passive. Passive actions yield neither feedback nor reward, and the corresponding belief states are updated predictively using the latent-state dynamics. By contrast, when project \(n\) is activated, it generates an ACK with probability \(0<\kappa_n<1\) conditional on the latent state being good, and the ACK yields reward \(r_n>0\). Upon observing ACK/NACK, the controller updates its posterior belief via Bayes' rule. Let \(X_n(t)\in[0,1]\) denote the posterior probability that project \(n\) is in the good state at the start of period \(t\). The vector \(\mathbf X(t)=(X_n(t))_{n\in\mathcal N}\) is a sufficient statistic, so the overall control problem becomes a high-dimensional \emph{Markov decision process} (MDP) on the belief-state space, induced by partial observability \citep{krishn25}.

Because belief-state RMABPs are generally intractable to solve exactly, the main goal of this paper is to exploit the special structure of the model---action-independent latent-state transitions, one-sided ACK/NACK feedback, and ACK-dependent rewards---to apply the \emph{partial conservation laws} (PCL)-based verification theorem of \citet{nmmor20} to this setting. That theorem provides sufficient conditions for indexability and Whittle index evaluation in discounted real-state restless bandits. In several threshold regimes, we obtain full analytical verification of the required \emph{PCL-indexability conditions}. For the remaining regime, in which a complete analysis remains elusive, we derive efficient computational schemes and use them to provide broad numerical evidence that those conditions also hold. We further compute the associated \emph{marginal productivity} (MP) index, which equals the Whittle index under the PCL conditions.

This problem structure is motivated by several modern engineering applications. \emph{Opportunistic spectrum access} (OSA) in cognitive radio networks is a natural example. Each project corresponds to a primary-user (PU) channel whose availability evolves as a two-state Markov chain, and the controller senses up to \(M\) channels in each period. Under sensing errors and an exogenous collision-tolerance constraint, secondary users (SUs) use the sensing outcomes to decide whether to transmit. By the separation principle of \citet{chenetal08}, the access decision is myopically optimal conditional on the sensing outcome, so the effect of sensing errors on each channel can be summarized in our reduced model by a single parameter \(\kappa_n\), the maximum feasible transmission probability when channel \(n\) is available. This sensing-access mechanism induces the one-sided feedback structure in our model: an ACK certifies that the channel is available, whereas \(\kappa_n\) determines the effective ACK-given-available probability. See Section~\ref{s:om} for details.

Related models also arise in link- or server-probing problems, where a controller sequentially tests or selects time-varying resources whose availability evolves exogenously and earns reward only upon successful transmission or service; see, e.g., \cite{liuZhao10,murugesanetal12}. Another example is remote diagnostics in distributed devices, where successful heartbeat checks confirm liveness whereas failed checks remain ambiguous yet informative; see, e.g., \cite{chandraToueg96,aguileraetal99}.

\subsection{Related work on partially observed restless bandits, Whittle indices, and OSA}
\label{s:pw}

Related work most relevant to this paper falls into three strands: (i) \emph{partially observed MDP} (POMDP) formulations and structural results for OSA with binary Markov channels; (ii) Whittle-index approaches for sensing and scheduling problems with continuous belief states; and (iii) results on threshold optimality and indexability under imperfect binary feedback, including work closest to the one-sided ACK/NACK setting studied here.

\subsubsection{OSA as a belief-state POMDP}
\label{s:osabsp}
Control of partially observed binary Markov systems is naturally formulated as a belief-state MDP, i.e., a POMDP whose sufficient statistic is the posterior probability of the good state. In the OSA literature, early work studied optimal transmission on a single channel---often under the Gilbert--Elliott model---and established structural properties such as optimality of threshold policies \citep{johnkrish06}. Subsequent work considered sensing errors and times, and delay or energy costs through constrained POMDP formulations and heuristic approaches \citep{hoangetal09}, while multichannel OSA with sensing errors and collision constraints was formulated as a POMDP in \citep{zhaoetal07}. A key later development is the \emph{separation principle} of \citet{chenetal08}, which decouples sensing and access decisions in a broad class of imperfect-sensing OSA models.

\subsubsection{From POMDPs to RMABPs and Whittle indices}
\label{s:fprwi}

Because such POMDPs are generally intractable to solve exactly, a substantial literature exploits their structure as RMABPs with continuous belief states. In the classic \emph{rested} case, the Gittins index yields an optimal index policy \citep{gi79}, but OSA and many monitoring and scheduling problems are \emph{restless}. This motivates using the \emph{Whittle index} \citep{whit88}, which provides a scalable heuristic policy and is often accompanied by computable performance bounds. Within OSA, \citet{ahmadetal09} established conditions under which \emph{myopic sensing} is optimal, while \citet{liuZhao10} established \emph{indexability} and derived closed-form Whittle-index expressions in the perfect-sensing case. Related imperfect-detection models were analyzed in \citep{liuZhaoKrishna10}, although indexability was not addressed there.

More broadly, restless bandits with imperfect, limited, or constrained observations have been studied in frameworks encompassing OSA-type models, including hidden-Markov and constrained-feedback formulations \citep{mehtaetal18,meshrametal18,kazaetal19}, as well as approaches aimed at low-complexity Whittle-index computation under imperfect feedback \citep{liuetal24}. Recent work has also considered data-driven and risk-aware RMAB extensions \citep{Mate2020NeurIPS,Mate2021RiskAwareRMAB}; see also \citep{aaltoetal19,akbarMahaj24}.

\subsubsection{Closest work under imperfect binary feedback}
\label{s:cwuibf}

The work most closely related to our setting addresses threshold optimality and/or Whittle indexability for belief-state RMABPs with imperfect binary feedback, typically under parameter restrictions. \citet{meshrametal18} study a broader model allowing action-dependent transitions and introduce \emph{approximate indexability}. Their threshold-optimality and indexability results hold only under stringent parameter restrictions; when specialized to our setting, these reduce to low autocorrelation \((0\le \rho \le 1/5)\) and strong discounting \((\beta<1/3)\).

\citet{wangetal18} study the OSA special case of our model and propose a fixed-point/region-partition analysis of the induced nonlinear belief dynamics, together with a piecewise-linearization argument, to support indexability and closed-form index expressions. Their treatment highlights the role of piecewise threshold dynamics, but several continuity and monotonicity issues across regime boundaries are handled only implicitly. The follow-up work \citet[Ch.\ 3]{wangChen21} provides additional detail. Their analysis focuses on selected iterates of the deterministic belief-update maps \(\phi^0\) and \(\phi^1\) (see Section~\ref{s:phi_props_main}). By contrast, threshold policies in our setting generate a broader family of mixed compositions of these maps, together with ACK-induced resets in the one-sided model, and this richer structure drives the piecewise and discontinuous behavior of the resulting metrics. Their sufficient sensing-error condition, in terms of the false-alarm probability \(\epsilon\) (where \(\kappa=1-\epsilon\)), specializes in our notation to
\[
\epsilon \le \frac{p_{01} p_{10}}{(1-p_{01}) (1-p_{10})}.
\]
This can be quite restrictive when \(\rho>0\), since the right-hand side may be arbitrarily small.

\citet{kazaetal19} study a hidden-state RMABP with ACK/NACK feedback and sparse observations. They establish threshold optimality for single-project subproblems under strong conditions (Theorem~1), requiring low autocorrelation \((0<\rho<b/5)\) and strong discounting \((\beta<b/5)\), where \(b=\min\{1,r\}\) in our model. Their indexability claim (Theorem~2) is stated under the condition \(\beta<1/3\), but the preceding development still relies on threshold optimality.

\citet{liuetal24} also consider the OSA case of our model, assuming positive channel autocorrelation. They give a sufficient condition for threshold optimality,
\[
\beta \le \frac{1}{(3-\epsilon)\rho}=\frac{1}{(2+\kappa)\rho},
\]
derive sufficient indexability conditions, and show that the model is indexable for \(\beta\le 1/2\). They also propose a low-complexity algorithm for \emph{approximate} Whittle-index computation under the condition
\[
\beta \le \min\Bigl\{\frac{1}{(2+\kappa)\rho},\,\frac12\Bigr\},
\]
which ensures both threshold optimality and indexability. This restriction can still be stringent, since \(1/((2+\kappa)\rho)\) can be close to \(1/3\) as \(\kappa\nearrow 1\) and \(\rho\nearrow 1\).

Finally, PCL methods have previously been applied to OSA-type models with binary success feedback. The perfect-feedback case \(\kappa=1\) was treated via PCLs in \citep{nmngi08} and later given a complete analysis in \citep[Section~12.2]{nmmor20}. The imperfect-feedback case \(0<\kappa<1\) was outlined in \citep{nmnetcoop09}, including a regime-wise decomposition of threshold dynamics, but without proofs of the PCL verification steps. The present paper builds on that line of work, develops the missing computational machinery, and uses it as a basis for a broad study of indexability and of the performance of the resulting MP index policy.

Beyond the OSA literature, \citet{danceSi19} gave the first proof of both Whittle indexability and threshold-policy optimality for Kalman-filter restless bandits, using the real-state PCL verification theorem of \citep{nmmor20}. As further context, see \citep{nmparxiv26} for a PCL-based complete indexability analysis of a belief-state adherence RMAB model.

The present paper extends this PCL program for real-state bandits to a one-sided imperfect-observation RMAB setting. Existing analytical results establish indexability only under fairly restrictive parameter conditions, which could suggest that such restrictions are intrinsic to the model. By contrast, our numerical results provide strong evidence that indexability persists over a much broader parameter range than is currently supported analytically.

\subsection{Goals, approach, and contributions}
\label{s:goals_appr}

Our main goal is to develop an analytical and computational framework for studying indexability and evaluating the MP index in belief-state restless bandits with one-sided ACK/NACK feedback and ACK-dependent rewards. Relative to much of the OSA literature reviewed above, our work differs both in emphasis and in methodology. Although we also exploit fixed points and regime-wise decompositions of the belief dynamics, our analysis is organized around the PCL-based verification theorem of \citet{nmmor20}, which reduces Whittle indexability and threshold-policy optimality to the verification of three conditions for certain project performance metrics under threshold policies.

A common route in the RMAB literature is first to establish optimality of threshold policies for single-project Lagrangian subproblems, and then to deduce indexability and evaluate the Whittle index from the associated optimal-threshold map and its monotonicity. By contrast, the PCL framework in \citep{nmmor20} yields the MP index directly as a ratio of marginal reward and marginal work metrics under threshold policies, thereby bypassing the need to invert a threshold map. In the present model, full analytical verification of \textup{(PCLI1--PCLI3)} remains elusive over part of the threshold range, because thresholding and ACK-induced resets generate intricate belief dynamics. Nevertheless, we show that the PCL route leads to a tractable and robust computational methodology: substantial parts of \textup{(PCLI1--PCLI3)} can be verified analytically, while the remaining parts can be investigated systematically by numerical means, yielding strong evidence in support of indexability over a broader parameter range than is currently supported analytically.

A terminological distinction is worth making explicit. The \emph{Whittle index} is defined implicitly, as the subsidy for passivity that makes active and passive actions equally desirable in the corresponding single-project Lagrangian problem. By contrast, the \emph{MP index} is defined explicitly as a ratio of marginal reward and marginal work metrics under threshold policies. Under \textup{(PCLI1--PCLI3)} over the full threshold range, the verification theorem in \citep{nmmor20} implies that the MP index coincides with the Whittle index. Since full analytical verification of \textup{(PCLI1--PCLI3)} remains open in some regimes, our numerical computations are formally for the MP index. Thus, throughout the paper we refer to the computed ratio as the MP index, and to the induced policy as the MP index policy; if the PCL conditions were verified over the full threshold range, this would ensure that the MP index is indeed the Whittle index.

This PCL program was initiated for the present model in our conference paper \citep{nmnetcoop09}. The present paper develops that line of work into a coherent analytical and computational framework. Our main contributions are as follows:
\begin{enumerate}[label=(C\arabic*),leftmargin=*]
\item We develop a regime-wise decomposition of the threshold-induced belief dynamics and derive closed-form expressions and renewal-type identities for the key reward and work metrics, tailored to the one-sided success-feedback structure. This yields explicit formulas in tractable regimes and efficient evaluation schemes elsewhere.

\item We obtain a direct PCL-based procedure for evaluating the MP index, and identify threshold regimes in which the sufficient conditions \textup{(PCLI1--PCLI3)} can be fully verified analytically.

\item We provide broad numerical evidence for the PCL conditions, and hence for indexability, well beyond the stringent discount-factor and autocorrelation restrictions imposed in prior work:  we carry out numerical tests for detecting violations of the PCL conditions, and do not observe any on the parameter grids explored.

\item We carry out extensive computational experiments showing that the resulting MP index policy is readily computable and typically outperforms the benchmark policies considered, often by a substantial margin. At the same time, the observed gaps to a Lagrangian dual upper bound can remain substantial, and experiments with two-type instances do not provide numerical evidence of asymptotic optimality.
\end{enumerate}

Finally, to help clarify a possible route to a complete indexability proof in the remaining threshold regime, we connect the threshold-induced belief dynamics with tools from \emph{combinatorics on words} (see, e.g., \citep{bersteletal09}), in particular Christoffel--Sturmian structure. In the spirit of \citet{danceSi19}, this points to a structural program for analyzing the ordering of action  itineraries induced by threshold policies. In the present paper, however, the actual computation of the reward and work metrics relies primarily on renewal-type decompositions rather than on a full characterization of that word structure. Further details on this symbolic-dynamics viewpoint are collected in Appendix~\ref{app:word_structure}.

\subsection{Organization of the paper}
\label{s:org}

The remainder of the paper is organized as follows. Section~\ref{s:osamprf} introduces the binary-success-feedback restless bandit model and shows how OSA with sensing errors and collision constraints fits within that framework. Section~\ref{s:rbiosa} reviews restless-bandit indexation via Lagrangian relaxation and presents the PCL-indexability framework and verification theorem underlying our approach. Section~\ref{s:compFG} develops the computational machinery for evaluating threshold-policy performance metrics, based on no-ACK skeleton dynamics and renewal decompositions, and identifies tractable threshold regimes together with explicit formulas for the associated marginal metrics and MP index. Section~\ref{s:experiments} reports the computational experiments, including broad numerical evidence on the PCL-indexability conditions and benchmarking of the resulting MP index policy against standard baseline policies and a Lagrangian upper bound. Section~\ref{s:conclusions} concludes and discusses directions for future work.

Appendix~\ref{app:compFG_proofs} collects supplementary proofs for the computational machinery developed in Section~\ref{s:compFG}. Appendix~\ref{app:illustrative_instance} presents an illustrative parameter instance together with diagnostic plots of the key metrics and the MP index. Appendix~\ref{app:word_structure} records the symbolic organization of intermediate-regime threshold itineraries via Christoffel and Sturmian words. Appendix~\ref{app:avgcrit_outline} briefly outlines the corresponding time-average criterion and its connection with the discounted analysis. Appendix~\ref{app:extra_experiments} collects supplementary parameter-dependence plots and implementation details.

\section{Binary-success-feedback restless bandit model and an OSA instantiation}
\label{s:osamprf}

\subsection{A generic binary-success-feedback restless bandit model}
\label{s:genmodel}

We consider a collection of \(N\) projects, labeled by
\(n\in\mathcal N\triangleq\{1,\ldots,N\}\), that evolve over discrete periods \(t=0,1,\ldots\).
Project \(n\) has a latent binary state \(S_n(t)\in\{0,1\}\), where
\(1\) represents a \emph{good} state and \(0\) a \emph{bad} state.
The latent state evolves as a two-state Markov chain with transition probabilities
\[
p_{n,ij}=\Prob\{S_n(t+1)=j\mid S_n(t)=i\}>0,\qquad i,j\in\{0,1\}.
\]
We restrict attention to the positively correlated case
\[
\rho_n\triangleq 1-p_{n,10}-p_{n,01}=p_{n,11}-p_{n,01}>0.
\]
This is natural in settings where latent conditions are persistent: a channel, server, or device that is currently good (respectively bad) is typically more likely to remain so in the next period than to switch immediately. In communication settings, this is the standard persistence captured by Gilbert--Elliott and related finite-state Markov channel models for correlated fading and bursty error processes \citep{johnkrish06,liuZhao10}. The latent-state processes are independent across projects.

At the start of each period \(t\), the controller selects an action
\(A_n(t)\in\{0,1\}\) for each project \(n\), where \(A_n(t)=1\) (\emph{active}) means that the project is selected or probed, and \(A_n(t)=0\) (\emph{passive}) means that it is not selected.
At most \(M\) projects can be active in any period, so
\begin{equation}
\label{eq:aklt1}
\sum_{n\in\mathcal N}A_n(t)\le M,\qquad t=0,1,2,\ldots.
\end{equation}

If \(A_n(t)=1\), the controller observes a one-sided binary feedback signal \(K_n(t)\in\{0,1\}\), interpreted as success/failure (ACK/NACK), where \(K_n(t)=1\) (ACK) certifies that the project is in the good state \(S_n(t)=1\), whereas \(K_n(t)=0\) (NACK) may occur in either latent state. Specifically, for some \(0<\kappa_n<1\),
\[
\Prob\{K_n(t)=1\mid S_n(t)=1,\,A_n(t)=1\}=\kappa_n,
\qquad
\Prob\{K_n(t)=1\mid S_n(t)=0,\,A_n(t)=1\}=0.
\]
If \(A_n(t)=0\), no feedback is received; for notational convenience, we set \(K_n(t)=0\) in that case.
Upon an ACK, project \(n\) earns reward \(r_n>0\), so the realized reward in period \(t\) is \(r_nK_n(t)\).

Let \(\mathcal H_t\) denote the global history available at the start of period \(t\), consisting of the initial prior together with all past actions and observed feedback up to time \(t\).
Since the latent states are unobserved, decisions are based on belief states. For project \(n\), the belief state is
\[
X_n(t)\triangleq \Prob\{S_n(t)=1\mid \mathcal H_t\}\in\mathcal X\triangleq[0,1],
\]
and the vector \(\mathbf X(t)=(X_n(t))_{n\in\mathcal N}\) is a sufficient statistic for control.
We consider scheduling policies \(\pi\in\Pi\) that select a feasible action vector \(\mathbf A(t)\) as a function of \(\mathbf X(t)\).

Given \(X_n(t)=x_n\) and \(A_n(t)=1\), Bayes' rule yields
\[
\Prob\{S_n(t)=1 \mid K_n(t)=1,\,A_n(t)=1\}=1,
\qquad
\Prob\{S_n(t)=1 \mid K_n(t)=0,\,A_n(t)=1\}
= \psi_n(x_n)\triangleq \frac{(1-\kappa_n)x_n}{1-\kappa_n x_n}.
\]
Thus, an ACK confirms a good latent state, whereas a NACK lowers the belief.

The belief dynamics follow by combining this Bayesian update with the latent Markov transition.
Under the passive action \(A_n(t)=0\), the update is deterministic:
\begin{equation}
\label{eq:pdn}
X_n(t+1)=\phi_n^0(X_n(t))\triangleq p_{n,01}+\rho_n X_n(t).
\end{equation}
Under the active action \(A_n(t)=1\), an ACK occurs at the end of period \(t\) with probability
\[
\Prob\{K_n(t)=1\mid X_n(t)=x_n,\,A_n(t)=1\}=\kappa_n x_n.
\]
Hence the active belief update is
\begin{equation}
\label{eq:pdnK01}
X_n(t+1)=
\begin{cases}
p_{n,01}+\rho_n=p_{n,11}, & \textup{w.p. }\kappa_n X_n(t),\\[2mm]
\phi_n^1(X_n(t))\triangleq p_{n,01}+\rho_n\,\psi_n(X_n(t)), & \textup{w.p. }1-\kappa_n X_n(t).
\end{cases}
\end{equation}

We focus primarily on the discounted-reward criterion. Given an initial belief vector \(\mathbf x\), the objective is
\begin{equation}
\label{eq:dop}
\max_{\pi\in\Pi}\;
\Ex_{\mathbf x}^{\pi}\!\left[\sum_{t=0}^{\infty}\sum_{n\in\mathcal N}\beta^t\,r_nK_n(t)\right],
\qquad 0<\beta<1,
\end{equation}
where \(\Ex_{\mathbf x}^{\pi}\) is expectation under policy \(\pi\) starting from \(\mathbf X(0)=\mathbf x\).
Equivalently, the expected one-step reward is
\[
R_n(x_n,a_n)\triangleq
\Ex[r_nK_n(t)\mid X_n(t)=x_n,\,A_n(t)=a_n]
=
r_n\kappa_n x_n a_n,
\]
so \eqref{eq:dop} can be written as
\begin{equation}
\label{eq:dopRn}
\max_{\pi\in\Pi}\;
\Ex_{\mathbf x}^{\pi}\!\left[\sum_{t=0}^{\infty}\sum_{n\in\mathcal N}\beta^t\,R_n(X_n(t),A_n(t))\right].
\end{equation}

For later reference, we also consider the average-reward criterion
\begin{equation}
\label{eq:aop}
\max_{\pi\in\Pi}\;
\liminf_{T\to\infty}\frac{1}{T+1}\,
\Ex_{\mathbf x}^{\pi}\!\left[\sum_{t=0}^{T}\sum_{n\in\mathcal N}R_n(X_n(t),A_n(t))\right].
\end{equation}
Its connection with the discounted formulation is outlined in Appendix~\ref{app:avgcrit_outline}.

Problems \eqref{eq:dop} and \eqref{eq:aop} are partially observed RMABPs with continuous belief states and are generally intractable. When the model is \emph{indexable}, the \emph{Whittle index policy} provides a scalable and well-grounded heuristic: at each time, it activates up to \(M\) projects with the largest nonnegative Whittle indices. In the PCL framework outlined later, we work instead with the MP index, which coincides with the Whittle index when conditions \textup{(PCLI1--PCLI3)} hold.

\subsection{Illustrative example: OSA with sensing errors}
\label{s:om}

We next show how the binary-success-feedback model of Section~\ref{s:genmodel} arises in opportunistic spectrum access (OSA) with sensing errors and collision constraints.
Consider a cognitive radio system with \(N\) licensed \emph{primary-user} (PU) channels, of which at most \(M\) can be sensed by the controller in each slot. Channel \(n\) offers throughput \(r_n\) Mb/slot and has PU occupancy state \(S_n(t)\in\{0,1\}\), where \(S_n(t)=1\) means that the channel is available to secondary users (SUs), and \(S_n(t)=0\) means that it is unavailable (busy). If channel \(n\) is sensed at time \(t\), a binary sensor outcome \(O_n(t)\in\{0,1\}\) is observed, where \(O_n(t)=1\) and \(O_n(t)=0\) indicate ``sensed available'' and ``sensed busy,'' respectively. Sensing errors are modeled by positive miss-detection and false-alarm probabilities \(\delta_n>0\) and \(\epsilon_n>0\), defined by
\[
\delta_n\triangleq \Prob\{O_n(t)=1\mid S_n(t)=0\},
\qquad
\epsilon_n\triangleq \Prob\{O_n(t)=0\mid S_n(t)=1\}.
\]
We assume that the sensor is informative \citep{chenetal08}, so
\begin{equation}
\label{eq:assdeltaeps}
\delta_n+\epsilon_n<1.
\end{equation}

An access attempt when \(S_n(t)=0\) causes a collision and therefore an unsuccessful transmission. Each channel \(n\) imposes a collision-tolerance requirement \(0<\zeta_n<1\), namely,
\[
\Prob\{\text{access at }t \mid S_n(t)=0,\,A_n(t)=1\}\le \zeta_n.
\]

To avoid SU contention, the controller selects at most \(M\) channels to be sensed in each slot, via actions \(A_n(t)\in\{0,1\}\) satisfying \eqref{eq:aklt1}. Conditional on sensing, an \emph{access rule} is specified by probabilities \(y_n(o)\in[0,1]\), \(o\in\{0,1\}\), where \(y_n(o)\) is the probability of attempting access given sensor outcome \(o\). Then
\begin{align*}
\Prob\{\text{access at }t \mid S_n(t)=1,\,A_n(t)=1\}
&=\epsilon_n y_n(0)+(1-\epsilon_n)y_n(1), \\
\Prob\{\text{access at }t \mid S_n(t)=0,\,A_n(t)=1\}
&=(1-\delta_n)y_n(0)+\delta_n y_n(1).
\end{align*}

Following the \emph{separation principle} of \citet{chenetal08}, we choose the access probabilities \(y_n(o)\) myopically so as to maximize the access probability when the channel is truly available, subject to the collision constraint:
\begin{equation}
\label{eq:lpphi}
\kappa_n \triangleq
\max_{0\le y_n(0),\,y_n(1)\le 1}
\Bigl\{
\epsilon_n y_n(0)+(1-\epsilon_n)y_n(1):
(1-\delta_n)y_n(0)+\delta_n y_n(1)\le \zeta_n
\Bigr\}.
\end{equation}
Thus, \(\kappa_n\) is the maximal probability of attempting access on channel \(n\) when it is in fact available, while respecting the collision-tolerance requirement. Under \eqref{eq:assdeltaeps}, the optimizer admits the closed form (see \citep[Proposition~2]{chenetal08})
\begin{equation}
\label{eq:accprob}
y_n^*(1)=\min\!\left\{1,\frac{\zeta_n}{\delta_n}\right\},
\qquad
y_n^*(0)=\left(\frac{\zeta_n-\delta_n}{1-\delta_n}\right)^+,
\qquad
\kappa_n=\epsilon_n y_n^*(0)+(1-\epsilon_n)y_n^*(1).
\end{equation}
If \(\delta_n=\zeta_n\) (the ``trust-the-sensor'' point in \citet[Theorem~2]{chenetal08}), then \((y_n^*(0),y_n^*(1))=(0,1)\) and \(\kappa_n=1-\epsilon_n\).

Under the collision-tolerance condition \(0<\zeta_n<1\), we have \(y_n^*(1)=\min\{1,\zeta_n/\delta_n\}>0\), and hence \(\kappa_n>0\). Further, the collision constraint rules out \((y_n^*(0),y_n^*(1))=(1,1)\), and since \(0<\epsilon_n<1\), this implies \(\kappa_n<1\). 

Let \(K_n(t)\in\{0,1\}\) denote the ACK/NACK indicator, where \(K_n(t)=1\) if an SU attempts access on channel \(n\) at time \(t\) and the channel is available, i.e., \(S_n(t)=1\). Under the optimal access rule, conditioning on \(X_n(t)=x_n\) gives
\[
\Prob\{K_n(t)=1\mid X_n(t)=x_n,\,A_n(t)=1\}=\kappa_n x_n.
\]
Moreover, a successful SU transmission on channel \(n\) yields throughput \(r_n\) Mb/slot, so the one-slot expected reward is
\[
R_n(x_n,a_n)\triangleq r_n\kappa_n x_n a_n.
\]
Thus, in the OSA interpretation, the good latent state corresponds to channel availability, the ACK event corresponds to successful transmission, and \(\kappa_n\) summarizes the optimal sensing-conditioned access rule under the collision constraint.

\section{Restless bandit indexation}
\label{s:rbiosa}

\subsection{Indexability, dual bound, and Whittle index policy}
\label{s:lrdi}

Problem~\eqref{eq:dop} is a discounted-reward RMABP with the per-period activation-capacity constraint~\eqref{eq:aklt1}. Since computing an optimal policy is generally intractable, we adopt Whittle's index approach \citep{whit88}, based on a Lagrangian relaxation and decomposition into single-project subproblems.

Let \(\widehat{\Pi}\) denote the class of stationary randomized policies that, at each time \(t\), select an action vector \(\mathbf A(t)\in\{0,1\}^N\) as a possibly randomized function of the current belief vector \(\mathbf X(t)\), without being required to satisfy the per-period constraint~\eqref{eq:aklt1}. We relax~\eqref{eq:aklt1} by imposing only that the expected total discounted activation effort not exceed \(M/(1-\beta)\):
\begin{equation}
\label{eq:avversac}
\Ex_{\mathbf{x}}^{\pi}\!\left[\sum_{t = 0}^\infty \sum_{n \in \mathcal{N}} \beta^t A_n(t)\right]
\;\leqslant\; \frac{M}{1-\beta}.
\end{equation}
The resulting relaxed problem has optimal value
\begin{equation}
\label{eq:rp}
\widehat{V}^*(\mathbf{x}) \;\triangleq\;
\max_{\pi\in \widehat{\Pi}}
\left\{
\Ex_{\mathbf{x}}^{\pi}\!\left[\sum_{t=0}^{\infty}\sum_{n\in\mathcal{N}}
R_n\bigl(X_n(t),A_n(t)\bigr)\,\beta^t\right]
\;:\; \eqref{eq:avversac}
\right\}.
\end{equation}
Since \(\Pi\subseteq \widehat{\Pi}\) and~\eqref{eq:aklt1} implies~\eqref{eq:avversac}, it follows that
$
V^*(\mathbf{x}) \leqslant \widehat{V}^*(\mathbf{x}).
$

Attaching a multiplier \(\lambda\) to \eqref{eq:avversac}, define
\begin{equation}
\label{eq:lr}
L(\mathbf{x};\lambda) \;\triangleq\;
\max_{\pi\in \widehat{\Pi}}
\Ex_{\mathbf{x}}^{\pi}\!\left[
\sum_{t=0}^{\infty}\sum_{n\in\mathcal{N}}
\bigl(R_n\bigl(X_n(t),A_n(t)\bigr)-\lambda\,A_n(t)\bigr)\,\beta^t
\right]
\;+\; \frac{M}{1-\beta}\,\lambda.
\end{equation}
For \(\lambda\geqslant 0\), this is a Lagrangian relaxation of \eqref{eq:rp}, and weak duality yields
$
\widehat{V}^*(\mathbf{x})\leqslant L(\mathbf{x};\lambda).
$
Minimizing over \(\lambda\geqslant 0\) gives the following Lagrangian dual bound, which upper-bounds both \(\widehat{V}^*(\mathbf{x})\) and \(V^*(\mathbf{x})\):
\begin{equation}
\label{eq:ld}
D(\mathbf{x}) \;\triangleq\; \min_{\lambda\geqslant 0} L(\mathbf{x};\lambda),
\end{equation}

The relaxation~\eqref{eq:lr} decouples across projects into single-project subproblems. For each \(n\in\mathcal{N}\), let \(\Pi_n\) denote the class of stationary policies for project \(n\) in isolation. Define the single-project value under activity charge \(\lambda\) by
\begin{equation}
\label{eq:nuwagesp}
L_n(x_n;\lambda)\;\triangleq\;
\max_{\pi_n\in\Pi_n}
\Ex_{x_n}^{\pi_n}\!\left[
\sum_{t=0}^{\infty}
\bigl(R_n\bigl(X_n(t),A_n(t)\bigr)-\lambda\,A_n(t)\bigr)\,\beta^t
\right].
\end{equation}
Then
\begin{equation}
\label{eq:LagDecomp}
L(\mathbf{x};\lambda)
\;=\; \sum_{n\in\mathcal{N}} L_n(x_n;\lambda)
\;+\; \frac{M}{1-\beta}\,\lambda.
\end{equation}

We say that project \(n\) is \emph{indexable} if there exists a function
\(\lambda_n^*:\mathcal X\to\mathbb{R}\) such that, for every \(x_n\in\mathcal X\) and \(\lambda\in\mathbb{R}\), the active action \(a_n=1\) is optimal in \eqref{eq:nuwagesp} at belief state \(x_n\) if and only if
\(\lambda_n^*(x_n)\ge \lambda\), while the passive action \(a_n=0\) is optimal at \(x_n\) if and only if
\(\lambda_n^*(x_n)\le \lambda\).
Thus both actions are optimal at \(x_n\) if and only if \(\lambda_n^*(x_n)=\lambda\).
This statewise formulation is the one adopted in our PCL verification framework \citep{nmmor20}. It is equivalent to Whittle's original set-expansion definition when the indifference charge is unique at each state.

Under indexability, Whittle's policy activates at each period \(t\) up to \(M\) projects with the largest nonnegative indices \(\lambda_n^*\!\bigl(X_n(t)\bigr)\). Later, within the PCL framework, we will introduce the MP index as an explicit candidate index; when the PCL conditions hold, it coincides with the Whittle index.

\begin{remark}[On Whittle's set-expansion definition of indexability]
\label{rem:setexp_vs_scalar}
Whittle's original definition of indexability is stated in terms of the optimal passive set
\[
\mathcal{P}_n(\lambda)=\{x_n\in\mathcal{X}: a_n=0 \text{ is optimal at }x_n\},
\]
which expands monotonically as \(\lambda\) grows. Equivalently, one may consider the optimal active set
\[
\mathcal{A}_n(\lambda)=\{x_n\in\mathcal{X}: a_n=1 \text{ is optimal at }x_n\}.
\]
For each state \(x_n\), this yields an index interval
\([\lambda_{\sup,n}^*(x_n),\,\lambda_{\inf,n}^*(x_n)]\), where
\[
\lambda_{\inf,n}^*(x_n)=\inf\{\lambda:\ x_n\in\mathcal{P}_n(\lambda)\},
\qquad
\lambda_{\sup,n}^*(x_n)=\sup\{\lambda:\ x_n\in\mathcal{A}_n(\lambda)\}.
\]
Under the set-expansion property one has
\[
\lambda_{\sup,n}^*(x_n)\le \lambda_{\inf,n}^*(x_n),
\]
and a nontrivial interval corresponds to a range of charges \(\lambda\) for which both actions are optimal at \(x_n\).

The interval collapses to a point precisely when there is a unique Whittle index \(\lambda_n^*(x_n)\), equal to that common value. This uniqueness is implicit in \citet{whit88}, where the index is defined as \emph{the} passivity subsidy, equivalently the activity charge, that makes the controller indifferent between the two actions at state \(x_n\).
\end{remark}

\subsection{PCL-based verification theorem for threshold-indexability}
\label{s:eitpob}

To apply Whittle's index approach \citep{whit88}, indexability must first be established for the single-project \(\lambda\)-charge subproblems. Rather than following the conventional route---first proving optimality of threshold policies for each subproblem \eqref{eq:nuwagesp}, and then establishing monotonicity of the optimal threshold as a function of \(\lambda\)---we adopt the alternative PCL-based approach. In the discrete-state setting, PCLs yield general sufficient indexability conditions  and  an  index algorithm; see \cite{nmaap01,nmmp02,nmtop07,nmmor06}. Here we use the PCL-indexability framework for real-state  projects in \citep{nmmor20}.

We henceforth focus on a single project and suppress its label. Let \(\Pi\) denote the class of stationary policies for the single-project belief process. For any \(\pi\in\Pi\) and initial belief \(x\in\mathcal X\), define the discounted \emph{reward} and \emph{work} metrics
\begin{equation}
\label{eq:FxGxpi}
F(x,\pi)\triangleq \Ex^{\pi}_{x}\!\Bigg[\sum_{t=0}^{\infty}
R\bigl(X(t),A(t)\bigr)\,\beta^{t}\Bigg],
\qquad
G(x,\pi)\triangleq \Ex^{\pi}_{x}\!\Bigg[\sum_{t=0}^{\infty}
A(t)\,\beta^{t}\Bigg],
\end{equation}
where
$
R(x,a)\triangleq a\,r\,\kappa\,x,\enspace a\in\{0,1\}.
$
Given an activation charge \(\lambda\in\mathbb R\), define the corresponding \(\lambda\)-charge performance by
\[
\mathcal L(x,\pi;\lambda)\triangleq F(x,\pi)-\lambda\,G(x,\pi),
\]
so that the single-project \(\lambda\)-charge problem is
\begin{equation}
\label{eq:cocs}
L(x;\lambda)\triangleq\max_{\pi\in\Pi}\,\mathcal L(x,\pi;\lambda).
\end{equation}

Threshold policies play a central role. For any \(z \in\Rbar\triangleq\mathbb R\cup\{-\infty,\infty\}\), the \emph{\(z\)-threshold policy} activates when \(x>z\) and remains passive otherwise. The extended thresholds have the natural limiting interpretation: \(z=-\infty\) yields the always-active policy, whereas \(z=+\infty\) yields the always-passive policy (though in the present setting finite thresholds suffice). We write \(F(x,z)\) and \(G(x,z)\) for the reward and work metrics under the \(z\)-threshold policy, and define
\[
\mathcal L(x,z;\lambda)\triangleq F(x,z)-\lambda G(x,z).
\]

Marginal metrics are defined as follows. For \(a\in\{0,1\}\) and threshold \(z\), let \(\langle a,z\rangle\) denote the policy that takes action \(a\) at time \(0\) and thereafter follows the \(z\)-threshold policy. Define the \emph{marginal reward} and \emph{marginal work} metrics by
\[
f(x,z)\triangleq F\bigl(x,\langle 1,z\rangle\bigr)-F\bigl(x,\langle 0,z\rangle\bigr),
\qquad
g(x,z)\triangleq G\bigl(x,\langle 1,z\rangle\bigr)-G\bigl(x,\langle 0,z\rangle\bigr).
\]
If \(g(x,z)>0\) for all \(x\) and \(z\), define the \emph{marginal productivity} (MP) metric and the associated \emph{MP index} by
\[
m(x,z)\triangleq\frac{f(x,z)}{g(x,z)},
\qquad
m(x)\triangleq m(x,x),\qquad x\in\mathcal X.
\]

We will verify the following real-state \emph{PCL-indexability conditions}:
\begin{enumerate}[label=(PCL\arabic*),leftmargin=*]
\item \(g(x,z)>0\) for all \(x\in\mathcal X\) and \(z\in\R\).

\item The MP index \(m(\cdot)\) is nondecreasing and continuous on \(\mathcal X\).

\item For each \(x\in\mathcal X\) and finite thresholds \(z_1<z_2\),
\[
F(x,z_2)-F(x,z_1)=\int_{(z_1,z_2]} m(z)\,G(x,dz).
\]
That is, as a function of the threshold, \(F(x,\cdot)\) is the indefinite \emph{Lebesgue--Stieltjes integral} \citep{cartvanBrunt00} of \(m(\cdot)\) with respect to the signed Stieltjes measure \(G(x,dz)\) induced by the right-continuous function \(z\mapsto G(x,z)\). Under \textup{(PCLI1)}, the latter is nonincreasing and of bounded variation; see \cite[Lemmas~10, 11, and~17]{nmmor20}.
\end{enumerate}

The single-project \(\lambda\)-charge problem \eqref{eq:cocs} is called \emph{threshold-indexable} if the project is indexable in the statewise Whittle sense of Section~\ref{s:lrdi} and, for each \(\lambda\in\mathbb R\), there exists an optimal \(z^*(\lambda)\)-threshold policy. A \emph{generalized inverse} of a nondecreasing function \(h:\mathcal X\to\mathbb R\) is any function \(z^*:\mathbb R\to\R\) such that, for all \(\lambda\in\mathbb R\),
\[
\inf\{x\in\mathcal X: h(x)\ge \lambda\}\le z^*(\lambda)\le \sup\{x\in\mathcal X: h(x)\le \lambda\},
\]
with the usual conventions \(\inf\varnothing=+\infty\) and \(\sup\varnothing=-\infty\).

With these definitions in place, the following verification theorem, specialized from \citep{nmmor20}, shows that the PCLI conditions imply both threshold-indexability and an explicit index formula.

\begin{theorem}[PCL verification theorem for threshold-indexability]
\label{the:sic}
Under \textup{(PCLI1--PCLI3)}, the single-project discounted \(\lambda\)-charge problem \eqref{eq:cocs} is threshold-indexable. Moreover, the Whittle index exists and equals the MP index \(m(x)\), and any optimal threshold map \(z^*(\cdot)\) is a generalized inverse of \(m(\cdot)\).
\end{theorem}

The remainder of the paper is devoted to verifying these conditions and deriving tractable expressions for \(F\), \(G\), and \(m(\cdot)\). Where possible, we establish the required properties analytically; elsewhere, especially in threshold regimes where the belief dynamics become more intricate, we provide systematic numerical evidence.

\section{Computing threshold-policy performance metrics via the no-ACK skeleton}
\label{s:compFG}

This section develops the computational machinery for evaluating the reward and work metrics \(F(x,z)\) and
\(G(x,z)\) under \(z\)-threshold policies, together with the associated marginal metrics \(f(x,z)\) and \(g(x,z)\) and the MP index
\(m(x)=f(x,x)/g(x,x)\). The key idea is to decompose threshold-driven belief evolution into two components:
\textup{(i)} deterministic \emph{pre-ACK} dynamics, described by the \emph{no-ACK skeleton} and its associated survival probabilities; and
\textup{(ii)} random ACK events that reset the belief and restart the same deterministic evolution. This decomposition yields
renewal-type identities for \(F\), \(G\), \(f\), and \(g\) that are numerically tractable and admit closed forms in several threshold regimes.
These identities also support practical verification of the PCL-indexability conditions and guide the numerical experiments reported later.
Longer proofs are deferred to Appendix~\ref{app:compFG_proofs}.

Throughout this section we focus on a single project and suppress its label. We assume
\[
0<\beta<1,\quad 0 < p_{01} < 1,\quad 0 < p_{10} < 1,\quad
\rho \triangleq 1-p_{10}-p_{01} = p_{11}-p_{01}> 0,\quad
0 < \kappa < 1,\quad r>0,
\]
and recall that \(\mathcal{X}\triangleq[0,1]\). Throughout, ``increasing'' and ``decreasing'' are interpreted in the strict sense.

\subsection{Belief-update maps and their properties}
\label{s:phi_props_main}

We begin by collecting the structural properties of the one-step belief-update maps that underpin the subsequent
computations under threshold policies. We introduce the \emph{passive update} \(\phi^0\) and the \emph{active no-ACK update}
\(\phi^1\), identify a forward-invariant belief interval, and characterize the fixed points and closed-form iterates of
both maps. These properties underpin the no-ACK skeleton and the itinerary/word machinery used later, and also prepare the
contractiveness argument obtained after an increasing change of variables to log-odds.

Given belief \(x\) and activation, the posterior belief after observing a NACK is
\begin{equation}
\label{eq:alpha_def}
\psi(x)\triangleq \Prob\{S(t)=1\mid K(t)=0,\,A(t)=1,\,X(t)=x\}
=\frac{(1-\kappa)x}{1-\kappa x},\qquad x\in\mathcal{X}.
\end{equation}
Let \(\phi^0\) and \(\phi^1\) denote the passive and active no-ACK belief-update maps, respectively:
\begin{equation}
\label{eq:phi01x}
\phi^0(x)\triangleq p_{01}+\rho x,
\qquad
\phi^1(x)\triangleq p_{01}+\rho\,\psi(x)
=\phi^0(x)-\rho\,\frac{\kappa x(1-x)}{1-\kappa x},
\qquad x\in\mathcal{X}.
\end{equation}

\begin{remark}
\label{rem:phi01x}
\begin{enumerate}[label=(\roman*), ref=\roman*]
\item \(\phi^1(x)\le \phi^0(x)\) for all \(x\in\mathcal{X}\), with strict inequality for \(0<x<1\).
\item The map \(\psi\) is increasing on \(\mathcal{X}\) and satisfies \(\psi(x)\le x\) for all \(x\in\mathcal{X}\),
with strict inequality for \(0<x<1\). Hence \(\phi^1\) is increasing on \(\mathcal{X}\).
\item The interval \(\widetilde{\mathcal{X}}\triangleq[p_{01},p_{11}]\) is forward invariant under both maps:
\(\phi^a(\widetilde{\mathcal{X}})\subseteq\widetilde{\mathcal{X}}\) for \(a\in\{0,1\}\).
In particular, since \(p_{01}> 0\) and \(p_{11}=1-p_{10}<1\), we have \(\widetilde{\mathcal{X}}\subset(0,1)\).
\end{enumerate}
\end{remark}

A key relation between the belief-update maps \(\phi^0\) and \(\phi^1\) is the following
\emph{one-step posterior-mean identity}:
\begin{equation}
\label{eq:posterior_mean_identity}
\Ex\!\left[X(t+1)\mid X(t)=x,\ A(t)=0\right]
=\phi^0(x)
=\kappa x\,p_{11}+(1-\kappa x)\,\phi^1(x)
=\Ex\!\left[X(t+1)\mid X(t)=x,\ A(t)=1\right].
\end{equation}
Thus, although the action changes the \emph{distribution} of the next belief \(X(t+1)\), it leaves its
\emph{conditional mean} unchanged: under either action, the one-step mean drift equals \(\phi^0(x)\).
This one-step mean identity propagates over time and yields the following lemma.
Let \(\phi_t^0\) denote the \(t\)-fold iterate of \(\phi^0\), with \(\phi_0^0(x)\equiv x\).

\begin{lemma}[Policy-invariant mean belief path]
\label{lem:policy_invariant_mean_path}
Fix an initial belief \(X(0)=x\) and let \(\pi\in\Pi\) be any admissible policy.
Then the mean belief path is \emph{policy-invariant}: for every \(t\ge 0\),
\begin{equation}
\label{eq:mean_path_policy_invariant}
\Ex_x^\pi[X(t)] = \phi_t^0(x).
\end{equation}
\end{lemma}

The proof is deferred to Appendix~\ref{app:phi_props_proofs}.

Our analysis of no-ACK skeleton action itineraries draws on the regularity assumptions used in
\citep[Assumption~2]{danceSi19}. To state them independently of our model, let
\(h^a:\mathcal X\to\mathcal X\) denote a generic belief-update map under action \(a\in\{0,1\}\).

\begin{assumption}[Regularity of belief-update maps {\citep[Assumption~2]{danceSi19}}]
\label{ass:DSA2_generic}
\begin{enumerate}[label=\textup{(\roman*)},leftmargin=2.2em]
\item \textup{(Monotonicity)} \(h^0\) and \(h^1\) are increasing on \(\mathcal X\).
\item \textup{(Unique fixed points and ordering)} Each \(h^a\) admits a unique fixed point \(x^a\) in \(\mathcal X\), and
\(x^1<x^0\).
\item \textup{(Contractiveness up to increasing conjugacy)} There exists an increasing bijection
\(\vartheta:(0,1)\to\mathbb R\) such that the conjugated maps
\[
\hat h^{a}\triangleq \vartheta\circ h^{a}\circ \vartheta^{-1}:\mathbb R\to\mathbb R,\qquad a\in\{0,1\},
\]
are \emph{contractive} in the strict sense that
\[
|\hat h^{a}(u)-\hat h^{a}(v)|<|u-v|\qquad \text{for all distinct }u,v\in\mathbb R,\ \ a\in\{0,1\}.
\]
\end{enumerate}
\end{assumption}

In our model, we verify \textup{(i)}--\textup{(ii)} directly for \((h^0,h^1)=(\phi^0,\phi^1)\), and we verify
\textup{(iii)} with \(\vartheta\) taken to be the logit map, as established in Lemma~\ref{lem:logit_conjugacy_A2}.
The fact that \(\vartheta\) is defined on \((0,1)\) is immaterial here, since the forward-invariant interval
\(\widetilde{\mathcal X}=[p_{01},p_{11}]\subset(0,1)\) contains all no-ACK skeleton iterates from time \(1\) onward; see Remark~\ref{rem:phi01x}\textup{(iii)}.

We next verify the relevant parts of Assumption~\ref{ass:DSA2_generic} for our model.
We show that \(\phi^0\) and \(\phi^1\) are increasing on \(\mathcal X\), that each has a unique fixed point in \(\mathcal X\)
with \(x^1<x^0\), and we derive closed-form expressions for their iterates, thereby establishing
Assumption~\ref{ass:DSA2_generic}\textup{(i)}--\textup{(ii)}.
The map \(\phi^0\) is immediately seen to be contractive on \(\mathcal X\).
By contrast, \(\phi^1\) need not be contractive in the belief variable, but it becomes a strict contraction after the
standard increasing change of variables to log-odds.
This odds-contractiveness is sufficient for the itinerary machinery adapted from \citet{danceSi19}.

For \(a\in\{0,1\}\), let \(\phi_t^{a}(x)\) denote the \(t\)-fold iterate of \(\phi^{a}\), with \(\phi_0^{a}(x)\triangleq x\).
We refer to \(\bigl(\phi_t^{0}(x)\bigr)_{t\ge0}\) as the \emph{passive iterates} and to
\(\bigl(\phi_t^{1}(x)\bigr)_{t\ge0}\) as the \emph{active no-ACK iterates}.

We first consider passive iterates and threshold up-crossing times.
The map \(\phi^0(x)=p_{01} + \rho x\) is contractive and has the unique fixed point
\[
x^0=\frac{p_{01}}{1-\rho}\in(p_{01},p_{11}).
\]
Its iterates admit the closed form
\begin{equation}
\label{eq:phi0_iter}
\phi_t^{0}(x)=x^0+(x-x^0)\rho^t,\qquad t\ge0.
\end{equation}
In particular, \(\phi^0\) is increasing on \(\mathcal X\), verifying
Assumption~\ref{ass:DSA2_generic}\textup{(i)} for \(h^0=\phi^0\).
If \(x<x^0\), then \(\bigl(\phi_t^{0}(x)\bigr)_{t\ge0}\) increases to \(x^0\), whereas if \(x>x^0\), it decreases to \(x^0\).

For a threshold \(z<x^0\), define the \emph{passive threshold up-crossing time}
\begin{equation}
\label{eq:tau-characterization}
\tau_{\uparrow}^0(x,z)\triangleq \min \{t\ge 0: \phi_t^{0}(x)>z\}
=
\begin{cases}
1+\left\lfloor \dfrac{\ln\!\bigl((x^0-z)/(x^0-x)\bigr)}{\ln\rho} \right\rfloor, & x\le z,\\[1ex]
0, & x>z,
\end{cases}
\end{equation}
where we have used that \(\phi_t^0(x)>z\) is equivalent to \(\rho^t<(x^0-z)/(x^0-x)\) when \(x\le z<x^0\).

We next turn to the active no-ACK iterates. The fixed points of \(\phi^1\) satisfy the quadratic
\begin{equation}
\label{eq:phi1xxqe}
\kappa x^2-\bigl(1-\rho+\kappa p_{11}\bigr)x+ p_{01}=0.
\end{equation}

\begin{lemma}
\label{lem:phi1_fixed_point_location}
Equation \eqref{eq:phi1xxqe} has two distinct real roots \(x_{\mathrm{lo}}^1<x_{\mathrm{hi}}^1\), namely
\begin{equation}
\label{eq:phi1_fixed_points}
x_{\mathrm{lo}}^1
\triangleq
\frac{1-\rho+\kappa p_{11}-\sqrt{\Delta(\kappa)}}{2\kappa},
\qquad
x_{\mathrm{hi}}^1
\triangleq
\frac{1-\rho+\kappa p_{11}+\sqrt{\Delta(\kappa)}}{2\kappa},
\end{equation}
where
\begin{equation}
\label{eq:phi1_discriminant}
\Delta(\kappa)\triangleq
\bigl(1-\rho+\kappa p_{11}\bigr)^2-4\kappa p_{01}
=
\bigl((p_{10}+p_{01})-\kappa p_{11}\bigr)^2+4\kappa p_{10}\rho
>0.
\end{equation}
Moreover, \(x_{\mathrm{lo}}^1\in(p_{01},p_{11})\) and \(x_{\mathrm{hi}}^1\in(1,1/\kappa)\). Thus, \(\phi^1\) has a unique fixed
point in \(\mathcal X\), namely \(x^1\triangleq x_{\mathrm{lo}}^1\).
\end{lemma}

The proof is deferred to Appendix~\ref{app:phi_props_proofs}.

A useful identity involving these fixed points follows from Vieta's formulas:
\[
x_{\mathrm{lo}}^1+x_{\mathrm{hi}}^1=\frac{1-\rho+\kappa p_{11}}{\kappa},
\qquad
x_{\mathrm{lo}}^1x_{\mathrm{hi}}^1=\frac{p_{01}}{\kappa}.
\]
Therefore,
\begin{equation}
\label{eq:fixedpoint_identity}
(1-\kappa x_{\mathrm{lo}}^1)(1-\kappa x_{\mathrm{hi}}^1)=\rho(1-\kappa),
\end{equation}
and, letting \(x^1\triangleq x_{\mathrm{lo}}^1\),
\begin{equation}
\label{eq:mu_identity}
(\phi^1)'(x^1)
=\frac{\rho(1-\kappa)}{(1-\kappa x^1)^2}
=\frac{1-\kappa x_{\mathrm{hi}}^1}{1-\kappa x^1}.
\end{equation}

\begin{lemma}[Ordering of the active and passive fixed points]
\label{lem:phi_order_refined}
We have \(x^1<x^0\).
\end{lemma}

The proof is deferred to Appendix~\ref{app:phi_props_proofs}.

Lemmas~\ref{lem:phi1_fixed_point_location} and~\ref{lem:phi_order_refined} verify
Assumption~\ref{ass:DSA2_generic}\textup{(ii)} for \(h^0=\phi^0\) and \(h^1=\phi^1\).

We next record the Möbius structure of the active no-ACK update. The map \(\phi^1\) is a nondegenerate Möbius (fractional-linear) map, i.e.,
\(\phi^1(x)=(ax+b)/(cx+d)\) with \(ad-bc\neq 0\); see, e.g., \citep{ahlfors79,beardon00}. In particular,
\begin{equation}
\label{eq:phi1_mobius}
\phi^1(x)
=
p_{01} + \rho \,\frac{(1-\kappa)x}{1-\kappa x}
=
\frac{\bigl(\rho(1-\kappa)-\kappa p_{01}\bigr)x+ p_{01}}{1-\kappa x},
\qquad x \in \mathcal{X}.
\end{equation}
We use the natural extension of \(\phi^1\) to \([0,1/\kappa)\) to obtain closed forms for the iterates \(\phi_t^1\).

The following standard conjugacy identity for Möbius maps will be useful. Let \(\Phi(x)\triangleq (ax+b)/(cx+d)\) be a nondegenerate Möbius map with two distinct fixed points \(x_-<x_+\), and define
\begin{equation}
\label{eq:psi_pm_def}
\Psi(x)\triangleq \frac{x-x_-}{x-x_+},\qquad x\neq x_+.
\end{equation}
Then, with \(\mu\triangleq \Phi'(x_-)\neq 0\),
\begin{equation}
\label{eq:mobius_conjugacy}
\Psi(\Phi(x))=\mu\,\Psi(x)\qquad (x\neq x_+).
\end{equation}
Defining \(\Phi_0(x)\triangleq x\) and \(\Phi_t\triangleq \Phi\circ \Phi_{t-1}\) for \(t\ge1\), we obtain
\[
\Psi(\Phi_t(x))=\mu^t\Psi(x)\qquad (t\ge0,\ x\neq x_+).
\]
Moreover, \(\Psi'(x)=(x_- - x_+)/(x-x_+)^2<0\), so \(\Psi\) is decreasing on \(\mathbb{R}\setminus\{x_+\}\) and hence
invertible there, with
\[
\Psi^{-1}(y)=\frac{y\,x_+-x_-}{y-1}\qquad (y\neq 1).
\]
Therefore,
\begin{equation}
\label{eq:Phitx}
\Phi_t(x)=\Psi^{-1}\!\bigl(\mu^t\Psi(x)\bigr)
=\frac{\mu^t\Psi(x)\,x_+ - x_-}{\mu^t\Psi(x)-1},
\qquad t\ge0,\ x\neq x_+.
\end{equation}

Specializing now to \(\Phi=\phi^1\), the fixed points are \(x_-=x^1\) and \(x_+=x_{\mathrm{hi}}^1\) by Lemma~\ref{lem:phi1_fixed_point_location}. Thus,
\begin{equation}
\label{eq:Psix}
\Psi(x)\triangleq \frac{x-x^1}{x-x_{\mathrm{hi}}^1},
\qquad x\neq x_{\mathrm{hi}}^1.
\end{equation}
Since \(x_{\mathrm{hi}}^1>1\), \(\Psi\) is well defined on \(\mathcal X\).

\begin{lemma}[Closed form and monotone behavior of \(\phi_t^{1}(x)\)]
\label{lem:phi1_iterates}
The following holds:
\begin{enumerate}[label=\textup{(\alph*)},leftmargin=2em]
\item For every \(t\ge1\), the map \(x\mapsto \phi_t^{1}(x)\) is increasing.
\item For all \(x \in \mathcal{X}\),
\begin{align}
\label{eq:phi1_minus_x_factor}
\phi^1(x)-x
&=
-\frac{\kappa\,(x-x^1)(x_{\mathrm{hi}}^1-x)}{1-\kappa x},\\
\label{eq:phi1_minus_phiinfty_factor}
\phi^1(x)-x^1
&=
\frac{\rho (1-\kappa)}{1-\kappa x^1}\, \frac{x-x^1}{1-\kappa x}.
\end{align}
Consequently,
\[
x<x^1\ \Longleftrightarrow\ x<\phi^1(x)<x^1,
\qquad
x>x^1\ \Longleftrightarrow\ x^1<\phi^1(x)<x,
\]
and \(\bigl(\phi_t^{1}(x)\bigr)_{t\ge0}\) is strictly monotone in \(t\) and converges to \(x^1\) for \(x\neq x^1\).
\item Let
\begin{equation}
\label{eq:mu_def}
\mu\triangleq (\phi^1)'(x^1)
=\frac{\rho(1-\kappa)}{(1-\kappa x^1)^2}
=\frac{1-\kappa x_{\mathrm{hi}}^1}{1-\kappa x^1}
\in(0,1).
\end{equation}
Then for all \(t\ge0\) and \(x \in \mathcal{X}\),
\begin{equation}
\label{eq:phi1_iter_closed}
\phi_t^{1}(x)
=
\frac{x^1(x-x_{\mathrm{hi}}^1)-x_{\mathrm{hi}}^1\,\mu^t(x-x^1)}{(x-x_{\mathrm{hi}}^1)-\mu^t(x-x^1)}
= x^1 +
\frac{\mu^t (x_{\mathrm{hi}}^1-x^1)\,(x-x^1)}
{(x_{\mathrm{hi}}^1-x^1)+(\mu^t-1)\,(x-x^1)},
\end{equation}
and \(\phi_t^{1}(x)\to x^1\) as \(t\to\infty\) with geometric rate \(\mu\).
\end{enumerate}
\end{lemma}

The proof is deferred to Appendix~\ref{app:phi_props_proofs}.

Finally, we address contractiveness. The passive map \(\phi^0\) is contractive on \(\mathcal X\):
\[
|\phi^0(x)-\phi^0(y)|=\rho\,|x-y|<|x-y|\qquad (x\neq y),
\]
verifying Assumption~\ref{ass:DSA2_generic}\textup{(iii)} for \(h^0=\phi^0\). The active no-ACK map \(\phi^1\) need not be
contractive on \(\mathcal X\) in the belief variable when \(\kappa\) is close to \(1\) (indeed, \(\sup_{x\in\mathcal X}(\phi^1)'(x)\)
may exceed \(1\)). We therefore use an increasing change of variables to log-odds to ensure contractiveness.

\subsubsection{Conjugacy to increasing contractions}
\label{s:logit_contraction}

The next lemma shows that our belief-update maps satisfy the contractiveness part of
Assumption~\ref{ass:DSA2_generic}\textup{(iii)} after an increasing change of variables, in the sense used in
\citet{danceSi19}.

\begin{lemma}[Logit conjugacy yields increasing contractions]
\label{lem:logit_conjugacy_A2}
Let \(\vartheta:(0,1)\to\mathbb R\) be the logit map
\[
\vartheta(x)\triangleq \log\frac{x}{1-x},
\qquad
\vartheta^{-1}(t)=\frac{e^t}{1+e^t}.
\]
For \(a\in\{0,1\}\), define the conjugated maps \(\hat \phi^{a}\triangleq \vartheta\circ \phi^{a}\circ \vartheta^{-1}\) on \(\mathbb R\).
Then the following hold:
\begin{enumerate}[label=\textup{(\alph*)},leftmargin=2.0em]
\item \(\hat \phi^{0}\) and \(\hat \phi^{1}\) are increasing and contractive on \(\mathbb R\);
\item each \(\hat \phi^{a}\) has the unique fixed point \(\hat x^a=\vartheta(x^a)\);
\item \(\hat x^1<\hat x^0\).
\end{enumerate}
\end{lemma}

The proof is deferred to Appendix~\ref{app:phi_props_proofs}.

\subsection{No-ACK skeleton dynamics and survival probabilities}
\label{s:skeleton_survival}

A central ingredient in our PCL-indexability analysis is an explicit description of the belief trajectories induced by threshold policies, since the performance metrics \(F\), \(G\), \(f\), and \(g\) can be expressed in terms of deterministic pre-ACK evolution together with ACK-triggered restarts. Fix a threshold \(z\in\mathbb{R}\) and an initial belief \(x\in\mathcal X\). Under the \(z\)-threshold policy, the action at time \(t\) is \(A(t)=\1_{\{X(t)>z\}}\). The resulting belief and action sequences,
\[
\orbit(x,z)\triangleq (X(t))_{t\ge0},
\qquad
\sigma(x,z)\triangleq (A(t))_{t\ge0},
\]
are called, respectively, the \emph{orbit} and \emph{itinerary}. Their evolution combines deterministic updates via \(\phi^0\) and \(\phi^1\) between ACK times with random resets to \(p_{11}\) upon an ACK. We next isolate the associated deterministic no-ACK skeleton, generated by a \emph{map-with-a-gap}, together with the corresponding no-ACK survival probabilities \(\Gamma_t(x,z)\).

\subsubsection{The no-ACK skeleton as an iterated map-with-a-gap}
\label{s:noackskimwag}

To separate deterministic evolution from random ACK resets, define the \emph{no-ACK skeleton} map
\(\varphi(\cdot,z):\mathcal X\to\mathcal X\) under the \(z\)-threshold policy by
\[
\varphi(x,z)\triangleq
\begin{cases}
\phi^1(x), & x>z,\\[1mm]
\phi^0(x), & x\le z,
\end{cases}
\]
which is generally discontinuous at \(z\); see Remark~\ref{rem:phi01x}\textup{(i)}. For each fixed \(z\), define its iterates by
\begin{equation}
\label{eq:varphi_iterates_def}
\varphi_0(x,z)\triangleq x,\qquad
\varphi_t(x,z)\triangleq \varphi\!\bigl(\varphi_{t-1}(x,z),z\bigr),\qquad t\ge1.
\end{equation}
The associated deterministic skeleton orbit and itinerary are
\begin{equation}
\label{eq:skeleton_orbit_itinerary_via_iterates}
\widetilde{\orbit}(x,z)\triangleq(\widetilde X_t(x,z))_{t\ge0},\ \ \widetilde X_t(x,z)\triangleq \varphi_t(x,z),
\qquad
\widetilde{\sigma}(x,z)\triangleq(\widetilde A_t(x,z))_{t\ge0},\ \ \widetilde A_t(x,z)\triangleq \1_{\{\widetilde X_t(x,z)>z\}}.
\end{equation}
Thus \((\widetilde X_t,\widetilde A_t)\) is the trajectory obtained by iterating the map-with-a-gap under the convention that no ACK is ever observed. The actual stochastic trajectory agrees with this skeleton up to the first ACK time; when an ACK occurs, the belief resets to \(p_{11}\) and the same deterministic rule restarts from that new initial state.

We write \(\varphi_t(x,z)\) when emphasizing the iterated map, and \(\widetilde X_t(x,z)\) and \(\widetilde A_t(x,z)\) when emphasizing the resulting deterministic skeleton state trajectory and itinerary.

\subsubsection{Pre-ACK representation and survival probabilities}
\label{s:preackrepsurvp}

Fix \((x,z)\) and let
\[
\tau^{\mathrm{ack}}\triangleq \inf\{t\ge0:\ \text{an ACK is observed at time }t\},
\]
with the convention \(\tau^{\mathrm{ack}}=\infty\) if no ACK is ever observed. On the event \(\{\tau^{\mathrm{ack}}\ge t\}\), that is, when no ACK occurs in periods \(0,\ldots,t-1\), the stochastic trajectory agrees with the skeleton up to and including time \(t\):
\begin{equation}
\label{eq:traj_equals_skeleton_preACK}
\{\tau^{\mathrm{ack}}\ge t\}\subseteq \bigl\{(X(j),A(j))=(\widetilde X_j(x,z),\widetilde A_j(x,z)),\ j=0,\ldots,t\bigr\}.
\end{equation}
Define the \emph{no-ACK survival probabilities} by
\begin{equation}
\label{eq:Gammatxz}
\Gamma_t(x,z)\triangleq \Prob_x^z\{\tau^{\mathrm{ack}}\ge t\},\qquad t\ge0.
\end{equation}
Since, conditional on \(\{\tau^{\mathrm{ack}}\ge t-1\}\), an ACK occurs at time \(t-1\) with probability
\[
\kappa\,\widetilde X_{t-1}(x,z)\,\widetilde A_{t-1}(x,z),
\]
we obtain the recursion
\[
\Gamma_0(x,z)=1,\qquad
\Gamma_t(x,z)=
\Gamma_{t-1}(x,z)\Bigl(1-\kappa\,\widetilde X_{t-1}(x,z)\,\widetilde A_{t-1}(x,z)\Bigr),
\qquad t\ge1,
\]
and hence the product representation
\begin{equation}
\label{eq:GammatxzProd}
\Gamma_t(x,z)
=\prod_{j=0}^{t-1}\Bigl(1-\kappa\,\widetilde X_j(x,z)\,\widetilde A_j(x,z)\Bigr)
=\prod_{0\le j<t:\ \widetilde X_j(x,z)>z}\bigl(1-\kappa\,\widetilde X_j(x,z)\bigr),
\qquad t\ge1.
\end{equation}
Finally, let
$
\Gamma_\infty(x,z)\triangleq \lim_{t\to\infty}\Gamma_t(x,z)=\Prob_x^z\{\tau^{\mathrm{ack}}=\infty\},
$
where the limit exists because \((\Gamma_t(x,z))_{t\ge0}\) is nonincreasing and bounded below by \(0\).
These deterministic quantities will be used below to derive renewal formulas for \(F(x,z)\) and \(G(x,z)\) in the intermediate threshold regimes.

The \(z\)-threshold policy activates at time \(t\) if and only if \(X(t)>z\), and its no-ACK skeleton is generated by \(\varphi(\cdot,z)\) above. It is also convenient to consider the \emph{\(z^{-}\)-threshold policy}, which activates if and only if \(X(t)\ge z\). Its no-ACK skeleton is generated by the map \(\varphi(\cdot,z^-)\), defined by
\[
\varphi(x,z^-)\triangleq
\begin{cases}
\phi^1(x), & x\ge z,\\[1mm]
\phi^0(x), & x< z,
\end{cases}
\qquad x\in\mathcal X,
\]
and iterated as in \eqref{eq:varphi_iterates_def}. The notation \(\varphi(x,z^-)\) reflects the fact that, for each fixed \(x\in\mathcal X\),
$
\varphi(x,z^-)=\lim_{\zeta \nearrow z}\varphi(x,\zeta),
$
that is, it is the left limit of the map \(\zeta\mapsto \varphi(x,\zeta)\) at threshold \(z\). We write
\(\widetilde{\sigma}(x,z^-)\) for the resulting skeleton itinerary.

\subsection{Renewal decompositions for \(F(x,z)\) and \(G(x,z)\)}
\label{s:renewal_FG}

We next derive renewal-type identities expressing the threshold-policy reward and work metrics \(F(x,z)\) and \(G(x,z)\) in terms of deterministic \emph{pre-ACK} quantities computed along the no-ACK skeleton and the restart state \(p_{11}\) induced by an ACK. The key observation is that, under a threshold policy, the stochastic trajectory coincides with the skeleton up to the first ACK time, at which point the belief resets to \(p_{11}\) and the same evolution restarts from that state. This reduces the computation of \(F\) and \(G\) to two ingredients: discounted sums along the no-ACK skeleton, weighted by no-ACK survival probabilities, and a continuation term evaluated at the restart state. The resulting pre-ACK metrics admit convergent series representations that can be evaluated numerically by truncation with explicit a priori tail bounds, yielding controllable approximation errors and stable computation.

For any threshold \(z\in\R\), the reward and work metrics under the \(z\)-threshold policy (see \eqref{eq:FxGxpi}) are the unique bounded solutions of
\begin{equation}
\label{eq:fseveq}
F(x,z)=
\begin{cases}
r\kappa x+\beta \kappa x\,F(p_{11},z)+\beta(1-\kappa x)\,F(\phi^1(x),z), & x>z, \\[1mm]
\beta\,F(\phi^0(x),z), & x\le z,
\end{cases}
\end{equation}
and
\begin{equation}
\label{eq:gseveq}
G(x,z)=
\begin{cases}
1+\beta \kappa x\,G(p_{11},z)+\beta(1-\kappa x)\,G(\phi^1(x),z), & x>z, \\[1mm]
\beta\,G(\phi^0(x),z), & x\le z.
\end{cases}
\end{equation}

To isolate the pre-ACK contribution, define the discounted \emph{pre-ACK reward} and \emph{pre-ACK work} metrics by
\begin{equation}
\label{eq:tilde_series}
\widetilde F(x,z)\triangleq r\kappa\sum_{t=0}^{\infty}\beta^t\,\Gamma_t(x,z)\,\widetilde X_t(x,z)\,\widetilde A_t(x,z),
\qquad
\widetilde G(x,z)\triangleq \sum_{t=0}^{\infty}\beta^t\,\Gamma_t(x,z)\,\widetilde A_t(x,z).
\end{equation}
These quantities are well defined, since \(0\le \Gamma_t(x,z)\le 1\) and \(\widetilde A_t(x,z)\in\{0,1\}\), and they depend only on the no-ACK skeleton \((\widetilde X_t,\widetilde A_t)\) and the associated survival probabilities \(\Gamma_t(x,z)\).

For \(T\in\mathbb Z_+\), define the truncated approximations
\[
\widetilde F_T(x,z)\triangleq r\kappa\sum_{t=0}^{T-1}\beta^t\,\Gamma_t(x,z)\,\widetilde X_t(x,z)\,\widetilde A_t(x,z),
\qquad
\widetilde G_T(x,z)\triangleq \sum_{t=0}^{T-1}\beta^t\,\Gamma_t(x,z)\,\widetilde A_t(x,z).
\]
Since \(0\le \Gamma_t\le 1\), \(0\le \widetilde X_t\le 1\), and \(\widetilde A_t\in\{0,1\}\), the tails satisfy the uniform bounds
\begin{equation}
\label{eq:tilde_tail_bounds}
0\le \widetilde G(x,z)-\widetilde G_T(x,z)\le \sum_{t=T}^{\infty}\beta^t=\frac{\beta^T}{1-\beta},
\qquad
0\le \widetilde F(x,z)-\widetilde F_T(x,z)\le r\kappa\sum_{t=T}^{\infty}\beta^t=\frac{r\kappa\,\beta^T}{1-\beta}.
\end{equation}
Thus, given a tolerance \(\varepsilon>0\), choosing \(T\) so that \(\beta^T/(1-\beta)\le \varepsilon\) yields an a priori
\(\varepsilon\)-accurate evaluation of \(\widetilde G\), while choosing \(T\) so that
\(r\kappa\,\beta^T/(1-\beta)\le \varepsilon\) yields an a priori \(\varepsilon\)-accurate evaluation of \(\widetilde F\).

For later use, define the discounted transform of the first ACK time \(\tau^{\mathrm{ack}}\) by
\begin{equation}
\label{eq:tildeTheta_preACK_def}
\widetilde{\Theta}(x,z)\triangleq \Ex_x^z[\beta^{\tau^{\mathrm{ack}}}]
=\sum_{t=0}^{\infty}\beta^t\,\Prob_x^z\{\tau^{\mathrm{ack}}=t\},
\end{equation}
with the convention \(\beta^\infty\triangleq 0\).
Since \(\{\tau^{\mathrm{ack}}=t\}\) means that no ACK occurs before period \(t\) and an ACK occurs in period \(t\), we have
\[
\Prob_x^z\{\tau^{\mathrm{ack}}=t\}
=\Gamma_t(x,z)\,\kappa\,\widetilde X_t(x,z)\,\widetilde A_t(x,z),\qquad t\ge0.
\]
A key simplification, established in Lemma~\ref{lem:renewal_first_ACK}\textup{(a)}, is
\begin{equation}
\label{eq:ThetaF_relation_preACK}
\widetilde F(x,z)=r\,\widetilde{\Theta}(x,z).
\end{equation}
Hence the tail bound for \(\widetilde F\) in \eqref{eq:tilde_tail_bounds} immediately yields an a priori tail bound for
\(\widetilde{\Theta}\) via the identity \(\widetilde{\Theta}(x,z)=\widetilde F(x,z)/r\).

We now decompose \(F(x,z)\) and \(G(x,z)\) at the first ACK time. Let \(\tau^{\mathrm{ack}}\) be the first ACK time under the \(z\)-threshold policy, with
\(\tau^{\mathrm{ack}}=\infty\) if no ACK ever occurs. An ACK regenerates the process by resetting the belief to \(p_{11}\) at the next period, so \(F(x,z)\) and \(G(x,z)\) split into a pre-ACK contribution plus a discounted continuation value from the post-ACK state \(p_{11}\). The next lemma makes this precise.

\begin{lemma}[Renewal decomposition at the first ACK time]
\label{lem:renewal_first_ACK}
Fix \(z\in\R\) and \(x\in\mathcal X\). Then:
\begin{enumerate}[label=\textup{(\alph*)},leftmargin=*]
\item Pre-ACK representations via the no-ACK skeleton.
\[
\Ex_x^z\!\bigg[\sum_{t=0}^{\tau^{\mathrm{ack}}}\beta^t\,r\kappa\,X(t)A(t)\bigg]
= \widetilde F(x,z),
\qquad
\Ex_x^z\!\bigg[\sum_{t=0}^{\tau^{\mathrm{ack}}}\beta^t\,A(t)\bigg]
= \widetilde G(x,z),
\qquad
\widetilde{\Theta}(x,z)=\frac{1}{r}\,\widetilde F(x,z).
\]

\item Renewal decomposition.
\[
F(x,z)=\widetilde F(x,z)+\beta\,\widetilde{\Theta}(x,z)\,F(p_{11},z),
\qquad
G(x,z)=\widetilde G(x,z)+\beta\,\widetilde{\Theta}(x,z)\,G(p_{11},z).
\]

\item Post-ACK fixed-point evaluation.
We have \(0\le \beta\,\widetilde{\Theta}(p_{11},z)<1\), and
\begin{equation}
\label{eq:FGrho_closed}
F(p_{11},z)=\frac{\widetilde F(p_{11},z)}{1-\beta\,\widetilde{\Theta}(p_{11},z)}
=\frac{r\,\widetilde{\Theta}(p_{11},z)}{1-\beta\,\widetilde{\Theta}(p_{11},z)},
\qquad
G(p_{11},z)=\frac{\widetilde G(p_{11},z)}{1-\beta\,\widetilde{\Theta}(p_{11},z)}.
\end{equation}
\end{enumerate}
\end{lemma}

The proof is deferred to Appendix~\ref{app:renewal_proofs}.

Lemma~\ref{lem:renewal_first_ACK} expresses \(F(x,z)\) and \(G(x,z)\) in terms of the pre-ACK quantities
\(\widetilde F(x,z)\), \(\widetilde G(x,z)\), and the discounted transform \(\widetilde{\Theta}(x,z)\) of the first ACK time \(\tau^{\mathrm{ack}}\). In several threshold regimes these objects simplify substantially---for example, when \(\tau^{\mathrm{ack}}\) is a.s.\ finite, when \(\tau^{\mathrm{ack}}=\infty\) a.s., or when the no-ACK skeleton permits only finitely many activations---yielding especially transparent closed forms. We exploit these simplifications next.

\subsection{Tractable threshold regimes and regime-wise formulas}
\label{s:tractable_regimes}

The no-ACK skeleton and the renewal identities provide a general route to evaluating
\(F(x,z)\), \(G(x,z)\), and the marginal metrics \(f(x,z)\) and \(g(x,z)\).
Several threshold regimes, however, admit substantial simplifications.
In this subsection we classify the corresponding skeleton itineraries, determine the finiteness behavior of the first ACK time
\(\tau^{\mathrm{ack}}\), and record the resulting regime-wise formulas for \(F(x,z)\) and \(G(x,z)\).

We use the standard word notation for binary itineraries:
for \(a\in\{0,1\}\) and \(m\in\mathbb Z_+\), \(a^m\) denotes a block of \(m\) repeated symbols \(a\), and
\(a^\infty\) denotes the infinite constant sequence \((a,a,\ldots)\).
Concatenation is written by juxtaposition, so \(01^\infty\) denotes \((0,1,1,\ldots)\), and more generally
\(0^{t_0}1^\infty\) denotes \(t_0\) consecutive zeros followed by ones forever.

\begin{lemma}[Skeleton itineraries by threshold regime]
\label{lem:threshold_regime_dynamics}
Fix \(x\in\mathcal X\) and \(z\in\R\).
The no-ACK skeleton itinerary \(\widetilde{\sigma}(x,z)\) behaves as follows:
\begin{enumerate}[label=\textup{(\alph*)},leftmargin=*]
\item If \(z<p_{01}\), then
$
\widetilde{\sigma}(x,z)=
\begin{cases}
1^\infty, & x>z,\\
01^\infty, & x\le z.
\end{cases}
$

\item If \(p_{01}\le z\le x^1\), then, with \(t_0\triangleq \tau_{\uparrow}^0(x,z)\),
$
\widetilde{\sigma}(x,z)=0^{t_0}1^\infty.
$

\item If \(x^1<z<x^0\), then \(\widetilde{\sigma}(x,z)\) alternates indefinitely between finite blocks of \(1\)'s and finite blocks of \(0\)'s.

\item If \(x^0\le z<p_{11}\), then
$
\widetilde{\sigma}(x,z)=
\begin{cases}
0^\infty, & x\le z,\\
1^{t_1}0^\infty, & x>z,
\end{cases}
\qquad
t_1\triangleq \tau_{\downarrow}^1(x,z)\ \text{when }x>z.
$

\item If \(z\ge p_{11}\), then
$
\widetilde{\sigma}(x,z)=
\begin{cases}
0^\infty, & x\le z,\\
10^\infty, & x>z.
\end{cases}
$
\end{enumerate}
\end{lemma}

The proof is deferred to Appendix~\ref{app:tractable_regimes_proofs}.

The itinerary classification above determines whether the first ACK time \(\tau^{\mathrm{ack}}\) is almost surely finite,
almost surely infinite, or defective, and hence governs the form of the renewal identities for \(F(x,z)\) and \(G(x,z)\).

\begin{lemma}[Finiteness regimes of \(\tau^{\mathrm{ack}}\)]
\label{lem:tauack_finite_alternating}
Fix \(z\in\R\) and \(x\in\mathcal X\), and let \(\tau^{\mathrm{ack}}\) be the first ACK time under the \(z\)-threshold policy. Then:
\begin{enumerate}[label=\textup{(\alph*)},leftmargin=*]
\item If \(z<x^0\), then \(\Prob_x^z\{\tau^{\mathrm{ack}}<\infty\}=1\).

\item If \(x^0\le z<p_{11}\), then the passive region \([0,z]\) is trapping:
\begin{itemize}
\item if \(x\le z\), then \(\tau^{\mathrm{ack}}=\infty\) a.s.;

\item if \(x>z\), then an ACK can occur only during the finite time window
\(0\le t<\tau_{\downarrow}^1(x,z)\), and
\begin{equation}
\label{eq:PTinfty_noreturn}
\Prob_x^z\{\tau^{\mathrm{ack}}=\infty\}
=\prod_{j=0}^{\tau_{\downarrow}^1(x,z)-1}\bigl(1-\kappa \phi_j^1(x)\bigr)\in(0,1).
\end{equation}
\end{itemize}

\item If \(z\ge p_{11}\), then activation is possible only at time \(t=0\), so
$
\Prob_x^z\{\tau^{\mathrm{ack}}<\infty\}=\kappa x\,\1_{\{x>z\}}.
$
\end{enumerate}
\end{lemma}

The proof is deferred to Appendix~\ref{app:tractable_regimes_proofs}.

\begin{proposition}[Piecewise constancy of the pre-ACK metrics in the tractable regimes]
\label{pro:tilde_metrics_piecewise_constant_tractable}
Fix \(x\in\mathcal X\). The maps
\[
z\longmapsto \widetilde F(x,z),\qquad
z\longmapsto \widetilde G(x,z),\qquad
z\longmapsto \widetilde\Theta(x,z)
\]
are piecewise constant on the tractable threshold regimes
$
(-\infty,x^1]\cup[x^0,\infty).
$
More precisely:
\begin{enumerate}[label=\textup{(\alph*)},leftmargin=*]
\item \textbf{Regime \(z<p_{01}\).}
The three metrics are constant on each of the intervals
$(-\infty,\min\{x,p_{01}\})$ and 
$[x,p_{01})\cap(-\infty,p_{01})$.
On the first interval the skeleton itinerary is \(1^\infty\), and on the second interval (when nonempty) it is
\(01^\infty\).

\item \textbf{Regime \(p_{01}\le z\le x^1\).}
For each \(n\in\mathbb Z_+\), let
$
I_n^\uparrow(x)\triangleq
\{z\in[p_{01},x^1]:\ \tau_\uparrow^0(x,z)=n\}.
$
Then \(I_n^\uparrow(x)\) is an interval (possibly empty), and on each nonempty interval \(I_n^\uparrow(x)\) one has
$
\widetilde{\sigma}(x,z)=0^n1^\infty.
$
Consequently, \(\widetilde F(x,z)\), \(\widetilde G(x,z)\), and \(\widetilde\Theta(x,z)\) are constant on each nonempty
interval \(I_n^\uparrow(x)\).

\item \textbf{Regime \(x^0\le z<p_{11}\).}
If \(x\le z\), then
$
\widetilde F(x,z)=\widetilde G(x,z)=\widetilde\Theta(x,z)=0.
$

If \(x>z\), then for each \(n\in\mathbb N\), let
$
I_n^\downarrow(x)\triangleq
\{z\in[x^0,\min\{x,p_{11}\}):\ \tau_\downarrow^1(x,z)=n\}.
$
Then \(I_n^\downarrow(x)\) is an interval (possibly empty), and on each nonempty interval \(I_n^\downarrow(x)\) one has
$
\widetilde{\sigma}(x,z)=1^n0^\infty.
$
Consequently, \(\widetilde F(x,z)\), \(\widetilde G(x,z)\), and \(\widetilde\Theta(x,z)\) are constant on each nonempty
interval \(I_n^\downarrow(x)\).

\item \textbf{Regime \(z\ge p_{11}\).}
The three metrics are constant on each of the intervals
$[p_{11},x)\cap[p_{11},\infty)$ and 
$[\max\{x,p_{11}\},\infty)$.
On the first interval (when nonempty) the skeleton itinerary is \(10^\infty\), and on the second interval it is
\(0^\infty\).
\end{enumerate}
\end{proposition}

The proof is deferred to Appendix~\ref{app:tractable_regimes_proofs}.

Building on the regime classification of the no-ACK skeleton and the renewal identities at the first ACK time,
we next record the corresponding formulas for \(F(x,z)\) and \(G(x,z)\).

\begin{proposition}[Reward/work metrics under a \(z\)-threshold policy: regime-wise formulas]
\label{pro:FG_closed_form_with_proof}
Fix \(x\in\mathcal X\). On the tractable threshold regimes
$(-\infty,x^1]\cup[x^0,\infty),$
the functions \(F(x,\cdot)\) and \(G(x,\cdot)\) are right-continuous step functions of \(z\), possibly with infinitely many jumps, and any discontinuity can occur only at a threshold \(z\) that coincides with a belief value visited along a no-ACK skeleton segment started either from \(x\) or from the post-ACK reset belief \(p_{11}\).

Moreover, \(F(x,z)\) and \(G(x,z)\) admit the following regime-wise expressions:

\begingroup
\renewcommand{\labelenumi}{(\alph{enumi})}
\begin{enumerate}

\item \textbf{Regime \(z<p_{01}\).}
\[
F(x,z)=
\begin{cases}
r\kappa\left[\dfrac{x^0}{1-\beta}-\dfrac{x^0-x}{1-\beta\rho}\right], & x>z,\\[2mm]
\beta\,r\kappa\left[\dfrac{x^0}{1-\beta}-\dfrac{x^0-\phi^0(x)}{1-\beta\rho}\right], & x\le z,
\end{cases}
\qquad
G(x,z)=
\begin{cases}
\dfrac{1}{1-\beta}, & x>z,\\[1mm]
\dfrac{\beta}{1-\beta}, & x\le z.
\end{cases}
\]

\item \textbf{Regime \(p_{01}\le z\le x^1\).}
\[
F(x,z)= r\kappa\,\beta^{\tau_{\uparrow}^0(x,z)}
\left[\frac{x^0}{1-\beta}-\frac{x^0-\phi^0_{\tau_{\uparrow}^0(x,z)}(x)}{1-\beta\rho}\right],
\qquad
G(x,z)=\frac{\beta^{\tau_{\uparrow}^0(x,z)}}{1-\beta}.
\]

\item \textbf{Regime \(x^1<z<x^0\).}
By Lemma~\ref{lem:renewal_first_ACK},
\begin{align*}
F(x,z)
&=\widetilde F(x,z)+\beta\,\widetilde{\Theta}(x,z)\,F(p_{11},z), \qquad G(x,z)
=\widetilde G(x,z)+\beta\,\widetilde{\Theta}(x,z)\,G(p_{11},z),
\end{align*}
with \(\widetilde F,\widetilde G\) defined by \eqref{eq:tilde_series}. Moreover,
\[
F(p_{11},z)=\frac{\widetilde F(p_{11},z)}{1-\beta\,\widetilde{\Theta}(p_{11},z)},
\qquad
G(p_{11},z)=\frac{\widetilde G(p_{11},z)}{1-\beta\,\widetilde{\Theta}(p_{11},z)}.
\]

\item \textbf{Regime \(x^0\le z<p_{11}\).}
For fixed \(x\), the functions \(F(x,z)\) and \(G(x,z)\) are piecewise constant in \(z\).

If \(x\le z\), then
$
F(x,z)=G(x,z)=0.
$

Assume \(x>z\). For \(n,m\in\mathbb N\), define the intervals
\[
I_{n,m}(x)\triangleq
\Bigl\{z\in[x^0,\min\{x,p_{11}\}) : \tau_{\downarrow}^1(x,z)=n,\ \tau_{\downarrow}^1(p_{11},z)=m\Bigr\}.
\]
Then each \(I_{n,m}(x)\) is an interval (possibly empty), and \(F(x,z)\) and \(G(x,z)\) are constant on each nonempty
 \(I_{n,m}(x)\).

More explicitly, if \(z\in I_{n,m}(x)\), define
\begin{align*}
A_x^{(n)}
&\triangleq
\sum_{t=0}^{n-1}\beta^t
\prod_{j=0}^{t-1}\bigl(1-\kappa\,\phi_j^1(x)\bigr),\\
B_x^{(n)}
&\triangleq
\kappa\sum_{t=0}^{n-1}\beta^t
\prod_{j=0}^{t-1}\bigl(1-\kappa\,\phi_j^1(x)\bigr)\,\phi_t^1(x),
\end{align*}
and similarly
\begin{align*}
A_{11}^{(m)}
&\triangleq
\sum_{t=0}^{m-1}\beta^t
\prod_{j=0}^{t-1}\bigl(1-\kappa\,\phi_j^1(p_{11})\bigr),\\
B_{11}^{(m)}
&\triangleq
\kappa\sum_{t=0}^{m-1}\beta^t
\prod_{j=0}^{t-1}\bigl(1-\kappa\,\phi_j^1(p_{11})\bigr)\,\phi_t^1(p_{11}).
\end{align*}
Then
$
\widetilde G(x,z)=A_x^{(n)},
\qquad
\widetilde\Theta(x,z)=B_x^{(n)},
\qquad
\widetilde F(x,z)=r\,B_x^{(n)},
$
and
\[
F(p_{11},z)=\frac{r\,B_{11}^{(m)}}{1-\beta\,B_{11}^{(m)}},
\qquad
G(p_{11},z)=\frac{A_{11}^{(m)}}{1-\beta\,B_{11}^{(m)}}.
\]
Consequently,
\[
F(x,z)=\frac{r\,B_x^{(n)}}{1-\beta\,B_{11}^{(m)}}, \quad
G(x,z)
=
A_x^{(n)}+\frac{\beta\,B_x^{(n)}\,A_{11}^{(m)}}{1-\beta\,B_{11}^{(m)}}
=
\frac{A_x^{(n)}\bigl(1-\beta\,B_{11}^{(m)}\bigr)+\beta\,B_x^{(n)}\,A_{11}^{(m)}}
{1-\beta\,B_{11}^{(m)}}.
\]

\item \textbf{Regime \(z\ge p_{11}\).}
$
F(x,z)=r\kappa x\,\1_{\{x>z\}}, \enspace
G(x,z)=\1_{\{x>z\}}.
$

\end{enumerate}
\endgroup
\end{proposition}

The proof is deferred to Appendix~\ref{app:tractable_regimes_proofs}.

\begin{lemma}[Monotonicity of the work metric in the tractable regimes]
\label{lem:G_monotone_tractable}
Fix \(z\in(-\infty,x^1]\cup[x^0,\infty)\). Then the function \(x\mapsto G(x,z)\) is nondecreasing on \(\mathcal X\).
\end{lemma}

The proof is deferred to Appendix~\ref{app:tractable_regimes_proofs}.

Together, the preceding lemmas and propositions provide the regime-wise formulas used later for numerical evaluation of
\(F\), \(G\), and the MP index, and for verifying the PCL conditions.

\subsection{Computation of marginal metrics}
\label{s:ia-fg}

We next turn to the marginal reward and work metrics \(f(x,z)\) and \(g(x,z)\) associated with a one-step deviation from a \(z\)-threshold policy. These marginal quantities are the  ingredients of the MP index 
($
m(x,z) \triangleq {f(x,z)}/{g(x,z)},
m(x) \triangleq m(x,x),
$)
which equals the Whittle index if the PCL conditions hold. Using the renewal decompositions and pre-ACK series developed above, we express \(f\) and \(g\) in terms of pre-ACK contributions and continuation values from the post-ACK reset state, yielding formulas that are efficient to evaluate and admit controlled truncation error.

For \(a\in\{0,1\}\) and threshold \(z\), let \(\langle a,z\rangle\) denote the policy that takes action \(a\) at time \(0\) and follows the \(z\)-threshold policy from time \(1\) onward. Let \(\tau^{\mathrm{ack}}\) denote the corresponding first ACK time under that policy. Define the associated pre-ACK reward, work, and transform metrics by (with the convention $\beta^\infty\triangleq 0$)
\begin{align*}
\widetilde F(x,\langle a,z\rangle)
&\triangleq
\Ex_x^{\langle a,z\rangle}\!\bigg[\sum_{t=0}^{\tau^{\mathrm{ack}}}\beta^t\,r\kappa\,X(t)A(t)\bigg], \quad
\widetilde G(x,\langle a,z\rangle)
\triangleq
\Ex_x^{\langle a,z\rangle}\!\bigg[\sum_{t=0}^{\tau^{\mathrm{ack}}}\beta^t\,A(t)\bigg], \quad
\widetilde\Theta(x,\langle a,z\rangle)
\triangleq
\Ex_x^{\langle a,z\rangle}\!\big[\beta^{\tau^{\mathrm{ack}}}\big].
\end{align*}

We then define the \emph{pre-ACK marginal} metrics as the differences between active and passive initializations:
\begin{align}
\label{eq:tildef_one_step_new}
\tilde f(x,z)
&\triangleq \widetilde F(x,\langle 1,z\rangle)-\widetilde F(x,\langle 0,z\rangle) = r\kappa x+\beta\Big((1-\kappa x)\,\widetilde F(\phi^1(x),z)-\widetilde F(\phi^0(x),z)\Big),\\
\label{eq:tildeg_one_step_new}
\tilde g(x,z)
&\triangleq \widetilde G(x,\langle 1,z\rangle)-\widetilde G(x,\langle 0,z\rangle)  = 1+\beta\Big((1-\kappa x)\,\widetilde G(\phi^1(x),z)-\widetilde G(\phi^0(x),z)\Big),\\
\label{eq:tildetheta_one_step_new}
\tilde\theta(x,z)
&\triangleq \widetilde\Theta(x,\langle 1,z\rangle)-\widetilde\Theta(x,\langle 0,z\rangle)  = \kappa x+\beta\Big((1-\kappa x)\,\widetilde\Theta(\phi^1(x),z)-\widetilde\Theta(\phi^0(x),z)\Big),
\end{align}
where the identities on the right follow from one-step conditioning.

\begin{lemma}[Decomposition of marginal metrics via pre-ACK metrics]
\label{lem:fg_decomp_preACK}
For every \(x\in\mathcal X\) and \(z\in\R\), the marginal reward and work metrics satisfy
\begin{align}
\label{eq:f_decomp_preACK}
f(x,z) &= \tilde f(x,z)+\beta\,\tilde\theta(x,z)\,F(p_{11},z),\\
\label{eq:g_decomp_preACK}
g(x,z) &= \tilde g(x,z)+\beta\,\tilde\theta(x,z)\,G(p_{11},z).
\end{align}
\end{lemma}

The proof is deferred to Appendix~\ref{app:marginal_metric_proofs}.

\begin{lemma}[Relation between \(\tilde f\) and \(\tilde\theta\)]
\label{lem:tildef_tildetheta_relation}
Assume \(r>0\). Then for every \(x\in\mathcal X\) and \(z\in\R\),
\begin{equation}
\label{eq:tildef_in_terms_of_tildetheta_new}
\tilde f(x,z)=r\,\tilde\theta(x,z).
\end{equation}
\end{lemma}

The proof is deferred to Appendix~\ref{app:marginal_metric_proofs}.

\begin{proposition}[Marginal reward and work metrics: one-step and regime-wise forms]
\label{pro:fg_closed_form}
For every \(x\in\mathcal X\) and \(z\in\R\), we have the one-step identities
\begin{align}
\label{eq:f_one_step}
f(x,z)
&=r\kappa x+\beta\Big(\kappa x\,F(p_{11},z)+(1-\kappa x)\,F(\phi^1(x),z)-F(\phi^0(x),z)\Big),\\
\label{eq:g_one_step}
g(x,z)
&=1+\beta\Big(\kappa x\,G(p_{11},z)+(1-\kappa x)\,G(\phi^1(x),z)-G(\phi^0(x),z)\Big),
\end{align}
where \(F(\cdot,z)\) and \(G(\cdot,z)\) are the reward and work metrics under the \(z\)-threshold policy
from Proposition~\ref{pro:FG_closed_form_with_proof}.

Substituting the regime-wise forms of \(F(\cdot,z)\) and \(G(\cdot,z)\) yields the following case split in \(z\).

\begin{enumerate}[label=\textup{(\alph*)},leftmargin=*]
\item \textbf{Regime \(z<p_{01}\).}
We have
$
f(x,z)=r\kappa x,\qquad g(x,z)=1.
$

\item \textbf{Regime \(p_{01}\le z\le x^1\).}
Let \(\tau_{\uparrow}^0(\cdot,z)\) be the passive up-crossing time, and set
\[
\tau_0\triangleq \tau_{\uparrow}^0(\phi^0(x),z),\qquad
\tau_1\triangleq \tau_{\uparrow}^0(\phi^1(x),z),
\qquad
\widehat\phi_a \triangleq \phi^0_{\tau_a}(\phi^a(x)),\quad a\in\{0,1\}.
\]
Define
$
F^{+}(u)\triangleq r\kappa\left[\frac{x^0}{1-\beta}-\frac{x^0-u}{1-\beta\rho}\right].
$
Then
$
f(x,z)
=
r\kappa x+\beta\Big(\kappa x\,F^{+}(p_{11})
+(1-\kappa x)\,\beta^{\tau_1}F^{+}(\widehat\phi_1)
-\beta^{\tau_0}F^{+}(\widehat\phi_0)\Big),
$
and
$
g(x,z)
=
1+\frac{\beta}{1-\beta}\Big(\kappa x+(1-\kappa x)\beta^{\tau_1}-\beta^{\tau_0}\Big).
$

\item \textbf{Regime \(x^1<z<x^0\).}
Let
$
D(z)\triangleq 1-\beta\,\widetilde{\Theta}(p_{11},z).
$
Then
\begin{align}
\label{eq:f_intermediate_compact}
f(x,z)
&=\frac{r\,\tilde\theta(x,z)}{D(z)},\\
\label{eq:g_intermediate_compact}
g(x,z)
&=\tilde g(x,z)+\beta\,\frac{\tilde\theta(x,z)}{D(z)}\,\widetilde G(p_{11},z).
\end{align}

\item \textbf{Regime \(x^0\le z<p_{11}\).}
The formulas in part \textup{(c)} remain valid; moreover, \(\widetilde F(\cdot,z)\), \(\widetilde G(\cdot,z)\), and \(\widetilde\Theta(\cdot,z)\) reduce to finite sums, truncating as in Proposition~\ref{pro:FG_closed_form_with_proof}\textup{(d)}.
In particular,
$
\widetilde F(y,z)=\widetilde G(y,z)=\widetilde\Theta(y,z)=0
\qquad\text{for }y\le z.
$

\item \textbf{Regime \(z\ge p_{11}\).}
We have
$
f(x,z)=r\kappa x,\qquad g(x,z)=1.
$
\end{enumerate}
\end{proposition}

The proof is deferred to Appendix~\ref{app:marginal_metric_proofs}.

\begin{lemma}[Marginal-work positivity in the tractable regimes]
\label{lem:pcli1_strong_easy}
For every \(x\in\mathcal X\):
\begin{enumerate}[label=\textup{(\alph*)},leftmargin=*]
\item If \(z<p_{01}\) or \(z\ge p_{11}\), then \(g(x,z)=1\).

\item If \(p_{01}\le z\le x^1\), then \(g(x,z)\ge 1-\beta\).

\item If \(x^0\le z<p_{11}\), then \(g(x,z)\ge 1-\beta\).
\end{enumerate}
\end{lemma}

The proof is deferred to Appendix~\ref{app:marginal_metric_proofs}.

\subsection{Partial closed-form formulas for the MP index}
\label{s:ia-partial-index}

We now specialize the marginal formulas to the MP index
$
m(x)\triangleq m(x,x).
$
The next lemma gives general identities for \(f(x,x)\), \(g(x,x)\), and \(m(x)\) in terms of the pre-ACK metrics.

\begin{lemma}[General formulas for \(f(x,x)\), \(g(x,x)\), and \(m(x)\)]
\label{lem:diag_fg_m_general}
For every \(x\in\mathcal X\), define
\begin{equation}
\label{eq:Dx_def_diag_general}
D(x)\triangleq 1-\beta\,\widetilde\Theta(p_{11},x).
\end{equation}
Then \(D(x)>0\), and
\begin{align}
\label{eq:f_diag_general}
f(x,x)
&=
\frac{\tilde f(x,x)}{D(x)}
=
\frac{r\,\tilde\theta(x,x)}{D(x)},\\
\label{eq:g_diag_general}
g(x,x)
&=
\tilde g(x,x)+\frac{\beta}{r}\,\frac{\tilde f(x,x)}{D(x)}\,\widetilde G(p_{11},x)
=
\tilde g(x,x)+\beta\,\frac{\tilde\theta(x,x)}{D(x)}\,\widetilde G(p_{11},x).
\end{align}
Consequently,
\begin{align}
\label{eq:m_diag_general}
m(x)
&=
\frac{\tilde f(x,x)}
{D(x)\,\tilde g(x,x)+\frac{\beta}{r}\,\tilde f(x,x)\,\widetilde G(p_{11},x)} =
\frac{r\,\tilde\theta(x,x)}
{D(x)\,\tilde g(x,x)+\beta\,\tilde\theta(x,x)\,\widetilde G(p_{11},x)}.
\end{align}
\end{lemma}

The proof is deferred to Appendix~\ref{app:marginal_metric_proofs}.

The next proposition gives exact formulas for the MP index in the tractable regimes. In the low- and high-belief
regions it coincides with the myopic index \(r\kappa x\), while on \([x^0,p_{11}]\) it admits an exact finite-sum
representation.

\begin{proposition}[Exact formulas for the MP index in the tractable regimes]
\label{pro:partial_myopic_index}
\begin{enumerate}[label=(\alph*),leftmargin=*]
\item For \(x\in[0,x^1]\),
$m(x)=r\kappa x.$

\item For \(x\in[x^0,p_{11}]\), let
$
u_t\triangleq \phi_t^1(p_{11}),\enspace t\ge0,
$
and
$
\Gamma_0^{11}\triangleq 1,
\enspace
\Gamma_t^{11}\triangleq \prod_{j=0}^{t-1}(1-\kappa u_j),\enspace t\ge1.
$
Then
\begin{align}
\label{eq:Gtilde_p11_diag_finite_tau}
\widetilde G(p_{11},x)
&=
\sum_{t=0}^{\tau_{\downarrow}^1(p_{11},x)-1}\beta^t\,\Gamma_t^{11},\\
\label{eq:Thetatilde_p11_diag_finite_tau}
\widetilde\Theta(p_{11},x)
&=
\kappa\sum_{t=0}^{\tau_{\downarrow}^1(p_{11},x)-1}\beta^t\,\Gamma_t^{11}\,u_t,
\end{align}
and
\begin{align}
\label{eq:f_closed_x0_p11}
f(x,x)
&=
\frac{r\kappa x}{1-\beta\,\widetilde\Theta(p_{11},x)},\\
\label{eq:g_closed_x0_p11}
g(x,x)
&=
\frac{1-\beta\,\widetilde\Theta(p_{11},x)+\beta\kappa x\,\widetilde G(p_{11},x)}
{1-\beta\,\widetilde\Theta(p_{11},x)},\\
\label{eq:m_closed_x0_p11}
m(x)
&=
\frac{r\kappa x}
{1-\beta\,\widetilde\Theta(p_{11},x)+\beta\kappa x\,\widetilde G(p_{11},x)}.
\end{align}
Equivalently,
\begin{equation}
\label{eq:m_closed_x0_p11_sum}
m(x)
=
\frac{r\kappa x}
{1+\beta\kappa\displaystyle\sum_{t=0}^{\tau_{\downarrow}^1(p_{11},x)-1}\beta^t\,\Gamma_t^{11}\,\bigl(x-u_t\bigr)}.
\end{equation}

\item For \(x\in[p_{11},1]\),
$
m(x)=r\kappa x.
$
\end{enumerate}
\end{proposition}

The proof is deferred to Appendix~\ref{app:marginal_metric_proofs}.

\begin{proposition}[Continuity and monotonicity of the MP index in the tractable regimes]
\label{pro:m_continuous_monotone_tractable}
Let $u_t,$
$\Gamma_0^{11}\triangleq 1,$ and
$\Gamma_t^{11}\triangleq \prod_{j=0}^{t-1}(1-\kappa u_j)$ be as in Proposition \ref{pro:partial_myopic_index}.
Also let
$
N_0\triangleq \tau_{\downarrow}^1(p_{11},x^0),
$
which is finite since \(u_t\downarrow x^1<x^0\). For \(n=1,\dots,N_0\), define
$
J_n\triangleq [\max\{x^0,u_n\},\,u_{n-1}),
\enspace
A_n\triangleq \sum_{t=0}^{n-1}\beta^t\,\Gamma_t^{11},
\enspace
B_n\triangleq \sum_{t=0}^{n-1}\beta^t\,\Gamma_t^{11}u_t.
$
Then the following hold.
\begin{enumerate}[label=\textup{(\alph*)},leftmargin=*]
\item On \([0,x^1]\), we have \(m(x)=r\kappa x\). Hence \(m(\cdot)\) is continuous and increasing on \([0,x^1]\).

\item For each \(n=1,\dots,N_0\) and every \(x\in J_n\),
\begin{equation}
\label{eq:m_piecewise_mobius_tractable}
m(x)=\frac{r\kappa x}{1+\beta\kappa(A_nx-B_n)}.
\end{equation}
Hence \(m(\cdot)\) is continuous and increasing on each interval \(J_n\).

\item The function \(m(\cdot)\) is continuous at every breakpoint \(u_n\) with \(1\le n\le N_0-1\), and also at \(x=p_{11}\).

\item On \([p_{11},1]\), we have \(m(x)=r\kappa x\). Hence \(m(\cdot)\) is continuous and increasing on \([p_{11},1]\).
\end{enumerate}

Consequently, \(m(\cdot)\) is continuous and increasing on each tractable regime
$
[0,x^1],\enspace [x^0,p_{11}],\enspace [p_{11},1],
$
and in particular is continuous at \(x=p_{11}\).
\end{proposition}

The proof is deferred to Appendix~\ref{app:marginal_metric_proofs}.

\begin{remark}[Where the MP index can differ from myopic]
Proposition~\ref{pro:partial_myopic_index} shows that possible departures of the MP index from the myopic index
\(r\kappa x\) are confined to the interval \((x^1,p_{11})\). On the subinterval \([x^0,p_{11}]\), the exact formula
\eqref{eq:m_closed_x0_p11} applies. When the PCL conditions hold, the same observation applies to the Whittle index.
\end{remark}

An illustrative parameter instance, together with diagnostic plots of the metrics
\(\widetilde F,\widetilde G,\widetilde\Theta,F,G,f,g\) and the diagonal MP index \(m(x)\), is presented in Appendix~\ref{app:illustrative_instance}. Those plots provide a concrete visualization of the tractable-regime formulas developed above and of the staircase structure that arises in the intermediate regime.

\subsection{What remains to be proven}
\label{s:open_gaps}

The tractable-regime analysis of Sections~\ref{s:tractable_regimes}--\ref{s:ia-partial-index}
already verifies the required PCL properties outside the intermediate threshold region
$
x^1<z<x^0.
$
Moreover, because the belief-update maps \(\phi^0\) and \(\phi^1\) satisfy the analogue of
\citet{danceSi19}'s Assumption~A2 after an increasing change of variables
(Lemma~\ref{lem:logit_conjugacy_A2}), the symbolic structure of threshold itineraries in this
intermediate region can be imported from the maps-with-gaps theory of \citet{danceSi19}; see
Theorem~\ref{thm:our_Thm12_mwords}, Corollary~\ref{cor:our_Cor13_pairs}, and
Proposition~\ref{pro:christoffel_interval_partition} below.

Accordingly, the remaining issue is no longer to describe the no-ACK threshold itineraries themselves, but to convert that symbolic description into the metric statements needed for full verification of
 \textup{(PCLI1)--(PCLI3)} on \(x^1<z<x^0\).
The subsections below push this reduction substantially further: they derive exact Christoffel-interval
formulas for the diagonal marginal metrics \(f(x,x)\), \(g(x,x)\), and \(m(x)\), and reduce the proof of
\textup{(PCLI2)} to an intervalwise monotonicity problem.

Thus, the main remaining analytic task is to establish \textup{(PCLI1)} and \textup{(PCLI2)} on
\(x^1<z<x^0\).
By contrast, \textup{(PCLI3)} is not an independent obstacle: as shown in
Subsection~\ref{s:pcli3_intermediate}, once \textup{(PCLI1)} and \textup{(PCLI2)} are available on that region,
\textup{(PCLI3)} follows from Proposition~6 of \citep[\S10.2]{nmmor20} via Assumption~4.

The symbolic organization of threshold itineraries on \(x^1<z<x^0\), including the
Christoffel--Sturmian structure inherited from the maps-with-gaps theory of \citet{danceSi19}, is
recorded in Appendix~\ref{app:word_structure}. That appendix identifies the itinerary patterns that
would need to be translated into metric statements in order to complete the remaining analytic verification.

\subsection{Reduction of \textup{(PCLI3)} to \textup{(PCLI1)} and \textup{(PCLI2)}}
\label{s:pcli3_intermediate}

Condition \textup{(PCLI3)} is not an independent obstacle in the present model. Indeed,
Proposition~6 in \citep[\S10.2]{nmmor20} shows that, under \textup{(PCLI1)} and
\textup{(PCLI2)}, \textup{(PCLI3)} follows once one verifies either Assumption~3
(piecewise-constant \(G(x,\cdot)\)) or the more structural Assumption~4, which requires that,
for each initial state \(x\), all states visited under any threshold policy lie almost surely
in a countable set \(D(x)\). For our model, Assumption~4 is natural and easy to verify.

\begin{proposition}[Verification of Assumption~4 for the present model]
\label{pro:ass4_present_model}
Fix an initial belief \(x\in\mathcal X\), and define
\begin{equation}
\label{eq:Dx_ass4_present_model}
D(x)\triangleq
\{\phi^{w}(x):\,w\in\{0,1\}^*\}
\ \cup\
\{\phi^{w}(p_{11}):\,w\in\{0,1\}^*\},
\end{equation}
where \(\{0,1\}^*\) denotes the set of all finite binary words, including the empty word.
Then \(D(x)\) is countable, and for every threshold \(z\in\R\),
\begin{equation}
\label{eq:trajectory_in_Dx}
\Prob_x^z\!\bigl\{ \{X(t)\}_{t\ge0}\subseteq D(x)\bigr\}=1.
\end{equation}
Hence Assumption~4 in \citep[\S10.2]{nmmor20} holds for the present model.
\end{proposition}

\begin{proof}
The set \(\{0,1\}^*\) of finite binary words is countable, so \(D(x)\) is countable.
Now fix a threshold \(z\). Before the first ACK time, the belief state follows the no-ACK
skeleton under the strict-threshold policy, and hence every visited state is of the form
\(\phi^w(x)\) for some finite word \(w\). When an ACK occurs, the belief resets to \(p_{11}\)
at the next period. Thereafter, until the next ACK, the visited states are of the form
\(\phi^w(p_{11})\) for finite words \(w\). The same argument applies after every subsequent
ACK, because each post-ACK excursion again starts from \(p_{11}\).
Therefore every state visited along the sample path belongs to \(D(x)\), proving
\eqref{eq:trajectory_in_Dx}. This is precisely Assumption~4 in \citep[\S10.2]{nmmor20}.
\end{proof}

The preceding proposition has the following immediate consequence.

\begin{corollary}[Reduction of \textup{(PCLI3)} to \textup{(PCLI1)} and \textup{(PCLI2)}]
\label{cor:pcli3_reduction_ass4}
For the present model, whenever \textup{(PCLI1)} and \textup{(PCLI2)} hold on a given threshold region,
\textup{(PCLI3)} also holds on that region.
\end{corollary}

\begin{proof}
By Proposition~\ref{pro:ass4_present_model}, Assumption~4 in \citep[\S10.2]{nmmor20} holds
for the present model. Hence Proposition~6 in \citep[\S10.2]{nmmor20} applies and yields
\textup{(PCLI3)} under \textup{(PCLI1)} and \textup{(PCLI2)}.
\end{proof}

Corollary~\ref{cor:pcli3_reduction_ass4} shows that \textup{(PCLI3)} is not an independent
obstacle herein. In particular, on the tractable threshold regimes, where
\textup{(PCLI1)} and \textup{(PCLI2)} have already been established, \textup{(PCLI3)} follows
immediately. Likewise, on the intermediate regime \(x^1<z<x^0\), once \textup{(PCLI1)} and
\textup{(PCLI2)} are proved, \textup{(PCLI3)} follows automatically by the same argument.
Thus the only genuinely unresolved analytical task is the verification of \textup{(PCLI1)} and
\textup{(PCLI2)} on that regime.

\section{Computational experiments}
\label{s:experiments}

This section complements the analytical results with large-scale computational experiments. Our goals are threefold: first, to provide broad numerical evidence for the remaining PCL-indexability conditions in parameter regimes not covered analytically; second, to assess the behavior of the MP index across representative problem families; and third, to benchmark the resulting MP index policy against standard alternatives and the Lagrangian dual upper bound. We begin with extensive numerical tests of \textup{(PCLI1)} and \textup{(PCLI2)}, and then turn to policy-performance comparisons on heterogeneous multi-project instances. Additional parameter-dependence plots for the MP index, together with implementation and reproducibility details, are reported in Appendix~\ref{app:extra_experiments}.

\subsection{Numerical verification of PCLI conditions}
\label{s:num_verify_pcli}

We begin with large-scale numerical tests of the two key PCL conditions that remain analytically unresolved in the intermediate regime, namely \textup{(PCLI1)} and \textup{(PCLI2)}. In both experiments, we explore a broad four-dimensional parameter grid in \((q,\rho,\kappa,\beta)\), with \(\rho\) parametrized as
$
\rho=\alpha(1-q),
$
so that the condition \(0<\rho<1-q\) is enforced.

\subsubsection{Testing condition \textup{(PCLI1)}}
\label{s:tcpcli1}

We first test the stronger condition
$
g(x,z)\ge 1-\beta.
$
Since this inequality has already been established analytically in the tractable regimes, the numerical study is restricted to the remaining strip
$
x^1\le z\le x^0,
$
where \(x^1\) and \(x^0\) are the fixed points of \(\phi^1\) and \(\phi^0\), respectively.

For each parameter tuple \((q,\rho,\kappa,\beta)\), we computed the minimum sampled slack
\[
\Delta_{\min}(q,\rho,\kappa,\beta)\triangleq \min_{x,z}\bigl(g(x,z)-(1-\beta)\bigr),
\]
over a finite grid with \(x\in[0,1]\) and \(z\in[x^1,x^0]\). The outer parameter grid consisted of \(14\)-point uniform grids on
\[
q\in[0.05,0.95],\qquad
\alpha\in[0.10,0.90],\qquad
\kappa\in[0.05,0.95],
\]
together with the discount-factor grid
$
\beta\in\{0.1,0.2,0.3,0.4,0.5,0.6,0.7,0.8,0.9,0.95,0.99\}.
$
This yields
$
14\times 14\times 14\times 11=30{,}184
$
parameter tuples.

For each tuple, we first computed \(x^1\) and \(x^0\), and then sampled \(x\) and \(z\) using endpoint-enriched cosine grids of size
$
N_x=N_u=121,
$
with
$
z=x^1+u(x^0-x^1),\qquad u\in[0,1].
$
Thus \(g(x,z)\) was evaluated at
$
121\times 121=14{,}641
$
sampled \((x,z)\)-pairs per tuple, for a total of
$
30{,}184\times 121\times 121=441{,}923{,}944
$
function evaluations.

The computation was implemented in Julia and run on \(8\) threads; the full experiment completed in about \(1.38\) hours. No sampled violation of \(g(x,z)\ge 1-\beta\) was found. The smallest observed slack over all sampled points was
$
\Delta_{\min}=2.60209967284708\times 10^{-4},
$
attained at
$
(q,\rho,\kappa,\beta)=\bigl(0.05,\;0.153462,\;0.95,\;0.1\bigr),
\enspace
(x^\star,z^\star)=\bigl(0.00273905,\;0.0504062\bigr).
$
Hence, this experiment provides strong numerical evidence that the stronger inequality \(g(x,z)\ge 1-\beta\) holds throughout the tested parameter region, including the intermediate strip \(x^1\le z\le x^0\).

\subsubsection{Testing condition \textup{(PCLI2)}}
\label{s:tcpcli2}

We next test condition \textup{(PCLI2)}, namely that the MP index
$
m(x)\triangleq f(x,x)/g(x,x)
$
is nondecreasing and continuous as a function of \(x\). We use the same outer parameter grid as in Subsection~\ref{s:tcpcli1}, and therefore again cover
$
14\times 14\times 14\times 11=30{,}184
$
parameter tuples.

For each tuple \((q,\rho,\kappa,\beta)\), we first computed the fixed points \(x^1\) and \(x^0\). We then sampled \(m(x)\) on a uniform grid over the core interval \([x^1,x^0]\), using
$
N_{\mathrm{mid}}=2001
$
equally spaced points, so that both endpoints were included exactly. To probe monotonicity and possible continuity effects at \(x^1\) and \(x^0\), we enlarged the interval to
\[
[x_L,x_R],\qquad
x_L=\max\{0,\;x^1-\eta(x^0-x^1)\},\qquad
x_R=\min\{1,\;x^0+\eta(x^0-x^1)\},
\]
with padding fraction
$
\eta=0.05,
$
and added
$
N_{\mathrm{side}}=201
$
uniformly spaced points on each side, excluding duplicated endpoints. Thus the nominal number of sampled \(x\)-values per tuple was
$
N_{\mathrm{mid}}+2N_{\mathrm{side}}=2001+2\cdot 201=2403,
$
for a total of
$
30{,}184\times 2403=72{,}532{,}152
$
evaluations of \(m(x)\).

To assess monotonicity, we computed the forward differences
$
\Delta_i\triangleq m(x_{i+1})-m(x_i),
$
and recorded both the minimum value over the full padded interval \([x_L,x_R]\) and the minimum value restricted to the core interval \([x^1,x^0]\). To assess continuity near the endpoints, we also recorded the one-step absolute differences immediately to the left and right of \(x^1\) and \(x^0\), and used their maximum as an endpoint continuity proxy.

The computation was implemented in Julia and run on \(8\) threads; the full experiment completed in about \(5.6\) minutes. No sampled violation of monotonicity was found. Over the padded interval \([x_L,x_R]\), the smallest forward difference was
$
1.350445\times 10^{-10},
$
attained for
$
(q,\rho,\kappa,\beta)=(0.95,\;0.005,\;0.05,\;0.99),
$
at a point near \(x\approx 0.954774\). Restricting attention to the core interval \([x^1,x^0]\), the smallest forward difference was
$
3.11747\times 10^{-10},
$
attained for
$
(q,\rho,\kappa,\beta)=(0.95,\;0.005,\;0.05,\;0.1),
$
near \(x\approx 0.954763\). Thus all sampled forward differences were positive, both on the core interval and on the enlarged interval used to probe the endpoints.

The largest observed endpoint continuity proxy was
$
6.846899\times 10^{-4}, 
$ which is
attained for
$
(q,\rho,\kappa,\beta)=(0.188462,\;0.730385,\;0.95,\;0.99).
$
For this case, the four one-step absolute differences around the endpoints were
\[
1.18554\times 10^{-4},\qquad
6.84690\times 10^{-4},\qquad
3.94556\times 10^{-4},\qquad
2.55875\times 10^{-5}.
\]
These values are small, and no visible jump-type behavior was detected at either endpoint in any sampled case.

Therefore, while this experiment does not constitute a proof, it provides strong numerical evidence in support of condition \textup{(PCLI2)}: throughout the tested parameter region, the sampled MP index \(m(x)\) was nondecreasing, including at and around the boundary points \(x^1\) and \(x^0\), and no visible discontinuities were observed.

\subsection{Policy benchmarking and dual bound comparisons}
\label{s:benchmark}

We next compare the MP index policy with three natural alternatives: the myopic policy, a round-robin policy, and a random policy. Here the myopic policy activates the \(M\) projects with the largest current one-step expected rewards \(r_n\kappa_n x_n\); the round-robin policy ignores beliefs and cycles deterministically through the project labels, activating the next \(M\) projects in a fixed cyclic order at each period; and the random policy ignores beliefs and activates \(M\) projects chosen uniformly at random at each period. The purpose of these experiments is twofold. First, we assess the practical value of the MP index policy in heterogeneous populations, where it need not coincide with the myopic rule. Second, we compare achieved rewards with the normalized Lagrangian dual upper bound in order to assess how close the best simulated policy comes to that benchmark and to look for evidence of asymptotic optimality as the system size grows.

\begin{remark}[Homogeneous-project ranking equivalence]
\label{rem:homogeneous_ranking_equiv}
Consider the homogeneous-project setting, where all projects share the same model parameters
\((p_{01},\rho,\kappa,r,\beta)\), and let \(I:\mathcal X\to\mathbb R\) be a common index function used to rank projects by their current belief states.
If \(I\) is nondecreasing, then selecting up to \(M\) projects with the largest index values is equivalent, up to tie-breaking, to selecting up to \(M\) projects with the largest beliefs.

In particular, the myopic index is
$
I_{\mathrm{my}}(x) \triangleq r\kappa x,
$
which is increasing in \(x\). Hence, if the MP index \(m(x)\) is nondecreasing on \(\mathcal X\) and nonnegative on the relevant state range, then the MP index policy and the myopic policy coincide in the homogeneous-project setting, up to tie-breaking among projects with equal beliefs. Consequently, meaningful benchmarking between the MP index and myopic rules requires heterogeneity across projects.
\end{remark}

We therefore focus on heterogeneous instances with two project types, denoted \(A\) and \(B\), intended to play the same conceptual role as the ``self-healing'' and ``fragile'' types used in the adherence-model experiments. Type \(A\) projects are assigned more favorable latent dynamics, while type \(B\) projects are assigned more fragile or persistent-bad dynamics. To obtain a reasonably rich but still interpretable grid, we considered \(4\) variants of type \(A\) crossed with \(3\) variants of type \(B\), yielding \(12\) scenario families in total. The four \(A\)-type parameter sets were
\[
( p_{01},\rho,\kappa,r )\in
\{(0.01,0.90,0.70,1),\ (0.03,0.80,0.70,1),\ (0.05,0.70,0.70,1),\ (0.02,0.85,0.55,1)\},
\]
and the three \(B\)-type parameter sets were
\[
( p_{01},\rho,\kappa,r )\in
\{(0.10,0.10,0.95,1),\ (0.15,0.05,0.95,1),\ (0.08,0.20,0.95,1)\}.
\]
Thus each scenario consists of one \(A\)-type and one \(B\)-type parameter vector.

Within each scenario, we varied the population composition over the nine type-\(A\) proportions
$
p_A\in\{0.1,\ldots,0.9\}$, $p_B=1-p_A,
$
capacity ratios
$
\alpha_M\in\{0.05,0.10,0.15,0.20,0.25,0.30,0.40,0.50\},
$
and  system sizes
$
N\in\{100,200,400,800,1600\}.
$
All projects were initialized at the common belief \(x_0=0.5\), so that performance differences are driven by the latent dynamics and the evolving information state rather than by exogenous heterogeneity in the initial beliefs. We chose all values of \(N\) as multiples of \(20\), so that \(\alpha_MN\) is an integer for every tested capacity ratio and the realized capacity \(M\) matches the target ratio exactly. Altogether, the grid comprises
$
12\times 9\times 8\times 5 = 4320
$
problem instances.

For each instance, we estimated the normalized discounted reward per project,
\[
J^\pi \approx \frac{1-\beta}{N}\,\Ex\!\left[\sum_{t=0}^{T-1}\beta^t R^\pi(t)\right],
\]
under each benchmark policy \(\pi\), using \(T=300\) periods, \(1000\) independent Monte Carlo replications, discount factor \(\beta=0.99\), and \(95\%\) confidence intervals. The MP index policy used precomputed type-specific index lookup tables on a fine belief grid, with linear interpolation between grid points. The myopic policy ranked projects by the one-period expected reward proxy \(r_n\kappa_n x_n\). The round-robin and random rules served as low-information baselines.

For each realized instance \((N,M)\), we also computed the normalized Lagrangian dual upper bound. A direct implementation of the bound routine proved to be the dominant computational bottleneck. We therefore replaced it by a grouped convex solver that exploits the fact that only two project types are present. Writing the dual objective as
\[
L_{\mathrm{vec}}(\lambda)
=
\frac{M\lambda}{1-\beta}
+
\sum_{n=1}^N L_n(x_n,\lambda),
\]
we evaluate the type-specific threshold \(z_k^\star(\lambda)\) only once per type \(k\), rather than once per project, and minimize the resulting continuous convex objective by bisection on a subgradient,
\[
s(\lambda)
=
\frac{M}{1-\beta}
-
\sum_{n=1}^N G_n^\star(x_n,\lambda).
\]
On representative test instances this yielded the same minimizer and bound value as the previous implementation, while reducing runtime by more than an order of magnitude; the full \(4320\)-instance experiment then completed in a few minutes rather than hours. This reduction in cost was essential for making the larger benchmark grid computationally feasible.
Additional implementation details, numerical tolerances, and reproducibility notes are collected in Appendix~\ref{app:extra_experiments}.

Across the full grid of \(4320\) instances, the MP index policy achieved the largest estimated mean reward in \(4185\) cases, or \(96.9\%\) of the total. The myopic policy was best in the remaining \(135\) cases (\(3.1\%\)), while round-robin and random were never best. Hence the MP index policy is the strongest benchmark overall, although, unlike in the homogeneous-project setting, it is not uniformly dominant.

The same conclusion holds under a more conservative confidence-interval comparison. Using the criterion that the lower endpoint of the MP index \(95\%\) confidence interval exceed the upper endpoint of the competing policy's interval, the MP index policy dominates round-robin and random in \(100\%\) of cases and dominates myopic in \(91.8\%\) of cases. Thus, even when it is not the best policy in point estimate, the MP index rule remains highly competitive, and clear inferiority to myopic is rare.

We next compare performance with the normalized Lagrangian dual upper bound through the relative gap
$
(\bar L - J^\pi)/\bar L.
$
Among the four policies tested, this gap was smallest for the MP index policy throughout the grid. Over all \(4320\) instances, the MP index relative gap ranged from \(5.66\%\) to \(39.02\%\), with mean \(14.46\%\). For the myopic policy, the corresponding range was \(6.58\%\) to \(62.93\%\), with mean \(24.20\%\). Round-robin and random performed substantially worse, with relative gaps typically above \(40\%\) and often much larger. Thus the MP index rule is not only the strongest simulated policy overall; it is also the one that remains closest to the dual benchmark across the tested policy class.

The smallest MP index relative gap occurred in the scenario
$
\texttt{A\_harder\_obs\_B\_very\_fragile},
$
with type proportions \((0.5,0.5)\), capacity ratio \(\alpha_M=0.5\), and \((N,M)=(200,100)\), where the relative gap was \(5.66\%\). The largest MP index relative gap occurred in
$
\texttt{A\_harder\_obs\_B\_persistent\_bad},
$
with proportions \((0.1,0.9)\), \(\alpha_M=0.05\), and \((N,M)=(100,5)\), where the relative gap was \(39.02\%\). These extremes already suggest two clear patterns: low-capacity regimes are harder, and scenarios involving the persistent-bad \(B\)-type tend to produce larger gaps.

To measure the practical value of the MP index rule relative to the strongest simpler benchmark, we also computed the relative improvement over myopic,
$
(J^{\mathrm{MP}}-J^{\mathrm{myopic}})/(J^{\mathrm{myopic}}).
$
Over the full grid, this quantity ranged from \(-0.91\%\) to \(69.18\%\), with mean approximately \(15.5\%\). Thus the MP index policy often provides a substantial gain over myopic, and even in the rare cases where it underperforms, the loss is very small.

The largest relative gain of the MP index policy over myopic occurred in the scenario
$
\texttt{A\_mild\_selfhealing\_B\_fragile},
$
with proportions \((0.3,0.7)\), \(\alpha_M=0.05\), and \((N,M)=(1600,80)\), where the gain reached \(69.18\%\). The smallest relative improvement occurred in
$
\texttt{A\_mild\_selfhealing\_B\_persistent\_bad},
$
with proportions \((0.7,0.3)\), \(\alpha_M=0.05\), and \((N,M)=(100,5)\), where the MP index policy fell below myopic by only \(0.91\%\). Hence the grid does contain a small number of cases where myopic slightly outperforms MP, but the overall comparison strongly favors the MP index rule.

Among all design factors, the capacity ratio \(\alpha_M=M/N\) has the clearest effect on the MP index gap to the dual bound. Averaging over all scenarios, proportions, and system sizes, the mean MP index relative gap decreases monotonically from about \(30.4\%\) at \(\alpha_M=0.05\) to about \(7.0\%\) at \(\alpha_M=0.50\). Thus the MP index policy comes much closer to the dual benchmark when capacity is less scarce.

The dependence on type composition is present but weaker. Averaged over scenarios, capacities, and system sizes, the mean MP index relative gap is smallest for intermediate mixtures, around \((0.4,0.6)\) and \((0.5,0.5)\), and larger at the extremes, especially for \(A\)-heavy populations. This suggests that strongly unbalanced mixtures are somewhat more difficult, especially when the more persistent type dominates.

By contrast, there is no convincing evidence here that the MP index relative gap vanishes as \(N\) grows. Averaging over scenarios, proportions, and capacities, the mean relative gap is essentially flat:
\[
14.53\%,\ 14.46\%,\ 14.44\%,\ 14.43\%,\ 14.42\%
\]
for \(N=100,200,400,800,1600\), respectively. Thus, over the tested range of system sizes, the experiments do not support asymptotic optimality of the MP index policy relative to the Lagrangian bound. What they do support is the more modest but practically relevant conclusion that the MP index policy remains closer to the dual benchmark than the simpler alternatives.

Overall, the numerical evidence points to three main conclusions. First, the MP index policy is the strongest benchmark across a large and systematically constructed grid of heterogeneous two-type instances. Second, its practical advantage over myopic can be substantial, often dramatic, especially in low-capacity and strongly heterogeneous regimes. Third, although the Lagrangian dual bound remains a useful common upper benchmark, it is not sufficiently tight here to certify asymptotic optimality of the MP index policy: the relative MP index gap remains non-negligible and essentially flat in \(N\) over the tested range.

Accordingly, the experiments support the MP index policy as a highly effective practical benchmark in heterogeneous populations, but they do not justify any claim of asymptotic optimality with respect to the Lagrangian relaxation. In particular, the low-capacity regime with persistent-bad \(B\)-type projects emerges as the most challenging part of the design space, whereas higher-capacity regimes with more fragile \(B\)-types appear substantially easier.

\section{Conclusions}
\label{s:conclusions}

We have studied a class of belief-state restless bandits with imperfect binary observations, motivated in particular by opportunistic spectrum access with sensing errors and collision constraints. The main contribution of the paper is a PCL-based analytical and computational framework for evaluating threshold-policy performance metrics and the associated marginal productivity (MP) index in this one-sided ACK/NACK setting.

At the technical level, the key step is the decomposition of threshold-policy dynamics into deterministic no-ACK skeleton evolution and ACK-triggered regenerative restarts. This leads to explicit renewal-type representations for the reward and work metrics \(F\) and \(G\), their marginal counterparts \(f\) and \(g\), and hence the MP index. In the tractable threshold regimes we obtained closed-form expressions for these quantities, including exact formulas for the diagonal MP index on \([0,x^1]\), \([x^0,p_{11}]\), and \([p_{11},1]\). In the intermediate regime \(x^1<z<x^0\), where the no-ACK skeleton alternates indefinitely between active and passive phases, the same renewal structure yields stable numerical evaluation schemes and isolates the remaining analytical gaps.

A second contribution is the clarification of the symbolic organization of threshold itineraries in the intermediate regime. Using the fact that the belief-update maps satisfy the analogue of the maps-with-gaps regularity conditions of \citet{danceSi19} after an increasing change of variables, we showed that the threshold axis is partitioned into Christoffel intervals and Sturmian points, and that these determine the corresponding one-step-deviation itineraries. This identifies a concrete route toward a complete analytical treatment of the unresolved part of the PCL verification.

The computational experiments provide broad numerical evidence for the PCL-indexability conditions beyond the parameter restrictions currently available in the literature. In particular, the large-scale tests of \textup{(PCLI1)} and \textup{(PCLI2)} revealed no counterexamples on extensive parameter grids, including the intermediate threshold strip \(x^1\le z\le x^0\). The experiments also showed that the MP index behaves continuously and monotonically in the tractable regimes and exhibits highly structured staircase behavior in the intermediate regime, consistent with the Christoffel--Sturmian organization of threshold itineraries.

From the policy-performance viewpoint, the MP-index policy emerged as the strongest benchmark across a large heterogeneous two-type test bed. It consistently outperformed round-robin and random baselines, and in the vast majority of instances also outperformed the myopic policy. Moreover, it generally remained closer than the benchmark policies to the normalized Lagrangian dual upper bound. At the same time, the experiments did \emph{not} provide convincing evidence that the MP-index policy becomes asymptotically optimal relative to that dual bound as the number of projects grows. Thus the numerical results support the MP-index policy as a highly effective practical benchmark, but not as an empirically established asymptotically optimal rule.

Taken together, the tractable-regime analysis, the reduction of \textup{(PCLI3)} to \textup{(PCLI1)} and \textup{(PCLI2)}, the Christoffel--Sturmian organization of the intermediate regime, and the extensive numerical evidence lead us to the following: 
\begin{conjecture}[Global PCL-indexability]
\label{con:gpcli}
For every feasible parameter tuple satisfying
\[
0<p_{01}<1,\qquad 0<\rho<1-p_{01},\qquad 0<\kappa<1,\qquad 0<\beta<1,
\]
the discounted single-project problem is PCL-indexable, that is, \textup{(PCLI1)--(PCLI3)} hold on the full threshold range.
\end{conjecture}
If true, this would imply that the model is Whittle-indexable throughout that parameter domain, and that the MP index coincides globally with the Whittle index.

Several directions remain open. The most immediate is to complete the analytical verification of \textup{(PCLI1)} and \textup{(PCLI2)} on the intermediate regime \(x^1<z<x^0\), thereby resolving the conjecture above and obtaining a full proof of Whittle indexability for the discounted model. A second direction is to develop the average-reward counterpart of the present discounted PCL analysis, building on the outline given in Appendix~\ref{app:avgcrit_outline}. More broadly, the computational framework developed here should also be useful for other partially observed restless bandit models with real-valued belief states.

\section{Declaration of generative AI and AI-assisted technologies in the manuscript preparation process}
\label{s:dgai}
During the preparation of this work, the author used ChatGPT (OpenAI) in order to assist with editing text, improving readability, and refining the presentation of the manuscript. After using this tool, the author reviewed and edited the content as needed and takes full responsibility for the content of the manuscript.

\bibliographystyle{plainnat}


\appendix

\section{Supplementary proofs for the computational machinery}
\label{app:compFG_proofs}

\subsection{Proof details for Section~\ref{s:phi_props_main}}
\label{app:phi_props_proofs}

We collect here the proofs omitted from Section~\ref{s:phi_props_main}.

\begin{proof}[Proof of Lemma~\ref{lem:policy_invariant_mean_path}]
Let \(\mathcal F_t\) denote the history up to time \(t\). Since \(A(t)\) is \(\mathcal F_t\)-measurable and
\eqref{eq:posterior_mean_identity} implies that \(\Ex[X(t+1)\mid X(t),A(t)]=\phi^0(X(t))\) a.s., we have
\[
\Ex\!\left[X(t+1)\mid \mathcal F_t\right]
=
\Ex\!\left[\Ex[X(t+1)\mid X(t),A(t)]\mid \mathcal F_t\right]
=
\Ex\!\left[\phi^0(X(t))\mid \mathcal F_t\right]
=
\phi^0(X(t)).
\]
Taking expectations and using the affine form of \(\phi^0\) yields the recursion
\[
\Ex_x^\pi[X(t+1)] = p_{01} + \rho \,\Ex_x^\pi[X(t)].
\]
Iterating it, with \(\Ex_x^\pi[X(0)]=x\), gives \(\Ex_x^\pi[X(t)]=\phi_t^0(x)\) for all \(t\ge 0\).
\end{proof}

\begin{proof}[Proof of Lemma~\ref{lem:phi1_fixed_point_location}]
Let \(B\triangleq 1-\rho+\kappa p_{11}\). The discriminant of \eqref{eq:phi1xxqe} is \(\Delta(\kappa)=B^2-4\kappa p_{01}\),
which can be rewritten as in \eqref{eq:phi1_discriminant} using \(1-\rho=p_{10}+p_{01}\) and \(p_{11}=1-p_{10}\). Hence \(\Delta(\kappa)>0\),
and \eqref{eq:phi1xxqe} has the roots in \eqref{eq:phi1_fixed_points}.

To locate the roots, write \(H(x)\triangleq \phi^1(x)-x\) for \(x\in[0,1/\kappa)\). Then
\[
H(x)=\frac{P(x)}{1-\kappa x},
\qquad
P(x)\triangleq \kappa x^2-Bx+p_{01}.
\]
Since \(1-\kappa x>0\) on \([0,1]\), the sign of \(H(x)\) on \(\mathcal X\) is the sign of the quadratic numerator \(P(x)\). We have
\[
P(p_{01})=p_{01}\rho(1-\kappa)>0,
\qquad
P(p_{11})=p_{01}-(1-\rho)p_{11}=-p_{10}\rho<0,
\]
so by continuity \(P\) (and hence \(H\)) has a root in \((p_{01},p_{11})\). This is the smaller root \(x_{\mathrm{lo}}^1\), and we
set \(x^1\triangleq x_{\mathrm{lo}}^1\in(p_{01},p_{11})\).

Next,
$
P(1)=\kappa-B+p_{01}=-p_{10}(1-\kappa)<0,
$
while \(P(x)\to+\infty\) as \(x\to+\infty\) since \(\kappa>0\), so the larger root satisfies \(x_{\mathrm{hi}}^1>1\).
Finally, Vieta's identity
$
x_{\mathrm{lo}}^1x_{\mathrm{hi}}^1={p_{01}}/{\kappa}
$
together with \(x_{\mathrm{lo}}^1>p_{01}\) gives \(x_{\mathrm{hi}}^1<1/\kappa\).
Thus \(x_{\mathrm{hi}}^1\in(1,1/\kappa)\).

Since \(\phi^1(\mathcal X)\subseteq[p_{01},p_{11}]\subseteq\mathcal X\), any fixed point of \(\phi^1\) in \(\mathcal X\) must lie
in \([p_{01},p_{11}]\), hence it must be \(x^1\).
\end{proof}

\begin{proof}[Proof of Lemma~\ref{lem:phi_order_refined}]
Set \(b\triangleq 1-\rho>0\) and \(y\triangleq p_{11}\in(0,1)\). From \(\phi^1(x)=x\) we obtain
\[
\kappa x^2-\bigl(b+\kappa y\bigr)x+p_{01}=0,
\]
hence \(p_{01}=x\bigl(b+\kappa y-\kappa x\bigr)\). Dividing by \(b\) and using \(x^0=p_{01}/b\) yields
\[
x^0-x=\frac{p_{01}}{b}-x=\frac{\kappa x(y-x)}{b}.
\]
Taking \(x=x^1\) and using \(x^1\in(p_{01},y)\) from Lemma~\ref{lem:phi1_fixed_point_location}, we get \(x^1>0\) and \(y-x^1>0\), hence
\(x^0-x^1>0\).
\end{proof}

\begin{proof}[Proof of Lemma~\ref{lem:phi1_iterates}]
For part \textup{(a)}, from \eqref{eq:phi1_mobius} we have
\[
(\phi^1)'(x)=\frac{\rho(1-\kappa)}{(1-\kappa x)^2}>0,\qquad x \in \mathcal{X},
\]
so \(\phi^1\) is increasing; compositions preserve strict monotonicity.

For part \textup{(b)}, writing \(\phi^1(x)=(Ax+p_{01})/(1-\kappa x)\) with \(A=\rho(1-\kappa)-\kappa p_{01}\) gives
\[
\phi^1(x)-x=\frac{\kappa x^2-\bigl(1-\rho+\kappa p_{11}\bigr)x+p_{01}}{1-\kappa x}.
\]
The numerator has roots \(x^1\) and \(x_{\mathrm{hi}}^1\), hence equals \(\kappa(x-x^1)(x-x_{\mathrm{hi}}^1)\), yielding
\eqref{eq:phi1_minus_x_factor}. Subtracting \(x^1=\phi^1(x^1)\) from \(\phi^1(x)\) gives \eqref{eq:phi1_minus_phiinfty_factor}.
The stated inequalities and monotonicity in \(t\) follow.

For part \textup{(c)}, apply \eqref{eq:mobius_conjugacy}--\eqref{eq:Phitx} to \(\Phi=\phi^1\) with \(x_-=x^1\), \(x_+=x_{\mathrm{hi}}^1\),
and \(\Psi\) as in \eqref{eq:Psix}, obtaining the first equality in \eqref{eq:phi1_iter_closed}. The second equality
follows by algebraic simplification. Since \(\mu\in(0,1)\), convergence to \(x^1\) is geometric.
\end{proof}

\begin{proof}[Proof of Lemma~\ref{lem:logit_conjugacy_A2}]
Fix \(a\in\{0,1\}\) and write \(t=\vartheta(x)\) with \(x\in(0,1)\). Since \(\vartheta'(x)=1/(x(1-x))\), we have
\begin{equation}
\label{eq:conj_derivative_formula_app}
(\hat{\phi}^a)'(t)
=
\frac{\vartheta'(\phi^a(x))}{\vartheta'(x)}\,(\phi^a)'(x)
=
\frac{(\phi^a)'(x)\,x(1-x)}{\phi^a(x)\bigl(1-\phi^a(x)\bigr)}.
\end{equation}
In particular, \((\hat \phi^a)'(t)>0\) for all \(t\), so \(\hat \phi^a\) is increasing. Moreover, if
\((\hat \phi^a)'(t)<1\) for all \(t\), then by the mean value theorem \(\hat \phi^a\) is contractive on \(\mathbb R\) in the
sense of Assumption~\ref{ass:DSA2_generic}\textup{(iii)}.

For \(a=0\), we have \(\phi^0(x)=p_{01}+\rho x\) and \((\phi^0)'(x)=\rho\), so by \eqref{eq:conj_derivative_formula_app} it suffices to show
\[
\phi^0(x)\bigl(1-\phi^0(x)\bigr)-\rho x(1-x) >0,\qquad x\in(0,1).
\]
A direct expansion yields
\[
\phi^0(x)\bigl(1-\phi^0(x)\bigr)-\rho x(1-x)
=
p_{01}(1-p_{01})-2p_{01}\rho\,x+\rho(1-\rho)\,x^2.
\]
The discriminant of this quadratic equals
$
4p_{01}\rho\,(p_{11}-1)=-4\rho p_{01}p_{10}<0,
$
so the quadratic is positive on \(\mathbb R\), and therefore \((\hat \phi^0)'(t)\in(0,1)\) for all \(t\).

For \(a=1\), we have
\[
\phi^1(x)=p_{01}+\rho\,\frac{(1-\kappa)x}{1-\kappa x},
\qquad
(\phi^1)'(x)=\frac{\rho(1-\kappa)}{(1-\kappa x)^2}.
\]
By \eqref{eq:conj_derivative_formula_app} it suffices to show
\[
\phi^1(x)\bigl(1-\phi^1(x)\bigr)-(\phi^1)'(x)\,x(1-x)>0,\qquad x\in(0,1).
\]
A straightforward calculation shows that the left-hand side equals \(Q(x)/(1-\kappa x)^2\), where \(Q\) is a quadratic satisfying
$
Q(0)=p_{01}(1-p_{01})>0
$
and whose discriminant is
$
4p_{01}\rho(\kappa-1)^2(p_{11}-1)
=-4\rho(1-\kappa)^2p_{01}p_{10}<0.
$
Hence \(Q(x)>0\) for all \(x\in\mathbb R\), and therefore \((\hat \phi^1)'(t)\in(0,1)\) for all \(t\).

Thus \(\hat\phi^0\) and \(\hat\phi^1\) are increasing and contractive on \(\mathbb R\), proving \textup{(a)}.
Since \(\vartheta\) is a bijection, fixed points are preserved by conjugacy: \(x\) is a fixed point of \(\phi^a\) if and only if
\(\vartheta(x)\) is a fixed point of \(\hat\phi^a\). Uniqueness of the fixed points then follows from contractiveness, proving
\textup{(b)}. Finally, because \(\vartheta\) is increasing, \(x^1<x^0\) implies \(\hat x^1<\hat x^0\), proving \textup{(c)}.
\end{proof}

\subsection{Proof details for Section~\ref{s:renewal_FG}}
\label{app:renewal_proofs}

We collect here the proofs omitted from Section~\ref{s:renewal_FG}.

\begin{proof}[Proof of Lemma~\ref{lem:renewal_first_ACK}]
For part \textup{(a)}, on the event \(\{\tau^{\mathrm{ack}}\ge t\}\), no ACK has occurred in periods \(0,\ldots,t-1\), so the trajectory up to time \(t\) coincides with the no-ACK skeleton. Hence
\[
\Ex_x^z\!\big[\1_{\{\tau^{\mathrm{ack}}\ge t\}}\,r\kappa X(t)A(t)\big]
=r\kappa\,\Gamma_t(x,z)\,\widetilde X_t(x,z)\,\widetilde A_t(x,z),
\]
and
\[
\Ex_x^z\!\big[\1_{\{\tau^{\mathrm{ack}}\ge t\}}\,A(t)\big]
=\Gamma_t(x,z)\,\widetilde A_t(x,z).
\]
Since
\[
\sum_{t=0}^{\tau^{\mathrm{ack}}}\beta^t\,r\kappa X(t)A(t)
=
\sum_{t=0}^{\infty}\beta^t\,\1_{\{\tau^{\mathrm{ack}}\ge t\}}\,r\kappa X(t)A(t),
\qquad
\sum_{t=0}^{\tau^{\mathrm{ack}}}\beta^t\,A(t)
=
\sum_{t=0}^{\infty}\beta^t\,\1_{\{\tau^{\mathrm{ack}}\ge t\}}\,A(t),
\]
taking expectations yields the first two identities.

For \(\widetilde{\Theta}(x,z)\), note that \(\{\tau^{\mathrm{ack}}=t\}\) means that no ACK occurs before period \(t\) and an ACK occurs in period \(t\). Given \(\{\tau^{\mathrm{ack}}\ge t\}\), an ACK in period \(t\) occurs only if \(A(t)=1\), and then with conditional probability \(\kappa X(t)=\kappa\widetilde X_t(x,z)\). Thus
\[
\Prob_x^z\{\tau^{\mathrm{ack}}=t\}
=\Gamma_t(x,z)\,\kappa\,\widetilde X_t(x,z)\,\widetilde A_t(x,z),
\]
and therefore
\[
\widetilde{\Theta}(x,z)
=\sum_{t=0}^{\infty}\beta^t\,\Prob_x^z\{\tau^{\mathrm{ack}}=t\}
=\sum_{t=0}^{\infty}\beta^t\,\Gamma_t(x,z)\,\kappa\,\widetilde X_t(x,z)\,\widetilde A_t(x,z)
=\frac{1}{r}\,\widetilde F(x,z).
\]

For part \textup{(b)}, for any nonnegative process \((Y(t))_{t\ge0}\),
\[
\sum_{t=0}^{\infty}\beta^tY(t)
=
\sum_{t=0}^{\tau^{\mathrm{ack}}}\beta^tY(t)
+\1_{\{\tau^{\mathrm{ack}}<\infty\}}\sum_{t=\tau^{\mathrm{ack}}+1}^{\infty}\beta^tY(t).
\]
Taking \(Y(t)=r\kappa X(t)A(t)\) and conditioning on \(\mathcal F_{\tau^{\mathrm{ack}}+1}\), on \(\{\tau^{\mathrm{ack}}<\infty\}\) an ACK is observed at the end of period \(\tau^{\mathrm{ack}}\), so \(X(\tau^{\mathrm{ack}}+1)=p_{11}\). By the strong Markov property at time \(\tau^{\mathrm{ack}}+1\),
\[
\Ex_x^z\!\left[
\1_{\{\tau^{\mathrm{ack}}<\infty\}}\sum_{t=\tau^{\mathrm{ack}}+1}^{\infty}\beta^t\,r\kappa X(t)A(t)\,\Big|\,\mathcal F_{\tau^{\mathrm{ack}}+1}
\right]
=
\1_{\{\tau^{\mathrm{ack}}<\infty\}}\beta^{\tau^{\mathrm{ack}}+1}F(p_{11},z).
\]
Taking expectations gives
\[
F(x,z)
=
\Ex_x^z\!\big[\sum_{t=0}^{\tau^{\mathrm{ack}}}\beta^t\,r\kappa X(t)A(t)\big]
+
\Ex_x^z\!\big[\beta^{\tau^{\mathrm{ack}}+1}\1_{\{\tau^{\mathrm{ack}}<\infty\}}\big]F(p_{11},z).
\]
Using \textup{(a)} and \(\Ex_x^z\big[\beta^{\tau^{\mathrm{ack}}+1}\big]=\beta\,\widetilde{\Theta}(x,z)\) yields the identity for \(F\). The argument for \(G\) is identical with \(Y(t)=A(t)\).

For part \textup{(c)}, apply part \textup{(b)} with \(x=p_{11}\) and rearrange to obtain \eqref{eq:FGrho_closed}. Finally, since
\(F(p_{11},z)\) and \(G(p_{11},z)\) are finite and \(\widetilde F,\widetilde G\ge0\), the identities in
\eqref{eq:FGrho_closed} imply \(1-\beta\,\widetilde{\Theta}(p_{11},z)>0\), that is,
\(0\le \beta\,\widetilde{\Theta}(p_{11},z)<1\).
\end{proof}

\subsection{Proof details for Section~\ref{s:tractable_regimes}}
\label{app:tractable_regimes_proofs}

We collect here the proofs omitted from Section~\ref{s:tractable_regimes}.

\begin{proof}[Proof of Lemma~\ref{lem:threshold_regime_dynamics}]
Both updates satisfy \(\phi^0(\mathcal X)\subseteq[p_{01},p_{11}]\) and \(\phi^1(\mathcal X)\subseteq[p_{01},p_{11}]\), so after at
most one step the skeleton state lies in \([p_{01},p_{11}]\) forever.
Moreover, the passive iterates \(\phi_t^0(\cdot)\) converge monotonically to \(x^0\), while the active no-ACK iterates
\(\phi_t^1(\cdot)\) converge monotonically to \(x^1\) (Lemma~\ref{lem:phi1_iterates}\textup{(b)--(c)}).
Finally, if \(z\ge x^0\), then \(\phi^0(z)\le z\), and monotonicity of \(\phi^0\) implies \(\phi^0(y)\le z\) for all \(y\le z\).

\smallskip
\noindent\emph{(a)}
If \(z<p_{01}\), then \(\widetilde X_1(x,z)\in[p_{01},p_{11}]\subset(z,\infty)\) regardless of \(x\).
Hence the skeleton is active from time \(1\) onward, and from time \(0\) onward if \(x>z\).

\smallskip
\noindent\emph{(b)}
If \(p_{01}\le z\le x^1\), then under passivity the skeleton drifts toward \(x^0>x^1\ge z\), so it up-crosses \(z\)
in finite time \(t_0=\tau_{\uparrow}^0(x,z)\).
Once \(\widetilde X_{t_0}(x,z)>z\), the active no-ACK evolution stays above \(z\), since it converges monotonically to \(x^1\ge z\), and in the boundary case \(z=x^1\) it remains strictly above \(x^1\) after the first up-crossing.

\smallskip
\noindent\emph{(c)}
Assume \(x^1<z<x^0\).
Since \(z<x^0\), any passive phase at or below \(z\) must up-cross \(z\) in finite time.
On the other hand, for every \(y>z\), the active no-ACK iterates satisfy
\[
\phi_t^1(y)\downarrow x^1<z,
\]
so every active excursion above \(z\) must down-cross \(z\) in finite time.
After each return to \([0,z]\), the skeleton must up-cross again in finite time under passivity.
Hence finite blocks of \(1\)'s and \(0\)'s alternate indefinitely.

\smallskip
\noindent\emph{(d)}
If \(x^0\le z<p_{11}\), then once the skeleton enters \([0,z]\) it cannot exceed \(z\) again.
Thus if \(x\le z\) it is passive forever, while if \(x>z\) it performs one active excursion until its first down-crossing of \(z\)
and is passive forever thereafter.

\smallskip
\noindent\emph{(e)}
If \(z\ge p_{11}\), then for  \(t\ge1\) the skeleton state lies in \([p_{01},p_{11}]\subseteq(-\infty,z]\),
so \(\widetilde A_t(x,z)=0\), while \(\widetilde A_0(x,z)=\1_{\{x>z\}}\).
\end{proof}

\begin{proof}[Proof of Lemma~\ref{lem:tauack_finite_alternating}]
\emph{(a)}
If \(z<x^0\), then by Lemma~\ref{lem:threshold_regime_dynamics}\textup{(a)--(c)} the skeleton itinerary contains infinitely many
\(1\)'s.

Writing
$
N_t^{\mathrm{act}}(x,z)\triangleq \sum_{j=0}^{t-1}\widetilde A_j(x,z),
$
we therefore have \(N_t^{\mathrm{act}}(x,z)\to\infty\).
Moreover, on every active period \(j\), \(\widetilde X_j(x,z)\in[p_{01},p_{11}]\), so
\[
1-\kappa\widetilde X_j(x,z)\le 1-\kappa p_{01}<1.
\]
Using \eqref{eq:GammatxzProd},
\[
\Gamma_t(x,z)
=\prod_{0\le j<t:\ \widetilde A_j(x,z)=1}\bigl(1-\kappa\,\widetilde X_j(x,z)\bigr)
\le (1-\kappa p_{01})^{N_t^{\mathrm{act}}(x,z)}\to0.
\]
Hence \(\Prob_x^z\{\tau^{\mathrm{ack}}=\infty\}=\Gamma_\infty(x,z)=0\).

\smallskip
\noindent\emph{(b)}
Assume \(x^0\le z<p_{11}\).
If \(x\le z\), then the policy is passive at \(t=0\), and since \(\phi^0(y)\le z\) for all \(y\le z\), it remains passive forever.
Thus \(\tau^{\mathrm{ack}}=\infty\) a.s.
If \(x>z\), then absent an ACK the skeleton is active until time \(\tau_{\downarrow}^1(x,z)\) and passive forever thereafter
(Lemma~\ref{lem:threshold_regime_dynamics}\textup{(d)}).
Hence ACKs are possible only in periods \(j=0,1,\ldots,\tau_{\downarrow}^1(x,z)-1\), during which
\(\widetilde X_j(x,z)=\phi_j^1(x)\) and \(\widetilde A_j(x,z)=1\).
Multiplying the corresponding no-ACK probabilities yields \eqref{eq:PTinfty_noreturn}.

\smallskip
\noindent\emph{(c)}
If \(z\ge p_{11}\), then after time \(0\) the belief lies in \([p_{01},p_{11}]\subseteq(-\infty,z]\) regardless of the outcome at time \(0\),
so the policy cannot activate after \(t=0\).
Therefore \(\Prob_x^z\{\tau^{\mathrm{ack}}<\infty\}=\kappa x\,\1_{\{x>z\}}\).
\end{proof}

\begin{proof}[Proof of Proposition~\ref{pro:tilde_metrics_piecewise_constant_tractable}]
The pre-ACK metrics \(\widetilde F(x,z)\), \(\widetilde G(x,z)\), and \(\widetilde\Theta(x,z)\) are defined in
\eqref{eq:tilde_series} and \eqref{eq:tildeTheta_preACK_def} as deterministic functionals of the no-ACK skeleton orbit
\(\widetilde X_t(x,z)\), itinerary \(\widetilde A_t(x,z)\), and survival probabilities \(\Gamma_t(x,z)\).
It is therefore enough to show that, on each of the stated \(z\)-intervals, the no-ACK skeleton orbit and survival
probabilities are independent of the particular value of \(z\).

\emph{(a)}
By Lemma~\ref{lem:threshold_regime_dynamics}\textup{(a)}, if \(z<p_{01}\), then the itinerary is \(1^\infty\) when
\(x>z\), and \(01^\infty\) when \(x\le z\).
Thus, on each of the two stated intervals, the skeleton itinerary is fixed.
Moreover, the corresponding skeleton orbit is also fixed:
it is \((\phi_t^1(x))_{t\ge0}\) in the case \(1^\infty\), and, in the case \(01^\infty\),
\[
x,\ \phi^0(x),\ \phi^1(\phi^0(x)),\ \phi_2^1(\phi^0(x)),\ldots.
\]

Hence the survival probabilities \(\Gamma_t(x,z)\) are also fixed, and the three pre-ACK metrics are constant.

\emph{(b)}
Fix \(n\in\mathbb Z_+\) and \(z\in I_n^\uparrow(x)\).
By definition, \(\tau_\uparrow^0(x,z)=n\), and by Lemma~\ref{lem:threshold_regime_dynamics}\textup{(b)}, the skeleton itinerary is \(0^n1^\infty\).

Therefore the skeleton orbit depends only on \(x\) and \(n\), not on the particular \(z\in I_n^\uparrow(x)\), being
\[
x,\ \phi^0(x),\ \phi_2^0(x),\ \ldots,\ \phi_n^0(x),\
\phi^1(\phi_n^0(x)),\ \phi_2^1(\phi_n^0(x)),\ldots.
\]

The corresponding survival probabilities \(\Gamma_t(x,z)\) are therefore also independent of \(z\) on \(I_n^\uparrow(x)\).
Hence \(\widetilde F(x,z)\), \(\widetilde G(x,z)\), and \(\widetilde\Theta(x,z)\) are constant on \(I_n^\uparrow(x)\).

\emph{(c)}
If \(x\le z\), then Lemma~\ref{lem:threshold_regime_dynamics}\textup{(d)} gives the itinerary \(0^\infty\), so no activation
ever occurs and therefore
\[
\widetilde F(x,z)=\widetilde G(x,z)=\widetilde\Theta(x,z)=0.
\]

Assume now that \(x>z\), and fix \(n\in\mathbb N\) and \(z\in I_n^\downarrow(x)\).
Then \(\tau_\downarrow^1(x,z)=n\), and by Lemma~\ref{lem:threshold_regime_dynamics}\textup{(d)} the skeleton itinerary is
\(1^n0^\infty\).
Hence the skeleton orbit depends only on \(x\) and \(n\), not on the particular \(z\in I_n^\downarrow(x)\), being
\[
x,\ \phi^1(x),\ \phi_2^1(x),\ \ldots,\ \phi_n^1(x),\
\phi^0(\phi_n^1(x)),\ \phi_2^0(\phi_n^1(x)),\ldots.
\]

Thus the survival probabilities are fixed on \(I_n^\downarrow(x)\), and the three pre-ACK metrics are constant there.

\emph{(d)}
By Lemma~\ref{lem:threshold_regime_dynamics}\textup{(e)}, if \(z\ge p_{11}\), then the itinerary is \(10^\infty\) when \(x>z\),
and \(0^\infty\) when \(x\le z\).

Hence, on each of the two stated intervals, both the skeleton itinerary and the corresponding skeleton orbit are fixed,
so the survival probabilities are fixed as well.
Therefore \(\widetilde F(x,z)\), \(\widetilde G(x,z)\), and \(\widetilde\Theta(x,z)\) are constant on each interval.
\end{proof}

\begin{proof}[Proof of Proposition~\ref{pro:FG_closed_form_with_proof}]
The step-function property follows because, in each regime, \(F(x,z)\) and \(G(x,z)\) depend on \(z\) only through 
threshold comparisons along the no-ACK skeleton and through integer-valued crossing times 
(\(\tau_{\uparrow}^0(x,z)\) and \(\tau_{\downarrow}^1(x,z)\)).
Hence jumps can occur only when \(z\) coincides with a skeleton belief value reached from \(x\) or from \(p_{11}\).

\smallskip
\noindent\textbf{(a)}
If \(z<p_{01}\), then \([p_{01},p_{11}]\subset(z,\infty)\), so after at most one step the policy is always active.
By Lemma~\ref{lem:threshold_regime_dynamics}\textup{(a)}, the itinerary is \(1^\infty\) if \(x>z\) and \(01^\infty\) if \(x\le z\).
If \(x>z\), then \(A(t)\equiv1\), so
\[
F(x,z)=r\kappa\sum_{t=0}^{\infty}\beta^t\,\Ex_x^z[X(t)],
\qquad
G(x,z)=\sum_{t=0}^{\infty}\beta^t=\frac{1}{1-\beta}.
\]
Writing \(m_t\triangleq \Ex_x^z[X(t)]\), the one-step posterior-mean identity \eqref{eq:posterior_mean_identity} yields
$
m_{t+1}=p_{01}+\rho m_t,\enspace m_0=x,
$
hence
$
m_t=x^0+(x-x^0)\rho^t.
$
Substituting and summing the geometric series gives
\[
F(x,z)
=r\kappa\left[\frac{x^0}{1-\beta}-\frac{x^0-x}{1-\beta\rho}\right].
\]

If \(x\le z\), then \(A(0)=0\), \(X(1)=\phi^0(x)\), and thereafter the policy is always active.
Thus
\[
G(x,z)=\sum_{t=1}^{\infty}\beta^t=\frac{\beta}{1-\beta},
\qquad
F(x,z)=\beta\,F(\phi^0(x),z),
\]
and the preceding formula applied at initial state \(\phi^0(x)\) yields the stated expression.

\smallskip
\noindent\textbf{(b)}
Let \(t_0\triangleq\tau_{\uparrow}^0(x,z)\).
By Lemma~\ref{lem:threshold_regime_dynamics}\textup{(b)}, the itinerary is \(0^{t_0}1^\infty\).
Hence
\[
G(x,z)=\sum_{t=t_0}^{\infty}\beta^t=\frac{\beta^{t_0}}{1-\beta}.
\]

Also, up to time \(t_0\) the process is passive, so \(X(t_0)=\phi_{t_0}^0(x)\) deterministically, and from time \(t_0\) onward the policy is always active.
Applying part \textup{(a)} from time \(t_0\) onward yields
\[
F(x,z)
=r\kappa\,\beta^{t_0}
\left[\frac{x^0}{1-\beta}-\frac{x^0-\phi_{t_0}^0(x)}{1-\beta\rho}\right].
\]

\smallskip
\noindent\textbf{(c)}
This is exactly Lemma~\ref{lem:renewal_first_ACK}\textup{(b)--(c)}.

\smallskip
\noindent\textbf{(d)}
If \(x\le z\), then the policy never activates by Lemma~\ref{lem:threshold_regime_dynamics}\textup{(d)}, so
\(F(x,z)=G(x,z)=0\).

Assume next that \(x>z\).
By Lemma~\ref{lem:threshold_regime_dynamics}\textup{(d)}, the no-ACK skeleton starting from \(x\) has the form
\[
\widetilde{\sigma}(x,z)=1^{n}0^\infty,
\qquad n=\tau_{\downarrow}^1(x,z),
\]
and the no-ACK skeleton starting from the post-ACK reset state \(p_{11}\) has the form
\[
\widetilde{\sigma}(p_{11},z)=1^{m}0^\infty,
\qquad m=\tau_{\downarrow}^1(p_{11},z).
\]
Hence, by Proposition~\ref{pro:tilde_metrics_piecewise_constant_tractable}\textup{(c)}, on every interval
\[
I_{n,m}(x)=
\Bigl\{z\in[x^0,\min\{x,p_{11}\}) : \tau_{\downarrow}^1(x,z)=n,\ \tau_{\downarrow}^1(p_{11},z)=m\Bigr\}
\]
the pre-ACK metrics are constant, and are given by the stated finite sums
\[
\widetilde G(x,z)=A_x^{(n)},
\qquad
\widetilde\Theta(x,z)=B_x^{(n)},
\qquad
\widetilde F(x,z)=r\,B_x^{(n)},
\]
together with
\[
\widetilde G(p_{11},z)=A_{11}^{(m)},
\qquad
\widetilde\Theta(p_{11},z)=B_{11}^{(m)},
\qquad
\widetilde F(p_{11},z)=r\,B_{11}^{(m)}.
\]

Applying Lemma~\ref{lem:renewal_first_ACK}\textup{(c)} at \(p_{11}\) yields
\[
F(p_{11},z)=\frac{r\,B_{11}^{(m)}}{1-\beta\,B_{11}^{(m)}},
\qquad
G(p_{11},z)=\frac{A_{11}^{(m)}}{1-\beta\,B_{11}^{(m)}}.
\]
Then Lemma~\ref{lem:renewal_first_ACK}\textup{(b)} gives
\[
F(x,z)=\widetilde F(x,z)+\beta\,\widetilde\Theta(x,z)\,F(p_{11},z)
=r\,B_x^{(n)}+\beta\,B_x^{(n)}\frac{r\,B_{11}^{(m)}}{1-\beta\,B_{11}^{(m)}}
=\frac{r\,B_x^{(n)}}{1-\beta\,B_{11}^{(m)}},
\]
\[
G(x,z)=\widetilde G(x,z)+\beta\,\widetilde\Theta(x,z)\,G(p_{11},z)
=A_x^{(n)}+\beta\,B_x^{(n)}\frac{A_{11}^{(m)}}{1-\beta\,B_{11}^{(m)}},
\]
which is the stated formula.

Since these expressions depend on \(z\) only through the pair of integers
\(\bigl(\tau_{\downarrow}^1(x,z),\tau_{\downarrow}^1(p_{11},z)\bigr)\), it follows that \(F(x,z)\) and \(G(x,z)\) are
constant on each nonempty interval \(I_{n,m}(x)\), and therefore piecewise constant on \([x^0,p_{11})\).

\smallskip
\noindent\textbf{(e)}
If \(z\ge p_{11}\), then regardless of the action and outcome at time \(0\), the belief at time \(1\) lies in
\([p_{01},p_{11}]\subseteq(-\infty,z]\).
Hence the policy cannot activate after time \(0\), so
\[
G(x,z)=\Ex_x^z[A(0)]=\1_{\{x>z\}},
\qquad
F(x,z)=\Ex_x^z[r\kappa X(0)A(0)]=r\kappa x\,\1_{\{x>z\}}.
\]
\end{proof}

\begin{proof}[Proof of Lemma~\ref{lem:G_monotone_tractable}]
We consider the tractable regimes separately.

\smallskip
\noindent\emph{Case 1: \(z<p_{01}\).}
By Proposition~\ref{pro:FG_closed_form_with_proof}\textup{(a)},
\[
G(x,z)=
\begin{cases}
\dfrac{\beta}{1-\beta}, & x\le z,\\[1mm]
\dfrac{1}{1-\beta}, & x>z,
\end{cases}
\]
which is clearly nondecreasing in \(x\).

\smallskip
\noindent\emph{Case 2: \(p_{01}\le z\le x^1\).}
By Proposition~\ref{pro:FG_closed_form_with_proof}\textup{(b)},
\[
G(x,z)=\frac{\beta^{\tau_\uparrow^0(x,z)}}{1-\beta}.
\]
For fixed \(z\), the map \(x\mapsto \tau_\uparrow^0(x,z)\) is nonincreasing, because \(\phi_t^0(\cdot)\) is increasing for every \(t\ge0\).
Since \(0<\beta<1\), the map \(n\mapsto \beta^n\) is decreasing on \(\mathbb Z_+\), hence
\(x\mapsto \beta^{\tau_\uparrow^0(x,z)}\) is nondecreasing. Therefore \(x\mapsto G(x,z)\) is nondecreasing.

\smallskip
\noindent\emph{Case 3: \(x^0\le z<p_{11}\).}
Let \(\mathcal{T}_z\) be the Bellman evaluation operator for the work metric:
\[
(\mathcal{T}_z H)(x)\triangleq
\begin{cases}
\beta\,H(\phi^0(x)), & x\le z,\\[1mm]
1+\beta\Big(\kappa x\,H(p_{11})+(1-\kappa x)\,H(\phi^1(x))\Big), & x>z.
\end{cases}
\]
Set \(H_0\equiv 0\) and \(H_{n+1}\triangleq \mathcal{T}_z H_n\).
Then \(H_n(x)\uparrow G(x,z)\) pointwise as \(n\to\infty\).

We prove by induction that each \(H_n\) is nondecreasing.
This is clear for \(H_0\).
Assume \(H_n\) is nondecreasing.
If \(x_1\le x_2\le z\), then \(\phi^0(x_1)\le \phi^0(x_2)\), and hence
\[
H_{n+1}(x_1)=\beta H_n(\phi^0(x_1))\le \beta H_n(\phi^0(x_2))=H_{n+1}(x_2).
\]
If \(z<x_1\le x_2\), then
\begin{align*}
H_{n+1}(x_2)-H_{n+1}(x_1)
&=
\beta\kappa(x_2-x_1)H_n(p_{11})
+\beta\Big((1-\kappa x_2)H_n(\phi^1(x_2))-(1-\kappa x_1)H_n(\phi^1(x_1))\Big)\\
&=
\beta\kappa(x_2-x_1)\bigl(H_n(p_{11})-H_n(\phi^1(x_2))\bigr)
+\beta(1-\kappa x_1)\bigl(H_n(\phi^1(x_2))-H_n(\phi^1(x_1))\bigr)\ge 0,
\end{align*}
because \(\phi^1\) is increasing, \(H_n\) is nondecreasing, and \(\phi^1(x_2)\le p_{11}\).

If \(x_1\le z<x_2\), then \(H_{n+1}(x_1)=\beta H_n(\phi^0(x_1))\), while
\[
H_{n+1}(x_2)=1+\beta\Big(\kappa x_2 H_n(p_{11})+(1-\kappa x_2)H_n(\phi^1(x_2))\Big)\ge 1.
\]
Since \(z\ge x^0\), the interval \([0,z]\) is trapping under \(\phi^0\), so \(\phi^0(y)\le z\) for every \(y\le z\). Starting from
\(H_0\equiv0\), it follows by induction that \(H_n(y)=0\) for every \(y\le z\) and every \(n\). In particular,
\(H_{n+1}(x_1)=0\le H_{n+1}(x_2)\).

Thus \(H_{n+1}\) is nondecreasing. Passing to the limit yields that \(G(\cdot,z)\) is nondecreasing.

\smallskip
\noindent\emph{Case 4: \(z\ge p_{11}\).}
By Proposition~\ref{pro:FG_closed_form_with_proof}\textup{(e)},
$
G(x,z)=\1_{\{x>z\}},
$
which is nondecreasing in \(x\).

Combining the four cases proves the result.
\end{proof}

\subsection{Proof details for Sections~\ref{s:ia-fg}--\ref{s:ia-partial-index}}
\label{app:marginal_metric_proofs}

We collect here the proofs omitted from Sections~\ref{s:ia-fg}--\ref{s:ia-partial-index}.

\begin{proof}[Proof of Lemma~\ref{lem:fg_decomp_preACK}]
By one-step conditioning at time \(0\),
\begin{align*}
f(x,z)
&=r\kappa x+\beta\Big(\kappa x\,F(p_{11},z)+(1-\kappa x)\,F(\phi^1(x),z)-F(\phi^0(x),z)\Big),\\
g(x,z)
&=1+\beta\Big(\kappa x\,G(p_{11},z)+(1-\kappa x)\,G(\phi^1(x),z)-G(\phi^0(x),z)\Big).
\end{align*}

Applying Lemma~\ref{lem:renewal_first_ACK}\textup{(b)} to \(F(y,z)\) and \(G(y,z)\) for
\(y\in\{\phi^1(x),\phi^0(x)\}\) gives
\[
F(y,z)=\widetilde F(y,z)+\beta\,\widetilde\Theta(y,z)\,F(p_{11},z),
\qquad
G(y,z)=\widetilde G(y,z)+\beta\,\widetilde\Theta(y,z)\,G(p_{11},z).
\]

Substituting these expressions and collecting the constant terms and the coefficients of
\(F(p_{11},z)\) and \(G(p_{11},z)\) yields \eqref{eq:f_decomp_preACK}--\eqref{eq:g_decomp_preACK}, together with
\eqref{eq:tildef_one_step_new}--\eqref{eq:tildetheta_one_step_new}.
\end{proof}

\begin{proof}[Proof of Lemma~\ref{lem:tildef_tildetheta_relation}]
By Lemma~\ref{lem:renewal_first_ACK}\textup{(a)}, for every \(y\in\mathcal X\) we have
$
\widetilde\Theta(y,z)={\widetilde F(y,z)}/{r}.
$
Substituting this into \eqref{eq:tildetheta_one_step_new} and comparing with \eqref{eq:tildef_one_step_new} gives
\[
\tilde\theta(x,z)
=\frac{1}{r}\Bigl[r\kappa x+\beta\bigl((1-\kappa x)\widetilde F(\phi^1(x),z)-\widetilde F(\phi^0(x),z)\bigr)\Bigr]
=\frac{\tilde f(x,z)}{r}.
\]
\end{proof}

\begin{proof}[Proof of Proposition~\ref{pro:fg_closed_form}]
The one-step identities \eqref{eq:f_one_step}--\eqref{eq:g_one_step} follow by conditioning on the time-\(0\) outcome:
under action \(1\), an ACK occurs with probability \(\kappa x\), sending the belief to \(p_{11}\), while with probability
\(1-\kappa x\) no ACK is observed and the belief updates to \(\phi^1(x)\); under action \(0\), the belief updates
deterministically to \(\phi^0(x)\). In either case, the policy follows the \(z\)-threshold policy from time \(1\) onward.

\smallskip
\noindent\emph{Regimes (a), (b), and (e).}
These follow by substituting into \eqref{eq:f_one_step}--\eqref{eq:g_one_step} the corresponding formulas for
\(F(\cdot,z)\) and \(G(\cdot,z)\) from Proposition~\ref{pro:FG_closed_form_with_proof}.

\smallskip
\noindent\emph{Regime (c).}
For \(x^1<z<x^0\), let
$
D(z)\triangleq 1-\beta\,\widetilde{\Theta}(p_{11},z).
$
Lemma~\ref{lem:renewal_first_ACK}\textup{(b)--(c)} gives
\[
F(p_{11},z)=\frac{\widetilde F(p_{11},z)}{D(z)},
\qquad
G(p_{11},z)=\frac{\widetilde G(p_{11},z)}{D(z)}.
\]

By Lemma~\ref{lem:fg_decomp_preACK},
\[
f(x,z)=\tilde f(x,z)+\beta\,\tilde\theta(x,z)\,F(p_{11},z),
\qquad
g(x,z)=\tilde g(x,z)+\beta\,\tilde\theta(x,z)\,G(p_{11},z).
\]

Using Lemma~\ref{lem:tildef_tildetheta_relation},
\[
f(x,z)
=r\,\tilde\theta(x,z)+\beta\,\tilde\theta(x,z)\,\frac{\widetilde F(p_{11},z)}{D(z)}
=\tilde\theta(x,z)\left(r+\beta\,\frac{\widetilde F(p_{11},z)}{D(z)}\right)
=\frac{r\,\tilde\theta(x,z)}{D(z)},
\]
which gives \eqref{eq:f_intermediate_compact}. Similarly, we obtain \eqref{eq:g_intermediate_compact}:
\[
g(x,z)
=\tilde g(x,z)+\beta\,\tilde\theta(x,z)\,\frac{\widetilde G(p_{11},z)}{D(z)}.
\]

\smallskip
\noindent\emph{Regime (d).}
The same substitutions as in regime \textup{(c)} apply. The fact that \(\widetilde F(\cdot,z)\), \(\widetilde G(\cdot,z)\), and \(\widetilde\Theta(\cdot,z)\) reduce to finite sums follows from Proposition~\ref{pro:FG_closed_form_with_proof}\textup{(d)} and Lemma~\ref{lem:renewal_first_ACK}\textup{(a)}.
\end{proof}

\begin{proof}[Proof of Lemma~\ref{lem:pcli1_strong_easy}]
\emph{(a)}
If \(z<p_{01}\) or \(z\ge p_{11}\), then Proposition~\ref{pro:fg_closed_form}\textup{(a)} and
\textup{(e)} give \(g(x,z)=1\).

\emph{(b)}
For \(p_{01}\le z\le x^1\), Proposition~\ref{pro:fg_closed_form}\textup{(b)} yields
\[
g(x,z)
=
1+\frac{\beta}{1-\beta}\Big(\kappa x+(1-\kappa x)\beta^{\tau_1}-\beta^{\tau_0}\Big),
\]
with
$
\tau_0=\tau_{\uparrow}^0(\phi^0(x),z),
\enspace
\tau_1=\tau_{\uparrow}^0(\phi^1(x),z).
$

We first show that
\begin{equation}
\label{eq:tau1_le_tau0_plus1_app}
\tau_1\le \tau_0+1.
\end{equation}

If \(x>x^1\), then Lemma~\ref{lem:phi1_iterates}\textup{(b)} gives
$
\phi^1(x)>x^1\ge z,
$
so \(\tau_1=0\), and \eqref{eq:tau1_le_tau0_plus1_app} is immediate.
Assume now that \(x\le x^1\). Then Lemma~\ref{lem:phi1_iterates}\textup{(b)} implies
$
\phi^1(x)\ge x,
$
hence
$
x^0-\phi^1(x)\le x^0-x.
$
Also,
\[
x^0-\phi^0(x)=x^0-(p_{01}+\rho x)=\rho(x^0-x),
\]
so
\[
x^0-\phi^1(x)\le \frac{x^0-\phi^0(x)}{\rho}.
\]
Let \(t=\tau_0\). By definition of \(\tau_\uparrow^0\),
$
\phi_t^0(\phi^0(x))>z.
$
Using the closed form \(\phi_t^0(y)=x^0+(y-x^0)\rho^t\), this is equivalent to
$
\rho^t\bigl(x^0-\phi^0(x)\bigr)<x^0-z.
$
Therefore,
\[
\rho^{t+1}\bigl(x^0-\phi^1(x)\bigr)
\le
\rho^t\bigl(x^0-\phi^0(x)\bigr)
<
x^0-z,
\]
which implies
$
\phi_{t+1}^0(\phi^1(x))>z.
$
Hence \(\tau_1\le t+1=\tau_0+1\), proving \eqref{eq:tau1_le_tau0_plus1_app}.

Since \(0<\beta<1\), \eqref{eq:tau1_le_tau0_plus1_app} implies
$
\beta^{\tau_1}\ge \beta^{\tau_0+1}.
$
Because \(1-\kappa x>0\), substituting this into the formula for \(g(x,z)\) gives
\begin{align*}
g(x,z)
&\ge
1+\frac{\beta}{1-\beta}\Big(\kappa x+(1-\kappa x)\beta^{\tau_0+1}-\beta^{\tau_0}\Big)=
1+\frac{\beta}{1-\beta}\Big(\kappa x(1-\beta^{\tau_0+1})-(1-\beta)\beta^{\tau_0}\Big)\\
&=
1-\beta^{\tau_0+1}
+\frac{\beta\kappa x}{1-\beta}\bigl(1-\beta^{\tau_0+1}\bigr)=
(1-\beta^{\tau_0+1})\left(1+\frac{\beta\kappa x}{1-\beta}\right)
\ge 1-\beta^{\tau_0+1}\ge 1-\beta.
\end{align*}

\emph{(c)}
Fix \(z\in[x^0,p_{11})\).
If \(x\le z\), then \(\phi^1(x)\le \phi^0(x)\le z\), where the second inequality follows from \(z\ge x^0\) and the monotonicity of \(\phi^0\).
Hence
$
G(\phi^1(x),z)=G(\phi^0(x),z)=0,
$
and the one-step identity \eqref{eq:g_one_step} yields
$
g(x,z)=1+\beta\,\kappa x\,G(p_{11},z)\ge 1.
$
Hence assume \(x>z\).
Define
\[
H_z(y)\triangleq \kappa y\,G(p_{11},z)+(1-\kappa y)\,G(\phi^1(y),z),\qquad y\in(z,1].
\]
By the evaluation equation \eqref{eq:gseveq}, for every \(y>z\),
\[
G(y,z)=1+\beta\,H_z(y).
\]

We claim that \(H_z(\cdot)\) is nondecreasing on \((z,1]\).
Indeed, if \(y_1\le y_2\), then
\begin{align*}
H_z(y_2)-H_z(y_1)
&=
\kappa(y_2-y_1)\,G(p_{11},z)
+\Big((1-\kappa y_2)G(\phi^1(y_2),z)-(1-\kappa y_1)G(\phi^1(y_1),z)\Big)\\
&=
\kappa(y_2-y_1)\Bigl(G(p_{11},z)-G(\phi^1(y_2),z)\Bigr)
+(1-\kappa y_1)\Bigl(G(\phi^1(y_2),z)-G(\phi^1(y_1),z)\Bigr)\ge0,
\end{align*}
because \(G(\cdot,z)\) is nondecreasing on \(\mathcal X\) by Lemma~\ref{lem:G_monotone_tractable},
\(\phi^1\) is increasing, and \(\phi^1(y_2)\le p_{11}\).

Now, since \(z\ge x^0\) and \(x>z\), we have \(\phi^0(x)\le x\).
If \(\phi^0(x)\le z\), then \(G(\phi^0(x),z)=0\), and the one-step identity \eqref{eq:g_one_step} yields
$
g(x,z)=1+\beta\,H_z(x)\ge1.
$
If instead \(\phi^0(x)>z\), then by \eqref{eq:gseveq},
\[
G(\phi^0(x),z)=1+\beta\,H_z(\phi^0(x))\le 1+\beta\,H_z(x),
\]
because \(H_z\) is nondecreasing and \(\phi^0(x)\le x\). Substituting this bound into \eqref{eq:g_one_step}, we obtain
\begin{align*}
g(x,z)
&=1+\beta\Bigl(H_z(x)-G(\phi^0(x),z)\Bigr)\ge 1+\beta\Bigl(H_z(x)-1-\beta H_z(x)\Bigr)=1-\beta+\beta(1-\beta)H_z(x).
\end{align*}
Since \(G(\cdot,z)\ge0\), we have \(H_z(x)\ge0\), and therefore
$
g(x,z)\ge 1-\beta.
$
\end{proof}

\begin{proof}[Proof of Lemma~\ref{lem:diag_fg_m_general}]
For every \(x\in\mathcal X\), let
$
D(x)\triangleq 1-\beta\,\widetilde\Theta(p_{11},x).
$
Applying Lemma~\ref{lem:renewal_first_ACK}\textup{(c)} with \(z=x\), and using
$
\widetilde F(p_{11},x)=r\,\widetilde\Theta(p_{11},x)
$
from Lemma~\ref{lem:renewal_first_ACK}\textup{(a)}, we obtain
\[
F(p_{11},x)=\frac{r\,\widetilde\Theta(p_{11},x)}{D(x)},
\qquad
G(p_{11},x)=\frac{\widetilde G(p_{11},x)}{D(x)}.
\]
Moreover, Lemma~\ref{lem:renewal_first_ACK}\textup{(c)} gives
$
0\le \beta\,\widetilde\Theta(p_{11},x)<1,
$
hence \(D(x)>0\).

Now apply Lemma~\ref{lem:fg_decomp_preACK} with \(z=x\):
\[
f(x,x)=\tilde f(x,x)+\beta\,\tilde\theta(x,x)\,F(p_{11},x),
\qquad
g(x,x)=\tilde g(x,x)+\beta\,\tilde\theta(x,x)\,G(p_{11},x).
\]
Substituting the preceding formulas into the first identity gives
\[
f(x,x)
=
\tilde f(x,x)+\beta\,\tilde\theta(x,x)\,\frac{r\,\widetilde\Theta(p_{11},x)}{D(x)}.
\]
Using Lemma~\ref{lem:tildef_tildetheta_relation}, namely \(\tilde f(x,x)=r\,\tilde\theta(x,x)\), this becomes
\[
f(x,x)
=
r\,\tilde\theta(x,x)\left(1+\frac{\beta\,\widetilde\Theta(p_{11},x)}{D(x)}\right).
\]
Since \(D(x)=1-\beta\,\widetilde\Theta(p_{11},x)\), we have
\[
1+\frac{\beta\,\widetilde\Theta(p_{11},x)}{D(x)}
=
\frac{D(x)+\beta\,\widetilde\Theta(p_{11},x)}{D(x)}
=
\frac{1}{D(x)},
\]
and therefore
\[
f(x,x)
=
\frac{r\,\tilde\theta(x,x)}{D(x)}
=
\frac{\tilde f(x,x)}{D(x)},
\]
which proves \eqref{eq:f_diag_general}. Similarly,
\[
g(x,x)
=
\tilde g(x,x)+\beta\,\tilde\theta(x,x)\,\frac{\widetilde G(p_{11},x)}{D(x)}
=
\tilde g(x,x)+\frac{\beta}{r}\,\frac{\tilde f(x,x)}{D(x)}\,\widetilde G(p_{11},x),
\]
which proves \eqref{eq:g_diag_general}. Finally, substituting \eqref{eq:f_diag_general} and \eqref{eq:g_diag_general} into
$
m(x) \triangleq {f(x,x)}/{g(x,x)}
$
yields \eqref{eq:m_diag_general}.
\end{proof}

\begin{proof}[Proof of Proposition~\ref{pro:partial_myopic_index}]
\emph{(a)} Let \(z=x\).
If \(0\le x<p_{01}\), then \(z=x<p_{01}\), so Proposition~\ref{pro:fg_closed_form}\textup{(a)} gives
$
g(x,x)=1,\enspace f(x,x)=r\kappa x,
$
hence \(m(x)=r\kappa x\).
If \(p_{01}\le x<x^1\), then \(\phi^0(x)>x\) (since \(x<x^0\)) and \(\phi^1(x)>x\) (since \(x<x^1\)).
Thus in Proposition~\ref{pro:fg_closed_form}\textup{(b)} with \(z=x\) one has
$
\tau_0=\tau_{\uparrow}^0(\phi^0(x),x)=0,
\enspace
\tau_1=\tau_{\uparrow}^0(\phi^1(x),x)=0,
$
so
\[
g(x,x)=1+\frac{\beta}{1-\beta}\bigl(\kappa x+(1-\kappa x)-1\bigr)=1.
\]
Moreover, since
\[
F^{+}(u)\triangleq r\kappa\left[\frac{x^0}{1-\beta}-\frac{x^0-u}{1-\beta\rho}\right]
\]
is affine in \(u\), and
$
\phi^0(x)=\kappa x\,p_{11}+(1-\kappa x)\phi^1(x)
$
by the one-step posterior-mean identity, the bracketed term in
Proposition~\ref{pro:fg_closed_form}\textup{(b)} cancels, yielding
$
f(x,x)=r\kappa x.
$
Hence \(m(x)=r\kappa x\).

Finally, at \(x=x^1\), Proposition~\ref{pro:fg_closed_form}\textup{(b)} gives
$
\tau_0=\tau_{\uparrow}^0(\phi^0(x^1),x^1)=0,\enspace
\tau_1=\tau_{\uparrow}^0(\phi^1(x^1),x^1)=1,
$
and therefore
$
g(x^1,x^1)=1-\beta+\beta\kappa x^1.
$
Similarly,
\[
f(x,x)-r\kappa x\,g(x,x)
=
-\frac{\beta\kappa r}{1-\beta\rho}\,
\Bigl[\kappa x^2-(1-\rho+\kappa p_{11})x+p_{01}\Bigr].
\]
Evaluating at \(x=x^1\), the bracket vanishes because \(x^1\) is the fixed point of \(\phi^1\), i.e., the root of
\eqref{eq:phi1xxqe}. Hence \(f(x^1,x^1)=r\kappa x^1\,g(x^1,x^1)\), so \(m(x^1)=r\kappa x^1\).
This proves part \textup{(a)}.

\emph{(b)} Fix \(x\in[x^0,p_{11}]\) and set \(z=x\). Since \(x\ge x^0\), we have \(\phi^0(x)\le x\), and since
\(x\ge x^0>x^1\), we also have \(\phi^1(x)<x\). Hence, after either continuation state \(\phi^0(x)\) or \(\phi^1(x)\),
the \(x\)-threshold policy is passive forever. Therefore,
\[
\widetilde F(\phi^0(x),x)=\widetilde F(\phi^1(x),x)=0,
\qquad
\widetilde G(\phi^0(x),x)=\widetilde G(\phi^1(x),x)=0,
\qquad
\widetilde\Theta(\phi^0(x),x)=\widetilde\Theta(\phi^1(x),x)=0.
\]
By \eqref{eq:tildef_one_step_new}--\eqref{eq:tildetheta_one_step_new}, it follows that
$
\tilde f(x,x)=r\kappa x,\enspace
\tilde g(x,x)=1,\enspace
\tilde\theta(x,x)=\kappa x.
$
Hence, by Lemma~\ref{lem:fg_decomp_preACK},
\[
f(x,x)=r\kappa x+\beta\kappa x\,F(p_{11},x),
\qquad
g(x,x)=1+\beta\kappa x\,G(p_{11},x).
\]

Under the \(x\)-threshold policy started from \(p_{11}\), the no-ACK skeleton remains active up to, but not including, time
$
\tau_{\downarrow}^1(p_{11},x),
$
and is passive thereafter.
 Therefore
\eqref{eq:Gtilde_p11_diag_finite_tau}--\eqref{eq:Thetatilde_p11_diag_finite_tau} hold.
Applying Lemma~\ref{lem:renewal_first_ACK}\textup{(c)} gives
\[
F(p_{11},x)=\frac{r\,\widetilde\Theta(p_{11},x)}{1-\beta\,\widetilde\Theta(p_{11},x)},
\qquad
G(p_{11},x)=\frac{\widetilde G(p_{11},x)}{1-\beta\,\widetilde\Theta(p_{11},x)}.
\]
Substituting these into the preceding identities yields \eqref{eq:f_closed_x0_p11} and
\eqref{eq:g_closed_x0_p11}, and taking their ratio gives \eqref{eq:m_closed_x0_p11}.
Finally, substituting \eqref{eq:Gtilde_p11_diag_finite_tau}--\eqref{eq:Thetatilde_p11_diag_finite_tau} into
\eqref{eq:m_closed_x0_p11} gives \eqref{eq:m_closed_x0_p11_sum}.

\emph{(c)} If \(x\in[p_{11},1]\), then \(z=x\ge p_{11}\), so by Proposition~\ref{pro:fg_closed_form}\textup{(e)}
$
g(x,x)=1, f(x,x)=r\kappa x.
$
Hence \(m(x)=r\kappa x\).
\end{proof}

\begin{proof}[Proof of Proposition~\ref{pro:m_continuous_monotone_tractable}]
\emph{(a)} This is Proposition~\ref{pro:partial_myopic_index}\textup{(a)}.

\emph{(b)} Fix \(n\in\{1,\dots,N_0\}\) and \(x\in J_n\). By construction,
$
u_n\le x<u_{n-1}.
$
For \(n<N_0\) this is immediate from the definition of \(J_n\). For \(n=N_0\), it follows from
$
J_{N_0}=[x^0,u_{N_0-1})
\qquad\text{and}\enspace
u_{N_0}\le x^0<u_{N_0-1},
$
the latter being equivalent to \(N_0=\tau_{\downarrow}^1(p_{11},x^0)\). Hence
$
\tau_{\downarrow}^1(p_{11},x)=n.
$
Therefore, Proposition~\ref{pro:partial_myopic_index}\textup{(b)} gives
$
\widetilde G(p_{11},x)=A_n,
\enspace
\widetilde\Theta(p_{11},x)=\kappa B_n.
$
Substituting these identities into \eqref{eq:m_closed_x0_p11} yields \eqref{eq:m_piecewise_mobius_tractable}.

Now differentiate \eqref{eq:m_piecewise_mobius_tractable} on \(J_n\):
\[
m'(x)
=
\frac{r\kappa\bigl(1-\beta\kappa B_n\bigr)}
{\bigl(1+\beta\kappa(A_nx-B_n)\bigr)^2}.
\]
Since, on \(J_n\),
$
1-\beta\kappa B_n
=
1-\beta\,\widetilde\Theta(p_{11},x)>0
$
by Lemma~\ref{lem:renewal_first_ACK}\textup{(c)}, it follows that \(m'(x)>0\) on \(J_n\).
Thus \(m(\cdot)\) is continuous and increasing on each \(J_n\).

\emph{(c)} Let \(1\le n\le N_0-1\). Since
$
A_{n+1}=A_n+\beta^n\Gamma_n^{11},
\enspace
B_{n+1}=B_n+\beta^n\Gamma_n^{11}u_n,
$
we obtain
$
1+\beta\kappa(A_{n+1}u_n-B_{n+1})
=
1+\beta\kappa(A_nu_n-B_n).
$
Hence the left and right formulas in \eqref{eq:m_piecewise_mobius_tractable} match at \(x=u_n\), so \(m\) is continuous at each \(u_n\).

At \(x=p_{11}\), Proposition~\ref{pro:partial_myopic_index}\textup{(c)} gives
$
m(p_{11})=r\kappa p_{11}.
$
Also, since \(u_0=p_{11}\), the interval immediately below \(p_{11}\) is \(J_1=[\max\{x^0,u_1\},p_{11})\), and on \(J_1\),
$
A_1=1,\qquad B_1=p_{11}.
$
Thus
\[
m(x)=\frac{r\kappa x}{1+\beta\kappa(x-p_{11})},
\qquad x\in J_1,
\]
so
$
\lim_{x\nearrow p_{11}}m(x)=r\kappa p_{11}=m(p_{11}).
$
Therefore \(m\) is continuous at \(x=p_{11}\).

\emph{(d)} This is Proposition~\ref{pro:partial_myopic_index}\textup{(c)}.
\end{proof}

\section{Illustrative instance and diagnostic plots}
\label{app:illustrative_instance}

To illustrate the preceding analytical and computational results, we consider the parameter instance
\[
p_{01}=0.25,\qquad \rho=0.6,\qquad \kappa=0.8,\qquad \beta=0.95,\qquad r=1.
\]
For this instance,
$
p_{11}=p_{01}+\rho=0.85,\enspace
x^0={p_{01}}/{1-\rho}=0.625,
$
and the active no-ACK fixed point is
$
x^1\approx 0.2967.
$
Thus the analytically unresolved threshold region is the intermediate interval
$
x^1<z<x^0,
$ i.e., 
$
z\in(0.2967,\,0.625),
$
while the complementary regions are covered by the tractable-regime analysis above.

Figures~\ref{fig:tildeFG_vs_z}--\ref{fig:fg_vs_z_plot} display, for a representative fixed belief \(x\), the
threshold dependence of the pre-ACK metrics \(\widetilde F(x,z)\), \(\widetilde G(x,z)\),
\(\widetilde \Theta(x,z)\), the reward/work metrics \(F(x,z)\), \(G(x,z)\), and the marginal metrics
\(\tilde f(x,z)\), \(\tilde g(x,z)\), \(f(x,z)\), and \(g(x,z)\).
The plots are consistent with the tractable-regime theory.
In particular, outside the interval \(x^1<z<x^0\), the curves exhibit the right-continuous stepwise behavior predicted by
Propositions~\ref{pro:tilde_metrics_piecewise_constant_tractable},
\ref{pro:FG_closed_form_with_proof}, and \ref{pro:fg_closed_form}, with jumps only at threshold values where the
underlying skeleton itinerary changes.
Within the intermediate regime \(x^1<z<x^0\), the plots display a much finer staircase structure.
This is consistent with the Christoffel--Sturmian itinerary organization of
Theorem~\ref{thm:our_Thm12_mwords} and Corollary~\ref{cor:our_Cor13_pairs}, even though the corresponding metric
consequences are not derived here.

Because \(r=1\) in this instance, Lemma~\ref{lem:tildef_tildetheta_relation} implies
$
\tilde f(x,z)=\tilde\theta(x,z),
$
so Figure~\ref{fig:tildef_tildeg_vs_z}  provides a useful consistency check on the
pre-ACK marginal quantities. Likewise, Figures~\ref{fig:fg_vs_z_plot} and~\ref{fig:fg_diag} indicate that the
marginal work metric remains positive throughout the plotted range, in agreement with the proved lower bounds on the
tractable regimes and with the broader numerical evidence reported later.

\begin{figure}[!htbp]
  \centering
  \includegraphics[width=0.4\textwidth]{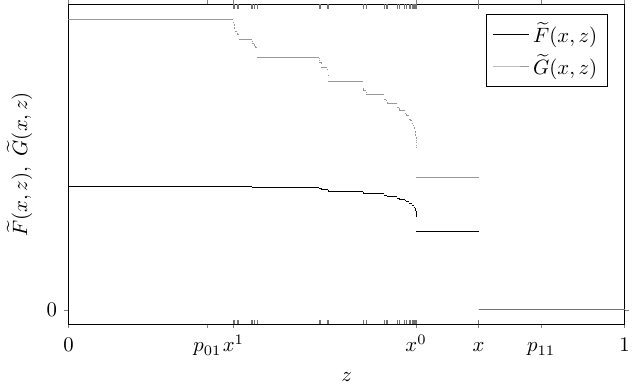}%
  \caption{Pre-ACK reward and work metrics $\widetilde{F}(x,z)$ and $\widetilde{G}(x,z)$ vs.\ $z$ for fixed belief $x$.}
  \label{fig:tildeFG_vs_z}
\end{figure}

\begin{figure}[!htbp]
  \centering
  \includegraphics[width=0.4\textwidth]{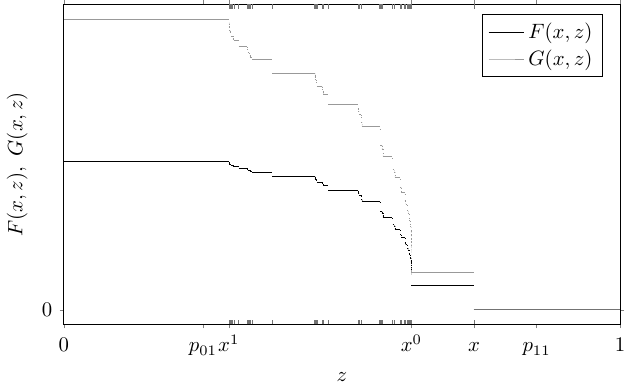}%
  \caption{Reward and work metrics $F(x,z)$ and $G(x,z)$ vs.\ $z$ for fixed belief $x$.}
  \label{fig:FG_vs_z}
\end{figure}

\begin{figure}[!htbp]
  \centering
  \includegraphics[width=0.4\textwidth]{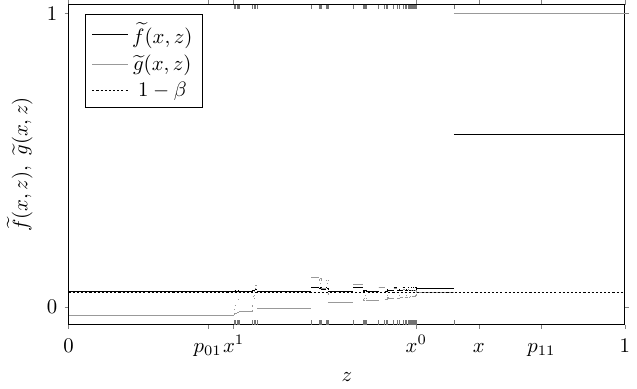}%
  \caption{Marginal metrics $\tilde{f}(x,z)$ and $\tilde g(x,z)$ vs.\ $z$ for fixed belief $x$.}
  \label{fig:tildef_tildeg_vs_z}
\end{figure}

\begin{figure}[!htbp]
  \centering
  \includegraphics[width=0.4\textwidth]{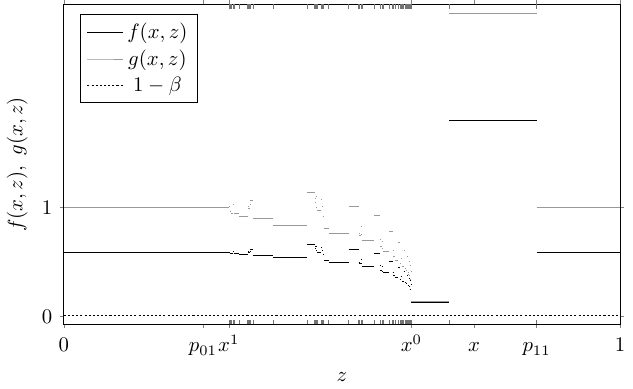}%
  \caption{Marginal metrics $f(x,z)$ and $g(x,z)$ vs.\ $z$ for fixed belief $x$.}
  \label{fig:fg_vs_z_plot}
\end{figure}

\begin{figure}[!htbp]
  \centering
  \includegraphics[width=0.4\textwidth]{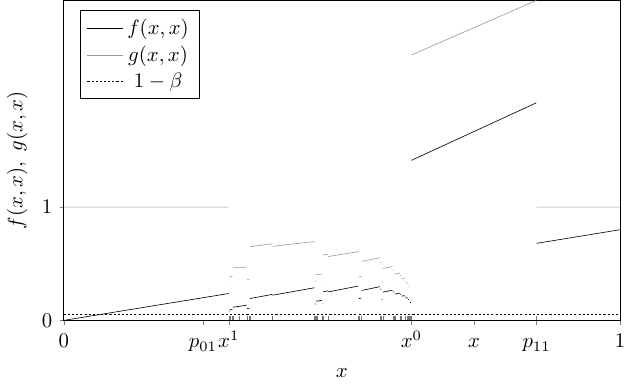}%
  \caption{Marginal metrics $f(x,x)$ and $g(x,x)$ vs.\ $x$.}
  \label{fig:fg_diag}
\end{figure}

\begin{figure}[!htbp]
  \centering
  \includegraphics[width=0.3\textwidth]{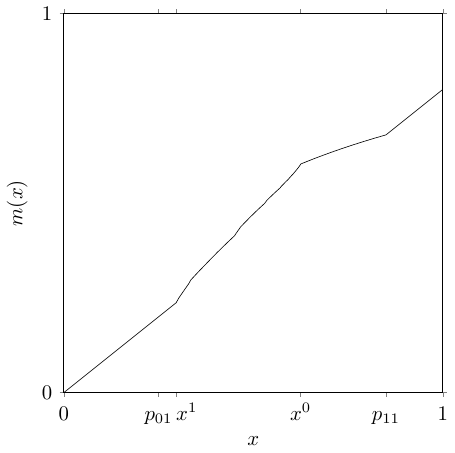}%
  \caption{MP index $m(x)$ vs.\ $x$.}
  \label{fig:mp_index_vs_x}
\end{figure}

Figures~\ref{fig:fg_diag} and~\ref{fig:mp_index_vs_x} show the diagonal quantities \(f(x,x)\), \(g(x,x)\), and
the MP index \(m(x)\). In the low- and high-belief regions, the plots agree with the exact formulas in
Proposition~\ref{pro:partial_myopic_index}, namely
$
m(x)=r\kappa x=0.8x
$ for
$x\in[0,x^1]\cup[p_{11},1].
$
On the interval \([x^0,p_{11}]\), the plot agrees with the finite-sum representation established in the same
proposition. In the remaining interval \(x^1<x<x^0\), where a complete analytical treatment is left for future work, the graph of
\(m(x)\) appears continuous and nondecreasing, which is precisely the behavior required for \textup{(PCLI2)}.

Overall, this illustrative instance highlights the main message of the paper.
The tractable-regime analysis accurately captures the coarse transitions of the metrics and the MP index, while the
intermediate regime exhibits a rich but highly structured staircase behavior that strongly suggests that the same
underlying symbolic organization persists at the metric level.
This motivates the broader numerical investigation of Section~\ref{s:experiments}.

\section{Itinerary organization via words}
\label{app:word_structure}

This appendix records the symbolic-dynamics material underlying the discussion in
Section~\ref{s:open_gaps}. It describes the organization of no-ACK threshold
itineraries in the intermediate regime in terms of binary words, and states the corresponding
Christoffel--Sturmian structure inherited from the maps-with-gaps theory of
\citet{danceSi19}.

A key simplification is that, for a fixed threshold \(z\), the no-ACK skeleton is fully
determined by the binary decisions ``active'' and ``passive'' along the trajectory. We
therefore encode skeleton itineraries by words over the alphabet
$
\mathcal A\triangleq\{0,1\},
$
where \(1\) denotes an active step (so the update is \(\phi^1\)) and \(0\) denotes a passive
step (so the update is \(\phi^0\)). This converts a threshold-driven piecewise iteration into a
composition of the two deterministic maps \(\phi^0\) and \(\phi^1\), and makes it possible to
compare trajectories through the combinatorial structure of the associated words.

A finite word is a string \(w=w_0\cdots w_{\ell-1}\in\mathcal A^\ell\) of length \(|w|=\ell<\infty\),
while an infinite word is \(w=w_0w_1w_2\cdots\in\mathcal A^{\mathbb N}\). We denote by
\(\emptyset\) the empty word. For a finite word \(w\), we write \(|w|_1\) and \(|w|_0\) for the
number of \(1\)'s and \(0\)'s occurring in \(w\), respectively. If \(u\) and \(v\) are words,
their concatenation is denoted \(uv\). We say that \(u\) is a prefix of \(w\) if \(w=uv\) for
some (possibly empty) word \(v\). For \(m\in\mathbb Z_+\), we denote by \(w^m\) the \(m\)-fold
concatenation of a finite word \(w\) (with \(w^0=\emptyset\)), and let \(w^\infty\) denote
infinite periodic repetition. For a word \(w\) and \(0\le i\le j\le |w|-1\), we write
$
w_{i:j}\triangleq w_i\cdots w_j.
$

For a word \(w=w_0\cdots w_{t-1}\), we denote by \(\phi^w\) the corresponding
composition of the deterministic belief-update maps \(\phi^0\) and \(\phi^1\),
$
\phi^w\triangleq \phi^{w_{t-1}}\circ\cdots\circ\phi^{w_0},
$
with the convention $\phi^{\emptyset}(x)\equiv x.$
If \(w=\widetilde{\sigma}(x,z)\) is the infinite skeleton itinerary, then its
length-\(t\) prefix is
$
w_{0:t-1}=\widetilde A_0(x,z)\cdots \widetilde A_{t-1}(x,z),\enspace t\ge1,
$
and
$
\widetilde X_t(x,z)=\varphi_t(x,z)=\phi^{w_{0:t-1}}(x),\enspace t\ge1.
$

Fix \(x\in\mathcal X\), and let \(u=u_0\cdots u_{t-1}\in\mathcal A^t\) be a finite word.
Define the associated deterministic belief sequence by
\[
y_0^u(x)\triangleq x,\qquad
y_{k+1}^u(x)\triangleq \phi^{u_k}\!\bigl(y_k^u(x)\bigr),\qquad k=0,\ldots,t-1,
\]
so that \(y_t^u(x)=\phi^u(x)\). Set
\[
L^u(x)\triangleq \max\{y_k^u(x):\,u_k=0\},
\qquad
U^u(x)\triangleq \min\{y_k^u(x):\,u_k=1\},
\]
with the conventions \(\max\emptyset=-\infty\) and \(\min\emptyset=+\infty\), and define the
word-realization interval
\begin{equation}
\label{eq:Iu_def_word_structure}
I^u(x)\triangleq [\,L^u(x),\,U^u(x)\,)\subset\mathbb R.
\end{equation}
Then \(u\) is realized as the length-\(t\) strict-threshold skeleton prefix if and only if
\(z\in I^u(x)\), and in that case
\begin{equation}
\label{eq:varphi_const_on_Iu_word_structure}
\varphi_t(x,z)=\phi^u(x)=y_t^u(x).
\end{equation}

The corresponding statement for infinite words is obtained by intersection over all finite
prefixes.

\begin{lemma}
\label{lma:varphixzpwcrcz}
Fix \(x\in\mathcal X\).
\begin{enumerate}[label=\textup{(\alph*)},leftmargin=2em]
\item For every \(t\in\mathbb Z_+\), the map \(z\mapsto \varphi_t(x,z)\) is piecewise constant
and right-continuous on \(\mathbb R\). Any point of discontinuity must be a self-consistency
point, in the sense that
$
z=\varphi_k(x,z)\enspace\text{for some }k\in\{0,1,\dots,t-1\}.
$

\item For every infinite word \(w\in\mathcal A^{\mathbb N}\), if the itinerary interval
$
I^w(x)\triangleq \bigcap_{t\ge1} I^{w_{0:t-1}}(x)
$
is nonempty, then for every \(z\in I^w(x)\) the strict-threshold skeleton itinerary equals
\(w\), and for every \(t\ge1\),
\begin{equation}
\label{eq:varphi_const_on_Iw_word_structure}
\varphi_t(x,z)=\phi^{\,w_{0:t-1}}(x).
\end{equation}
In particular, each \(\varphi_t(x,\cdot)\) is constant on \(I^w(x)\).
\end{enumerate}
\end{lemma}

\begin{proof}
Part \textup{(a)} follows from \eqref{eq:varphi_const_on_Iu_word_structure}, since for fixed \(t\)
there are finitely many words \(u\in\mathcal A^t\). Right-continuity follows because
each realization interval is of the form \([L,U)\). Part \textup{(b)} follows by applying
\eqref{eq:varphi_const_on_Iu_word_structure} to every prefix \(w_{0:t-1}\).
\end{proof}

The next lemma formalizes the standard fact, used explicitly in
\cite[\S2.4, remark after Theorem~12]{danceSi19}, that threshold itineraries are preserved
by an increasing change of variables.

\begin{lemma}[Itinerary invariance under increasing conjugacy]
\label{lem:itinerary_invariance_conjugacy}
Let \(\gamma:I\to J\) be increasing and set
$
\hat\phi^{a}=\gamma\circ\phi^{a}\circ \gamma^{-1},\enspace a\in\{0,1\}.
$
Then for all \(x,z\in I\),
\[
\widetilde{\sigma}(x,z\mid \phi^0,\phi^1)
=
\widetilde{\sigma}(\gamma(x),\gamma(z)\mid \hat\phi^0,\hat\phi^1),
\qquad
\widetilde{\sigma}(x,z^{-}\mid \phi^0,\phi^1)
=
\widetilde{\sigma}(\gamma(x),\gamma(z)^{-}\mid \hat\phi^0,\hat\phi^1).
\]
\end{lemma}

\begin{proof}
Let \((x_j)\) be the \(z\)-threshold orbit from \(x\) under \((\phi^0,\phi^1)\), and set
\(y_j=\gamma(x_j)\). Since \(\gamma\) is increasing, \(x_j>z\) iff \(y_j>\gamma(z)\), and
\(x_j\ge z\) iff \(y_j\ge \gamma(z)\). Moreover,
$
y_{j+1}=\gamma(\phi^0(x_j))=\hat\phi^0(y_j)
$
when \(y_j\le \gamma(z)\) (resp.\ \(<\gamma(z)\)), and
$
y_{j+1}=\gamma(\phi^1(x_j))=\hat\phi^1(y_j)
$
when \(y_j>\gamma(z)\) (resp.\ \(\ge\gamma(z)\)).
Thus \((y_j)\) is exactly the \(\gamma(z)\)-threshold (resp.\ \(\gamma(z)^-\)-threshold) orbit from
\(\gamma(x)\) under \((\hat\phi^0,\hat\phi^1)\), and the induced itinerary letters coincide.
\end{proof}

A word is \emph{balanced} if, for every \(\ell\ge1\), the number of \(1\)'s in any two
length-\(\ell\) factors differs by at most one. A source of balanced words is provided
by \emph{mechanical words}: for \(\alpha\in(0,1)\) and \(\eta\in\mathbb R\), the \emph{lower mechanical
word} is
\[
(L_{\alpha,\eta})_j
\triangleq
\big\lfloor (j+1)\alpha+\eta\big\rfloor-\big\lfloor j\alpha+\eta\big\rfloor,
\qquad j\ge0.
\]
When \(\alpha\) is rational, the lower mechanical word is periodic, and its primitive period is,
up to cyclic shift, a \emph{Christoffel word}. When \(\alpha\) is irrational, one obtains an aperiodic
balanced word, i.e., a \emph{Sturmian word}.

If \(\alpha=m/n\in(0,1)\) is rational in lowest terms, the corresponding lower Christoffel
word is
$
C_{m/n}\triangleq (L_{m/n,0})_{0:(n-1)}.
$
It has the standard palindromic factorization
$
C_{m/n}=0p1,
$
where \(p\) is a (possibly empty) palindrome. If \(\alpha\) is irrational, the lower mechanical
word \(L_{\alpha,0}\) is a Sturmian word. Following \citet[Definition~6]{danceSi19}, we call
\(C_{m/n}\) the \(\mathcal M\)-word of rate \(\alpha=m/n\) when \(\alpha\) is rational, and
\(L_{\alpha,0}\) the \(\mathcal M\)-word of rate \(\alpha\) when \(\alpha\) is irrational.

For a finite non-empty word \(w\), Assumption~\ref{ass:DSA2_generic} implies that \(\phi^w\)
has a unique fixed point \(x^w\in\mathcal X\). For a Sturmian \(\mathcal M\)-word \(0s\),
define \(x^s\) via the common limit of fixed points along the associated Christoffel
approximants, exactly as in \cite[\S2.4]{danceSi19}; existence and coincidence of the two
limits follow from \cite[Lemma~55]{danceSi19}.

\begin{theorem}[Threshold itineraries as Christoffel and Sturmian words]
\label{thm:our_Thm12_mwords}
Let \(0p1\) be a Christoffel word and let \(0s\) be a Sturmian \(\mathcal M\)-word. Then the
fixed points \(x^{01p}\), \(x^{10p}\), and \(x^s\) exist in \((0,1)\). Moreover, the
active-at-threshold left-threshold itinerary \(z\mapsto\widetilde{\sigma}(z,z^-)\) is
lexicographically nonincreasing on \(z\in(0,1)\) and is given by
\[
\widetilde{\sigma}(z,z^-)=
\begin{cases}
1^\infty, & \text{if and only if } z\le x^1,\\
(10p)^\infty, & \text{if and only if } z\in [x^{01p},x^{10p}],\\
10s, & \text{if and only if } z=x^s,\\
10^\infty, & \text{if and only if } z\ge x^0.
\end{cases}
\]
\end{theorem}

\begin{proof}
Let \(\vartheta\) be the logit map and let \((\hat\phi^0,\hat\phi^1)\) be the conjugated maps
from Lemma~\ref{lem:logit_conjugacy_A2}. By that lemma, \((\hat\phi^0,\hat\phi^1)\)
satisfies Assumption~2 of \citet{danceSi19} on \(\mathbb R\), with ordered fixed points
\(\hat x^1<\hat x^0\). Therefore \cite[Theorem~12]{danceSi19} applies to
\(\widetilde{\sigma}(\cdot,\cdot^-)\) for the conjugated maps.

By Lemma~\ref{lem:itinerary_invariance_conjugacy}, for each \(z\in(0,1)\),
$
\widetilde{\sigma}(z,z^- \mid \phi^0,\phi^1)
=
\widetilde{\sigma}(\vartheta(z),\vartheta(z)^- \mid \hat\phi^0,\hat\phi^1),
$
and fixed points for word compositions correspond under \(\vartheta\). Translating the
conclusion of \citep[Theorem~12]{danceSi19} back through \(\vartheta\) yields the stated
characterization.
\end{proof}

\begin{corollary}[One-step-deviation itineraries]
\label{cor:our_Cor13_pairs}
Let \(0p1\) be a Christoffel word and let \(0s\) be a Sturmian \(\mathcal M\)-word. Then the
pair of itineraries
$
\bigl(\widetilde{\sigma}(\phi^0(z),z),\ \widetilde{\sigma}(\phi^1(z),z)\bigr)
$
satisfies
\[
\bigl(\widetilde{\sigma}(\phi^0(z),z),\ \widetilde{\sigma}(\phi^1(z),z)\bigr)
=
\begin{cases}
(1^\infty,1^\infty), & z<x^1,\\
(1^\infty,01^\infty), & z=x^1,\\
((1p0)^\infty,(0p1)^\infty), & z\in [x^{01p},x^{10p}),\\
((1p0)^\infty,\,0p(01p)^\infty), & z=x^{10p},\\
(1s,0s), & z=x^s,\\
(0^\infty,0^\infty), & z\ge x^0.
\end{cases}
\]
\end{corollary}

\begin{proof}
Apply Lemmas~\ref{lem:logit_conjugacy_A2} and \ref{lem:itinerary_invariance_conjugacy}
as in the proof of Theorem~\ref{thm:our_Thm12_mwords}, and then invoke
\cite[Corollary~13]{danceSi19} for the conjugated maps.
\end{proof}

\begin{proposition}[Christoffel-interval partition of the threshold axis]
\label{pro:christoffel_interval_partition}
Let \(\mathcal C\) denote the set of Christoffel words \(0p1\) with rational rate in \((0,1)\).
For each \(0p1\in\mathcal C\), define the Christoffel interval
$
I_{0p1}\triangleq [\,x^{01p},\,x^{10p}\,).
$
Define the Sturmian set
$
\mathcal S \triangleq \{x^s:\ 0s \text{ is a Sturmian \(\mathcal M\)-word}\}.
$
Then:
\begin{enumerate}[label=\textup{(\roman*)},leftmargin=2.2em]
\item the family \(\{I_{0p1}:0p1\in\mathcal C\}\) is pairwise disjoint and
\[
(x^1,x^0)
=
\Big(\bigsqcup_{0p1\in\mathcal C} I_{0p1}\Big)\sqcup \mathcal S;
\]

\item on each Christoffel interval \(I_{0p1}\), the left-threshold itinerary at the threshold
is constant:
\[
\widetilde{\sigma}(z,z^-)= (10p)^\infty,\qquad z\in I_{0p1};
\]

\item for each Sturmian point \(x^s\in\mathcal S\), one has
$
\widetilde{\sigma}(x^s,(x^s)^-)=10s;
$

\item the extreme regimes are
\[
\widetilde{\sigma}(z,z^-)=1^\infty\ \Longleftrightarrow\ z\le x^1,
\qquad
\widetilde{\sigma}(z,z^-)=10^\infty\ \Longleftrightarrow\ z\ge x^0.
\]
\end{enumerate}
\end{proposition}

\begin{proof}
Theorem~\ref{thm:our_Thm12_mwords} gives the threshold-itinerary classification on the closed
intervals \([x^{01p},x^{10p}]\). Passing to the half-open convention
\(I_{0p1}=[x^{01p},x^{10p})\) removes the endpoint overlaps between adjacent Christoffel
intervals and yields a disjoint partition of \((x^1,x^0)\), with the remaining threshold values
given by the Sturmian points \(x^s\). The itinerary statements in \textup{(ii)}--\textup{(iv)}
are then exactly those of Theorem~\ref{thm:our_Thm12_mwords}.
\end{proof}

Proposition~\ref{pro:christoffel_interval_partition} shows that the threshold axis in the
intermediate regime is partitioned into Christoffel intervals and Sturmian points. On each
Christoffel interval, the active-at-threshold itinerary is fixed and periodic, while
Corollary~\ref{cor:our_Cor13_pairs} fixes the corresponding one-step-deviation itineraries for the passive-at-threshold case.
These are precisely the symbolic inputs needed for a future intervalwise analysis of the
pre-ACK quantities \(\widetilde F,\widetilde G,\widetilde\Theta\), the marginal metrics
\(f,g\), and hence the MP index \(m\) on the intermediate regime.

\section{Outline of the time-average criterion}
\label{app:avgcrit_outline}

The discounted analysis developed above naturally suggests a corresponding long-run average formulation. Since the present paper is primarily concerned with the discounted criterion, we only sketch here how the no-ACK skeleton and renewal machinery extend to the time-average setting. A full treatment of the average-criterion PCL-indexability framework for real-state projects is left for future work.

Fix a threshold \(z\in\R\). Let \(F_\beta(x,z)\), \(G_\beta(x,z)\), \(f_\beta(x,z)\), and \(g_\beta(x,z)\) denote the discounted reward, work, and marginal metrics introduced above, with the discount factor made explicit. Motivated by Abelian-limit considerations, define the corresponding time-average quantities by
\begin{equation}
\label{eq:avg_metric_defs}
F(x,z)\triangleq \lim_{\beta\nearrow 1}(1-\beta)F_\beta(x,z),
\qquad
G(x,z)\triangleq \lim_{\beta\nearrow 1}(1-\beta)G_\beta(x,z),
\end{equation}
and
\begin{equation}
\label{eq:avg_marginal_defs}
f(x,z)\triangleq \lim_{\beta\nearrow 1} f_\beta(x,z),
\qquad
g(x,z)\triangleq \lim_{\beta\nearrow 1} g_\beta(x,z).
\end{equation}
In this outline we take for granted that these limits exist and that $F(x,z)$ and $G(x,z)$ are independent of the initial belief \(x\). This is consistent with the regenerative structure of the model and with the numerical evidence discussed below. When this independence holds, we simply write \(F(z)\) and \(G(z)\). The associated average-criterion MP index is then
\begin{equation}
\label{eq:avg_mp_index_defs}
m(x,z)\triangleq \frac{f(x,z)}{g(x,z)},
\qquad
m(x)\triangleq m(x,x),
\end{equation}
whenever \(g(x,z)>0\).

In the tractable regimes, the discounted formulas of Proposition~\ref{pro:FG_closed_form_with_proof} suggest the following \(\beta\nearrow1\) limits. In the eventually-always-active regimes \(z\le x^1\), namely cases \textup{(a)}--\textup{(b)} of Proposition~\ref{pro:FG_closed_form_with_proof}, one is led to
\begin{equation}
\label{eq:avg_FG_active_regimes}
F(z)=r\kappa x^0,
\qquad
G(z)=1,
\qquad z\le x^1,
\end{equation}
where \(x^0=p_{01}/(1-\rho)\) is the passive fixed point. Thus, once the threshold is at or below \(x^1\), the policy is eventually always active, the average work rate is \(1\), and the average reward rate is \(r\kappa x^0\).

At the other extreme, in the eventually-passive regimes \(z\ge x^0\), the threshold policy activates only finitely many times (or not at all), so the discounted reward and work totals remain \(O(1)\) as \(\beta\nearrow1\). This suggests
\begin{equation}
\label{eq:avg_FG_passive_regimes}
F(z)=0,
\qquad
G(z)=0,
\qquad z\ge x^0.
\end{equation}
In particular, \eqref{eq:avg_FG_passive_regimes} covers both \(x^0\le z<p_{11}\) and \(z\ge p_{11}\).

Therefore, the only genuinely nontrivial average-criterion regime is the intermediate strip
$
x^1<z<x^0,
$
which is precisely the regime where the no-ACK skeleton alternates indefinitely between active and passive phases.

The limiting marginal metrics \(f(x,z)\) and \(g(x,z)\) may be obtained by passing to the limit in the discounted one-step formulas of Proposition~\ref{pro:fg_closed_form}. In the  extreme regimes \(z<p_{01}\) and \(z\ge p_{11}\), those formulas remain unchanged:
\begin{equation}
\label{eq:avg_fg_extreme_regimes}
f(x,z)=r\kappa x,
\qquad
g(x,z)=1,
\qquad
m(x,z)=r\kappa x.
\end{equation}
In the remaining tractable subregimes \(p_{01}\le z\le x^1\) and \(x^0\le z<p_{11}\), explicit closed forms for \(f(x,z)\) and \(g(x,z)\) are suggested by taking \(\beta\nearrow1\) in the corresponding discounted formulas; they involve the same threshold crossing times \(\tau_\uparrow^0\) and \(\tau_\downarrow^1\), and, in the \(x^0\le z<p_{11}\) regime, finite sums over the active no-ACK skeleton segment. Since the present section is only meant as an outline, we do not record those expressions here.

For \(x^1<z<x^0\), the average-criterion metrics can be approached in two natural ways. The first is the simplest computationally: evaluate the discounted metrics \(F_\beta(x,z)\), \(G_\beta(x,z)\), \(f_\beta(x,z)\), and \(g_\beta(x,z)\) for \(\beta\) sufficiently close to \(1\), and then use \eqref{eq:avg_metric_defs}--\eqref{eq:avg_marginal_defs} as numerical approximations. The second is to work directly with the regenerative structure at ACK times. Writing \(\tau^{\mathrm{ack}}\) for the first ACK time under the \(z\)-threshold policy, define the undiscounted cycle quantities
\begin{align}
\label{eq:avg_cycle_metrics}
\overline R(x,z)
&\triangleq
\Ex_x^z\!\bigg[\sum_{t=0}^{\tau^{\mathrm{ack}}} R(X(t),A(t))\bigg], \quad
\overline W(x,z)
\triangleq
\Ex_x^z\!\bigg[\sum_{t=0}^{\tau^{\mathrm{ack}}} A(t)\bigg], \quad
\overline T(x,z)
\triangleq
\Ex_x^z\big[\tau^{\mathrm{ack}}+1\big].
\end{align}
Formally, when the threshold policy regenerates at the post-ACK state \(p_{11}\), one expects the average reward and work rates to satisfy
\begin{equation}
\label{eq:avg_regen_formulas}
F(z)=\frac{\overline R(p_{11},z)}{\overline T(p_{11},z)},
\qquad
G(z)=\frac{\overline W(p_{11},z)}{\overline T(p_{11},z)}.
\end{equation}
Likewise, the average marginal metrics \(f(x,z)\) and \(g(x,z)\) may be obtained either as limits of the discounted marginals or by combining one-step deviation formulas with the same regenerative continuation values. Thus, the discounted skeleton and renewal machinery provide a natural computational route to the time-average criterion as well.

The tractable-regime formulas above strongly suggest that the average-criterion metrics inherit the same structural properties as their discounted counterparts. In particular, in the tractable regimes one expects the average-criterion analogues of \textup{(PCLI1)} and \textup{(PCLI2)} to hold, namely positivity of \(g(x,z)\) and monotonicity of the diagonal MP index \(m(x)\). For the extreme regimes \(z<p_{01}\) and \(z\ge p_{11}\), this is immediate from \eqref{eq:avg_fg_extreme_regimes}. The corresponding statements in the remaining tractable regimes are suggested by the limiting closed forms obtained from Proposition~\ref{pro:fg_closed_form}.

This leads naturally to the following conjecture.

\begin{conjecture}
\label{con:pcliavcr}
The average-criterion limiting metrics \(f(x,z)\), \(g(x,z)\), and \(m(x)=f(x,x)/g(x,x)\) satisfy the average-criterion analogues of \textup{(PCLI1)} and \textup{(PCLI2)} on the full threshold range, including the regime \(x^1<z<x^0\).
\end{conjecture}

If true, this would yield an average-criterion MP index for the present real-state project model and provide the time-average counterpart of the discounted PCL-indexability analysis developed in this paper.

The main point of this section is that the no-ACK skeleton and renewal decomposition are not specific to the discounted criterion. They also suggest a natural average-criterion theory, with explicit tractable-regime formulas, direct computational schemes in the intermediate regime, and a plausible extension of the PCL-indexability conditions. This average-criterion development is beyond the scope of the present paper, but the discounted results obtained here provide its natural starting point.

\section{Supplementary experimental plots and implementation details}
\label{app:extra_experiments}

\subsection{Parameter dependence of the MP index}
\label{app:param_dependence}

\subsubsection{Dependence of the MP index on \(\beta\)}
\label{s:mp_family_beta}

We begin with the dependence of the MP index on the discount factor \(\beta\), using the same
base parameter instance as in Appendix~\ref{app:illustrative_instance}, namely
$
p_{01}=0.25, \rho=0.6,  \kappa=0.8,  r=1.
$
Figure~\ref{fig:mp_family_beta} plots the MP index \(m(x)\) for
$
\beta\in\{0.1,\,0.3,\,0.5,\,0.70,\,0.90,\,0.95,\,0.999\}.
$

Several features are apparent. First, over the displayed range, the MP index appears to be increasing in \(\beta\), and the curves appear to converge to a limiting profile as \(\beta\nearrow 1\), consistent with the time-average discussion of Appendix~\ref{app:avgcrit_outline}. Second, by Proposition~\ref{pro:partial_myopic_index}, the MP index coincides with the myopic index
$
m(x)=r\kappa x
$
on the low- and high-belief regions
$
x\in[0,x^1]\cup[p_{11},1].
$
Accordingly, the visible dependence on \(\beta\) is concentrated in the intermediate belief
region \((x^1,p_{11})\), where future information and restart effects matter most.
The sensitivity to \(\beta\) is strongest in the nontrivial interval
\(x^1<x<x^0\) and persists on \([x^0,p_{11}]\)).

\begin{figure}[!htbp]
  \centering
  \includegraphics[width=0.30\textwidth]{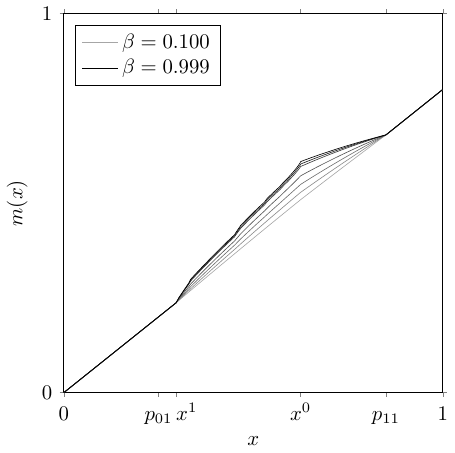}%
  \caption{MP index \(m(x)\) versus \(x\) for several values of \(\beta\).}
  \label{fig:mp_family_beta}
\end{figure}

\subsubsection{Dependence of the MP index on \(\kappa\)}
\label{s:mp_family_kappa}

We next examine the dependence of the MP index on the sensing-success parameter \(\kappa\), using
the same base parameter instance as in Appendix~\ref{app:illustrative_instance}, with
$
p_{01}=0.25,\rho=0.6, \beta=0.95, r=1
$
held fixed.
Figure~\ref{fig:mp_family_kappa} plots the MP index \(m(x)\) for
$
\kappa\in\{0.1,\,0.2,\,0.3,\,0.4,\,0.5,\,0.6,\,0.70,\,0.8,\,0.90,\,0.95,\,0.999\}.
$

Unlike the \(\beta\)-family case, the active no-ACK fixed point \(x^1\) varies with \(\kappa\),
whereas the passive fixed point \(x^0\) and the reset state \(p_{11}\) remain unchanged.
For this reason, the boundary between the low-belief myopic region and the intermediate region
moves with \(\kappa\), while the points \(x^0\) and \(p_{11}\) remain fixed on the horizontal axis.
Accordingly, the default \(x\)-axis ticks in Figure~\ref{fig:mp_family_kappa} show only these
\(\kappa\)-invariant reference points.

The figure shows that the MP index appears to be increasing in \(\kappa\) over the displayed range, and
to converge to a limiting profile as \(\kappa\nearrow 1\). As in the previous subsection,
the most significant changes occur in the nontrivial belief region. In particular, since
Proposition~\ref{pro:partial_myopic_index} gives
$m(x)=r\kappa x$ for $x\in[0,x^1]\cup[p_{11},1],$
the dependence on \(\kappa\) is especially informative on the interval \((x^1,p_{11})\), where
both \(x^1\) and the magnitude of the index vary with \(\kappa\).

\begin{figure}[!htbp]
  \centering
  \includegraphics[width=0.30\textwidth]{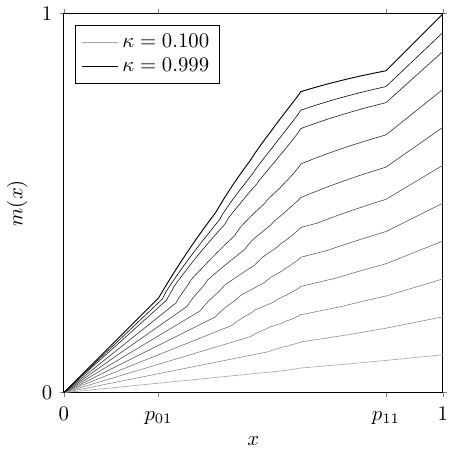}%
  \caption{MP index \(m(x)\) versus \(x\) for several values of \(\kappa\).}
  \label{fig:mp_family_kappa}
\end{figure}

\subsubsection{Dependence of the MP index on \(\rho\)}
\label{s:mp_family_rho}

We next examine the dependence of the MP index on the persistence parameter \(\rho\), using
the same base parameter instance as in Appendix~\ref{app:illustrative_instance}, with
$
p_{01}=0.25, \kappa=0.8, \beta=0.95, r=1
$
held fixed.
Figure~\ref{fig:mp_family_rho} plots the MP index \(m(x)\) for
$
\rho\in\{0.001,\,0.1,\,0.2,\,0.3,\,0.4,\,0.5,\,0.6,\,0.70,\,0.749\}.
$
Since \(p_{01}=0.25\) is held fixed, feasibility requires
$
0<\rho<1-p_{01}=0.75,
$
so the largest plotted value \(\rho=0.749\) lies just below the upper admissible boundary.

Unlike the \(\beta\)- and \(\kappa\)-families, the key internal reference points
$
x^1, x^0, p_{11}=p_{01}+\rho
$
all depend on \(\rho\), so they no longer occur at common locations on the horizontal axis.
For this reason, the default \(x\)-axis ticks in Figure~\ref{fig:mp_family_rho} use only
\(\rho\)-invariant reference points.

The dependence of \(m(x)\) on \(\rho\) is qualitatively different from its dependence on
\(\beta\) and \(\kappa\). The family of curves is not monotone in \(\rho\):
for some belief values the index increases with \(\rho\), whereas for others it decreases,
and the overall shape of the non-myopic region changes substantially as \(\rho\) varies.
Figure~\ref{fig:mp_index_vs_rho}, which plots \(m(x)\) as a function of \(\rho\) for a fixed
representative belief \(x\), makes this non-monotone dependence more explicit.

As \(\rho\) approaches the upper feasibility boundary \(1-p_{01}=0.75\), the MP index appears
to approach a limiting profile. Also, the movement of the transition points
\(x^1\), \(x^0\), and \(p_{11}\) shows that changes in persistence affect not only the magnitude
of the index but also the location and width of the regions where the MP index differs from
myopic.

\begin{figure}[!htbp]
  \centering
  \includegraphics[width=0.30\textwidth]{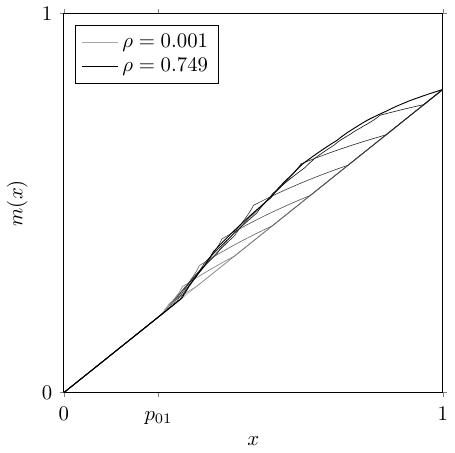}%
  \caption{MP index \(m(x)\) versus \(x\) for several values of \(\rho\).}
  \label{fig:mp_family_rho}
\end{figure}

\begin{figure}[!htbp]
  \centering
  \includegraphics[width=0.30\textwidth]{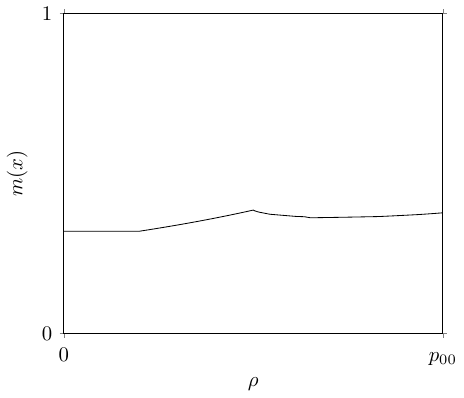}%
  \caption{MP index \(m(x)\) versus \(\rho\) for a fixed belief \(x\).}
  \label{fig:mp_index_vs_rho}
\end{figure}

\subsubsection{Dependence of the MP index on \(p_{01}\)}
\label{s:mp_family_p01}

We next examine the dependence of the MP index on the transition parameter \(p_{01}\), using
the same base parameter instance as in Appendix~\ref{app:illustrative_instance}, except that
\(p_{01}\) is now varied while
$
\rho=0.6, \kappa=0.8, \beta=0.95, r=1
$
are held fixed.
Figure~\ref{fig:mp_family_q} plots the MP index \(m(x)\) for
$
p_{01}\in\{0.001,\,0.10,\,0.15,\,0.20,\,0.30,\,0.35,\,0.399\}.
$

Since \(\rho=0.6\) is fixed, feasibility requires
$
0<p_{01}<1-\rho=0.4,
$
so the largest plotted value \(p_{01}=0.399\) lies just below the upper admissible boundary.

As in the \(\rho\)-family case, the key internal reference points
$
x^1, x^0={p_{01}}/({1-\rho}), p_{11}=p_{01}+\rho
$
all depend on the varying parameter. Hence they no longer occur at common locations on the
horizontal axis, and the default \(x\)-axis ticks in Figure~\ref{fig:mp_family_q} display only
\(p_{01}\)-invariant reference points.

The dependence of \(m(x)\) on \(p_{01}\) is again non-monotone: for some belief values the
index increases with \(p_{01}\), whereas for others it decreases, and the shape and location of
the non-myopic region vary substantially across the family. Figure~\ref{fig:mp_index_vs_q},
which plots \(m(x)\) as a function of \(p_{01}\) for a fixed representative belief \(x\),
makes this non-monotone dependence more explicit.

As \(p_{01}\) approaches the upper feasibility boundary \(1-\rho=0.4\), the MP index appears
to approach a limiting profile. At the same time, the movement of \(x^1\), \(x^0\), and \(p_{11}\)
shows that varying \(p_{01}\) changes not only the magnitude of the index but also the location
and extent of the regions in which the MP index differs from the myopic index.

\begin{figure}[!htbp]
  \centering
  \includegraphics[width=0.30\textwidth]{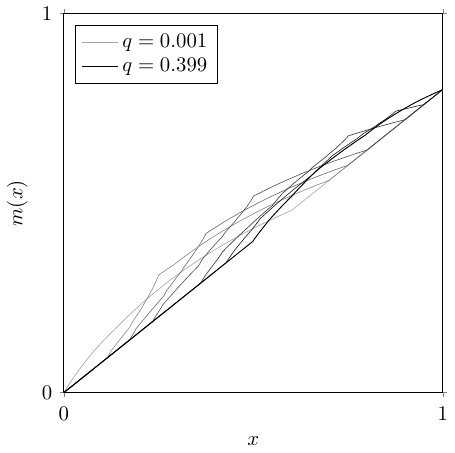}%
  \caption{MP index \(m(x)\) versus \(x\) for several values of $p_{01}$.}
  \label{fig:mp_family_q}
\end{figure}

\begin{figure}[!htbp]
  \centering
  \includegraphics[width=0.35\textwidth]{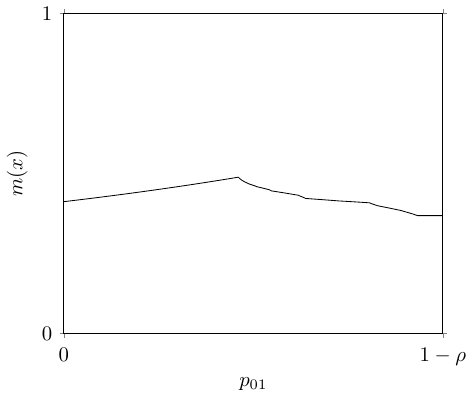}%
  \caption{MP index \(m(x)\) versus \(p_{01}\) for a fixed belief \(x\).}
  \label{fig:mp_index_vs_q}
\end{figure}

\subsection{Implementation details and reproducibility}
\label{app:repro}

All computations were implemented in Julia. The numerical evaluation of the discounted metrics and the MP index relied on the no-ACK skeleton and renewal
decompositions developed in Sections~\ref{s:compFG}--\ref{s:ia-partial-index}. Infinite pre-ACK
series were evaluated by truncation using the explicit tail bounds derived in
\eqref{eq:tilde_tail_bounds}, so that the truncation depth could be chosen to meet a prescribed
error tolerance. Unless otherwise stated, the numerical routines used a tolerance
\(\varepsilon=10^{-10}\).

For dual-bound computation in the policy-benchmarking experiments, we used the grouped convex
solver described in Section~\ref{s:benchmark}, which exploits repeated project types and computes
the minimizing Lagrange multiplier by bisection on a subgradient. For policy simulation, the
MP index policy used precomputed type-specific lookup tables with linear interpolation, while
Monte Carlo estimates were based on fixed random seeds to ensure reproducibility.

The large-scale numerical tests of \textup{(PCLI1)} and \textup{(PCLI2)}, as well as the
policy-benchmarking experiments, were run in multithreaded mode. The code was checked internally
by comparing special cases against the tractable closed forms of
Sections~\ref{s:tractable_regimes}--\ref{s:ia-partial-index}, by verifying identities such as
\(\widetilde F=r\widetilde\Theta\), and by confirming agreement between the original and
accelerated dual-bound solvers on representative test instances.

The code used to generate the numerical results and figures is available from the author upon request.

\end{document}